\documentclass[11pt]{article}
\usepackage{graphicx,amsmath,amsfonts,amssymb,bm,cite}
\usepackage{algorithm,algorithmic}
\usepackage[tight]{subfigure}
\usepackage[mathscr]{eucal}
\usepackage{fullpage}
\usepackage[small,bf]{caption}
\usepackage{multirow}
\usepackage{url}
\usepackage{mathabx}
\makeatletter
\renewcommand{\@thesubfigure}{\hskip\subfiglabelskip}
\makeatother

\newcommand{\R}{\mathbb{R}}
\newcommand{\C}{\mathbb{C}}

\newcommand{\PP}{\mathbb{P}}
\newcommand{\E}{\mathbb{E}}
\newcommand{\<}{\langle}
\renewcommand{\>}{\rangle}

\newcommand{\ten}[1]{\boldsymbol{\mathscr{#1}}}
\newcommand{\vct}[1]{\boldsymbol{#1}}
\newcommand{\mtx}[1]{\boldsymbol{#1}}
\newcommand{\tub}[1]{\mathring{\vct{#1}}}
\newcommand{\tc}[1]{\vec{\vct{#1}}}

\newcommand{\rank}{\operatorname{rank}}
\newcommand{\tr}{\operatorname{Tr}}
\newcommand{\bdiag}{\operatorname{blockdiag}}
\newcommand{\vv}{\operatorname{vec}}
\newcommand{\sgn}{\operatorname{sgn}}

\newcommand{\eijk}{\tc{e}_i \ast \tub{e}_k \ast \tc{e}^{H}_j}
\newcommand{\eabc}{\tc{e}_a \ast \tub{e}_b \ast \tc{e}^{H}_c}

\newcommand{\Pg}{{\cal P}_{\Gamma}}
\newcommand{\Pgp}{{\cal P}_{\Gamma'}}
\newcommand{\Pgc}{{\cal P}_{\Gamma^\perp}}
\newcommand{\Pgs}[1]{{\cal P}_{\Gamma_{#1}}}
\newcommand{\PT}{{\cal P}_T}
\newcommand{\PTc}{{\cal P}_{T^\perp}}
\newcommand{\PO}{{\cal P}_{\Omega}}
\newcommand{\POc}{{\cal P}_{\Omega^\perp}}
\newcommand{\Pl}{{\cal P}_{\Lambda}}
\newcommand{\OpId}{\mathcal{I}}

\newcommand{\fft}{ \mbox{\tt fft} }
\newcommand{\ifft}{ \mbox{\tt ifft} }
\newcommand{\svd}{ \mbox{\tt svd} }

\newtheorem{definition}{Definition}[section]
\newtheorem{lemma}{Lemma}[section]
\newtheorem{theorem}{Theorem}[section]
\newtheorem{remark}{Remark}[section]
\newtheorem{corollary}[theorem]{Corollary}

\def \endprf{\hfill {\vrule height6pt width6pt depth0pt}\medskip}
\newenvironment{proof}{\noindent {\bf Proof} }{\endprf\par}

\title{Exact Tensor Completion from Sparsely Corrupted Observations via Convex Optimization}

\author{Jonathan Q. Jiang and Michael K. Ng\\
 \vspace{-.1cm}\\
  Department of Mathematics, Hong Kong Baptist University
}

\date{\today}

\begin{document}

\maketitle

\vspace{-0.3in}

\begin{abstract}
This paper conducts a rigorous analysis for provable estimation of multidimensional arrays, in particular third-order tensors, from a random subset of its corrupted entries. Our study rests heavily on a recently proposed tensor algebraic framework in which we can obtain tensor singular value decomposition (t-SVD) that is similar to the SVD for matrices, and define a new notion of tensor rank referred to as the tubal rank. We prove that by simply solving a convex program, which minimizes a weighted combination of tubal nuclear norm, a convex surrogate for the tubal rank, and the $\ell_1$-norm, one can recover an incoherent tensor exactly with overwhelming probability, provided that its tubal rank is not too large and that the corruptions are reasonably sparse. Interestingly, our result includes the recovery guarantees for the problems of tensor completion (TC) and tensor principal component analysis (TRPCA) under the same algebraic setup as special cases. An alternating direction method of multipliers (ADMM) algorithm is presented to solve this optimization problem. Numerical experiments verify our theory and real-world applications demonstrate the effectiveness of our algorithm.
\end{abstract}


\section{Introduction}
\label{sec1}

The last decade has witnessed an explosion of academic interest in robust recovery of low-rank matrices from severely compressive, incomplete, or even corrupted measurements. The interest has been mainly aroused by the striking fact that data in science, engineering, and society, such as images, videos, texts and microarrays,  all lie on or near some low-dimensional subspaces~\cite{Eckart1936,Tenenbaum2000,BelkinN03}. This discovery says that if we stack all the data points as column vectors of a matrix, the matrix should be low-rank, or approximately so. Surprisingly, it has been shown that under some mild assumptions, efficient techniques based on convex programming to minimize the nuclear norm, as an approximation for the matrix rank, can accurately recover the low-rank matrices~\cite{Cai2010,Candes2011,Wright2013,Shang2014}, as long as their left and right singular vectors are incoherent with the matrix standard basis~\cite{Candes2011,Wright2013,Chen2013,Candes2009,Recht2011,Li2013}.

As modern information technology keeps developing rapidly, multidimensional data is becoming prevalent in many application domains, ranging from image processing~\cite{Plataniotis2000} and computer vision~\cite{Kim2009,Lui2012} to neuroscience~\cite{Miwakeichi2004} and bioinformatics~\cite{Omberg2007}. Conventional methods that rearrange the multidimensional data into matrices by some specific \lq\lq unfolding\rq\rq~or~\lq\lq flattening\rq\rq~strategies, may cause the problem of \lq\lq curse of dimensionality\rq\rq~and also damage the inherent structure, like spatial correlation, within original data. Tensor-based modeling, which can take full advantage of their multilinear structures to provide better understanding and higher precision, is a natural choice in these situations.

Mimicking their low-dimensional predecessors, tensor-based completion~\cite{Liu2013,Tomioka2010,Gandy2011} and robust principal component analysis formulations~\cite{Li2010,Goldfarb2013} have been applied to real applications with promising empirical performance. The recovery theory for low-rank tensor estimation problems, however, is far from being well-established. This is mainly attributed to that tensor rank has different definitions in the literature, each with its own drawback. The CANDECOMP/PARAFAC (CP) decomposition~\cite{Harshman1970,Carroll70} approximates a tensor as sum of rank-one outer products and the minimal number of such decomposition is defined as the CP rank. However, computing the CP rank of a specific tensor is NP-hard in general~\cite{Kolda2009}. Other kinds of decompositions, such as Tucker~\cite{Tucker66} and Tensor Train (TT)~\cite{Oseledets2011}, reveal the algebraic structure in the data with the notion of rank extended to \textit{multi-rank}, expressed as a vector of ranks of the factors. Clearly, such decompositions can not offer the best rank-$k$ approximation\footnote{Such a problem is known as the Eckart-Young-Mirsky approximation for matrix case.} of a tensor.

Unlike the existing models, the t-product and associated algebraic constructs introduced for tensors of order three~\cite{Kilmer2011} and higher~\cite{Martin2013}, provide a new framework in which we can obtain a SVD-like factorization named the tensor-SVD (t-SVD)~\cite{Kilmer2011,Martin2013}, and derive a notion of tensor rank referred to as the tubal rank~\cite{Kilmer2013}. Compared with other tensor decompositions, t-SVD has been shown to be superior in capturing the \lq\lq spatial-shifting\rq\rq~correlation that is ubiquitous in real-world data~\cite{Kilmer2011,Martin2013,Kilmer2013,Zhang2014}. Using this algebraic framework, two recent papers~\cite{Zhang2017} and~\cite{Lu2016} gives sufficient conditions for convex programming to succeed in exact recovery of low-rank tensors from incomplete (tensor completion) and grossly corrupted (tensor robust principal component analysis) observations respectively.

This paper considers a more challenging problem of learning a low-rank tensor from undersampled and possibly arbitrarily corrupted measurements. This problem arises in a wide range of important applications in which the data contain missing values and gross errors simultaneously, due to various factors such as information loss, sensor failures and software malfunctions. The reader might jump to Section~\ref{sec7} to see some practical examples. Actually, this problem is the tensor-based generalization of robust matrix completion (RMC)~\cite{Shang2014,Li2013} and therefore we call it robust tensor completion (RTC) hereafter. Leveraging on the t-SVD algebraic framework, we show that one can obtain an exact recovery of the target tensor with high probability by simply solving a convex program whose object is a weighted combination of tubal nuclear norm~\cite{Zhang2014,Lu2016,Zhang2017}, serving as a convex surrogate for the tubal rank, and the $\ell_1$-norm. The conditions under which our result holds, similar to the regular matrix incoherence conditions~\cite{Candes2011,Candes2009,Li2013}, coincide very well with and are much weaker than the couterparts given by~\cite{Zhang2017} and~\cite{Lu2016} respectively.

We are aware that the RTC problem has been rigorously examined in~\cite{Huang2014}, which proposes a strongly convex program that can be proved to guarantee exact recovery under certain conditions as well. Despite considering the same problem, our study departs from it on several fronts. First, the t-SVD algebraic framework, in which third-order tensors are treated as linear operators over matrices oriented laterally~\cite{Kilmer2013,Braman2010}, is quite different from the classic multilinear algebraic setup for Tucker decomposition used in that work. Besides, the tubal rank and tubal nuclear norm defined in the Fourier domain (see Definition~\ref{def8} and~\ref{def9}), differ seriously from the multi-rank and its convex relaxation \textit{sum-of-nuclear-norms (SNN)}~\cite{Liu2013}. Hence, the recovery theory established in~\cite{Huang2014} is not directly comparable to our result. Our analysis has one additional advantage that is of significant practical importance. It identifies a simple, non-adaptive choice of the regularization parameter in our model. In contrast, the heuristic rule for parameter setting suggested by~\cite{Huang2014} usually suffers a failure in real-world applications, as shown in~\cite{Goldfarb2013,Lu2016} and our experiments.

For convenience, we concentrate on the analysis for third-order tensors in this paper. But the results given here can be easily extended to the case of $N$th-order tensors with $N \geq 3$, by exploiting the higher-order t-SVD framework~\cite{Martin2013}.

The rest of this paper is organized as follows. In Section~\ref{sec2}, we begin with a brief review of related work. The notation and some preliminaries of tensors are introduced in Section~\ref{sec3}, where we outline the t-SVD algebraic framework for third-order tensors. Section~\ref{sec4} describes our main results and discusses the key similarities and differences between our theory and some prior works. We then provide the full proof of Theorem~\ref{the1} in Section~\ref{sec5} and introduce the ADMM algorithm to solve the optimization problem in Section~\ref{sec6}. Finally, we report the numerical and empirical results in Section~\ref{sec7} and draw the conclusions in Section~\ref{sec8}.

\section{Related Work}
\label{sec2}

In this section, we go over related work on low-rank tensor recovery based on different tensor factorizations and associated algebraic frameworks, which can be coarsely spit into two branches: tensor completion (TC) and tensor robust principal component analysis (TRPCA).

\subsection{Tensor Completion}
\label{sec2:sub1}

In TC problem, we would like to recover a low-rank tensor when a limited number of its entries are observed. Jain and Oh~\cite{Jain2014} show that an $n \times n \times n$ symmetric tensor with CP-rank $r$ can be accurately estimated from $O(n^{3/2}r^5\log^4n)$ randomly sampled entries under standard incoherence conditions on the tensor factors. In~\cite{Karlsson2016}, highly scalable algorithms have been proposed for the tasks of filling the missing entries in multidimensional data by the integration of CP decomposition and block coordinate descent (BCD) methods. This optimization problem is non-convex and hence only local minimum can be arrived at. As we all know, it is often computationally intractable to determine the CP rank or its best convex approximation of a tensor, which makes it very difficult to recover tensors with low CP rank, particularly via convex programming.

Inspired by the relation between matrix rank and nuclear norm, Liu \textit{et al.}~\cite{Liu2013} propose a convex surrogate for multi-rank (also known as the Tucker rank or tensor $n$-rank), which is referred to as the SNN. Soon after, this tractable measure of the tensor rank has been successfully applied to various practical problems (see, e.g.,~\cite{Tomioka2010,Gandy2011} and reference therein). Besides the empirical studies, some progress on recovery theory has been achieved at the same time. Tomioka \textit{et al.}~\cite{Tomioka2011} conduct a statistical analysis for tensor decomposition and provide the first theoretical guarantee for SNN minimization. This result was significantly enhanced in a later study~\cite{Mu2013}, which not only proves that the complexity bound obtained in~\cite{Tomioka2011} is tight when employing the SNN as the convex surrogate, but also proposes a simple improvement that works much better for high-order tensors. Unfortunately, all the researches assume Gaussian measurements, while
in practice the problem settings are more often similar to matrix completion problems~\cite{Chen2013,Candes2009,Recht2011}. To fill the gap, Zhang and Aeron~\cite{Zhang2017} derive theoretical performance bounds for the algorithm proposed in~\cite{Zhang2014} for third-order tensor recovery from limited sampling using the t-SVD algebraic framework. They prove that by solving a convex optimization problem, which minimizes tubal nuclear norm as a convex approximation of the tubal rank, one can exactly recover a $n_1 \times n_2 \times n_3$ tensor with tubal rank $r$, given $O(rn_1n_3\log((n_1+n_2)n_3))$ random samples when certain tensor incoherence conditions are satisfied.

\subsection{Tensor Robust Principal Component Analysis}
\label{sec2:sub2}

The goal of TRPCA problem is to learn a target tensor that is a superposition of the low-rank component and a sparse corruption component from observations. This problem, after first being proposed in~\cite{Li2010}, has been extensively investigated theoretically in~\cite{Shah2015,Gu2014} and algorithmically in~\cite{Goldfarb2013,Huang2014,Tan2013}. Shah \textit{et al.}~\cite{Shah2015} consider robust CP decomposition based on a randomized convex relaxation formulation. Under their random sparsity model, the proposed algorithm provides guaranteed recovery as long as the number of non-zero entries per fiber is $O(\sqrt{n})$. Using the SNN as a convex relaxation for the multi-rank, Gu \textit{et al.}~\cite{Gu2014} provide perfect recovery of both components (with respective  nonasymptotic Frobenius-norm estimation error bound)  under restricted eigenvalue conditions. But these conditions are opaque and it is not clear regarding the level of sparsity that can be handled.

The rank sparsity tensor decomposition (RSTD) algorithm~\cite{Li2010} applies variable-splitting to both components, and utilizes a classic BCD algorithm to solve an unconstrained problem obtained by relaxing all the constraints as quadratic penalty terms. This method has many parameters to tune and does not have a iteration complexity guarantee. The Multi-linear Augmented Lagrange Multiplier (MALM) method~\cite{Tan2013} %is based on the ADMM algorithm and
divides the original TRPCA problem into independent robust principal component analysis (RPCA) problems~\cite{Candes2011}. This reformulation makes the final solution hard to be optimal since consistency among the auxiliary variables is not considered. In~\cite{Goldfarb2013}, convex and non-convex approaches derived from the ADMM algorithm, are introduced, but there are no guarantees on their recovery performance. Lu \textit{et al.}~\cite{Lu2016} propose a convex optimization, which is indeed a simple and elegant tensor extension of RPCA. They show that under certain incoherence conditions, the solution to the convex optimization perfectly recovers the low-rank and the sparse components, provided that the tubal rank of target tensor is not too large, and that corruption term is reasonably sparse.

\section{The t-SVD Algebraic Framework}
\label{sec3}

Throughout this paper, tensors are denoted by boldface Euler letters and matrices by boldface capital letters. Vectors are represented by boldface lowercase letters and scalars by lowercase letters. The field of real number and complex number are denoted as $\R$ and $\C$, respectively. For a third-order tensor $\ten{A} \in \R^{n_1 \times n_2 \times n_3}$, we denote its $(i, j, k)$-th entry as $\ten{A}_{ijk}$ and use the \textsc{Matlab} notation $\ten{A}(i,:,:)$, $\ten{A}(:,i,:)$ and $\ten{A}(:,:,i)$ to denote the $i$-th horizontal, lateral and frontal slice, respectively. Specifically, the front slice $\ten{A}(:,:,i)$ is denoted compactly as $\ten{A}^{(i)}$. $\ten{A}(i, j, :)$ denotes a tubal fiber oriented into the board obtained by fixing the first two indices and varying the third. Moreover, a tensor tube of size $1 \times 1 \times n_3$ is denoted as $\tub{a}$ and a tensor column of size $n_1 \times 1 \times n_3$ is denoted as $\tc{b}$.

The inner product of $\mtx{A}$ and $\mtx{B}$ in $\C^{n_1 \times n_2}$ is given by $\<\mtx{A}, \mtx{B}\> = \tr(\mtx{A}^{H}\mtx{B})$, where $\mtx{A}^{H}$ denotes the conjugate transpose of $\mtx{A}$ and $\tr(\cdot)$ denotes the matrix trace. The inner product of $\ten{A}$ and $\ten{B}$ in $\C^{n_1 \times n_2 \times n_3}$ is defined as $\<\ten{A}, \ten{B}\> = \sum_{i=1}^{n_3}\<\mtx{A}^{(i)}, \mtx{B}^{(i)}\>$.

Some norms of vector, matrix and tensor are used. For a vector $\vct{v} \in \C^{n}$, the $\ell_2$-norm is $\|\vct{v}\|_2 = \sqrt{\sum_{i} |v_i|^2}$. The spectral norm of a matrix $\mtx{A} \in \C^{n_1 \times n_2}$ is denoted as $\|\mtx{A}\| = \max_{i}\sigma_{i}(\mtx{A})$, where $\sigma_{i}(\mtx{A})$'s are the singular values of $\mtx{A}$. The matrix nuclear norm is $\|\mtx{A}\|_{\ast} = \sum_{i}\sigma_{i}(\mtx{A})$. For a tensor $\ten{A}$, we denote the $\ell_1$-norm as $\|\ten{A}\|_1 = \sum_{ijk}|\ten{A}_{ijk}|$, the infinity norm as $\|\ten{A}\|_{\infty} = \max_{ijk}|\ten{A}_{ijk}|$ and the Frobenius norm as $\|\ten{A}\|_F = \sqrt{\sum_{ijk}|\ten{A}_{ijk}|^2}$. It is easy to verify that these norms reduce to the corresponding vector or matrix norms if $\ten{A}$ is a vector or a matrix.

$\widehat{\ten{A}}$ represents a third-order tensor obtained by taking the Discrete Fourier Transform (DFT) of all the tubes along the third dimension of $\ten{A}$, i.e.,
\begin{equation}
\vv(\widehat{\ten{A}}(i, j, :)) = \mathcal{F}(\vv(\ten{A}(i, j, :))),
\label{eq1}
\end{equation}
where $\vv$ is the vectorization operator that maps the tensor tube to a vector, and $\mathcal{F}$ stands for the DFT. For compactness, we will denote the Fast Fourier Transform (FFT) along the third dimension by $\widehat{\ten{A}} = \fft(\ten{A}, [], 3)$. In the same fashion, one can also compute $\ten{A}$ from $\widehat{\ten{A}}$ via $\ifft(\widehat{\ten{A}}, [], 3)$ using the inverse FFT operation along the third-dimension. For sake of brevity, we direct the interested readers to~\cite{Kilmer2011,Kilmer2013}.

After introducing the tensor notation and terminology, we give the basic definitions on t-SVD and outline the associated algebraic framework from~\cite{Kilmer2011,Kilmer2013,Zhang2014,Lu2016,Zhang2017}, which serve as the foundation for our analysis in next section.
\begin{definition}[t-product\cite{Kilmer2011}]
The t-product $\ten{A} \ast \ten{B}$ of $\ten{A} \in \R^{n_1 \times n_2 \times n_3}$ and $\ten{B} \in \R^{n_2 \times n_4 \times n_3}$ is a tensor $\ten{C} \in \R^{n_1 \times n_4 \times n_3}$ whose $(i,j)$th tube $\tub{c}_{ij}$ is given by
\begin{equation}
\tub{c}_{ij} = \ten{C}(i,j,:) = \sum_{k=1}^{n_2} \ten{A}(i, k, :) \ast \ten{B}(k, j, :),
\label{eq2}
\end{equation}
where $\ast$ denotes the circular convolution between two tubes of same size.
\label{def1}
\end{definition}
Note that a third-order tensor of size $n_1 \times n_2 \times n_3$ can be regarded as an $n_1 \times n_2$ matrix with each entry as a tube lies in the third dimension. Hence, the t-product of two tensors is analogous to matrix-matrix multiplication, expect that the multiplication operation between the scalars is replaced by circular convolution between the tubes. This new perspective has endowed multidimensional data arrays with an advantageous representation in real-world applications~\cite{Kilmer2011,Martin2013,Kilmer2013,Zhang2014}.

\begin{definition}[Tensor conjugate transpose\cite{Kilmer2011}]
The conjugate transpose of a tensor $\ten{A} \in \R^{n_1 \times n_2 \times n_3}$ is the tensor $\ten{A}^{H} \in \R^{n_2 \times n_1 \times n_3}$ obtained by conjugate transposing each of the frontal slice and then reversing the order of transposed frontal slices 2 through $n_3$, i.e.,
\begin{eqnarray}
\left(\ten{A}^H\right)^{(1)} & = & \left(\ten{A}^{(1)}\right)^H, \nonumber\\
\left(\ten{A}^H\right)^{(i)} & = & \left(\ten{A}^{(n_3+2-i)}\right)^H, \,\,\, i = 2, \dots, n_3. \nonumber
\end{eqnarray}
\label{def2}
\end{definition}

\begin{definition}[Block diagonal form of third-order tensor\cite{Kilmer2011}]
Let $\widebar{\mtx{A}}$ to be the block diagonal matrix of the tensor $\ten{A} \in \R^{n_1 \times n_2 \times n_3}$ in the Fourier domain, namely,
\begin{equation}
\widebar{\mtx{A}}  = \bdiag(\widehat{\ten{A}}) = \left[
\begin{array}{llll}
\widehat{\ten{A}}^{(1)} & & & \\
 & \widehat{\ten{A}}^{(2)} & & \\
 & & \ddots &\\
 & & & \widehat{\ten{A}}^{(n_3)}
\end{array}
\right] \in \C^{n_1 n_3 \times n_2 n_3}.
\label{eq3}
\end{equation}
\label{def3}
\end{definition}
It is easy to discover that the block diagonal matrix of $\ten{A}^{H}$ is equal to the conjugate transpose of the block diagonal matrix of $\ten{A}$,
\begin{equation}
\widebar{\mtx{A}^{H}} = {\widebar{\mtx{A}}}^{H}.
\label{eq4}
\end{equation}
The following facts will be used through out the paper. For any tensor $\ten{A} \in \R^{n_1 \times n_2 \times n_3}$ and $\ten{B} \in \R^{n_2 \times n_4 \times n_3}$, we have
\begin{equation}
\ten{A} \ast \ten{B} = \ten{C} \Leftrightarrow \widebar{\mtx{A}}\,\,\widebar{\mtx{B}} = \widebar{\mtx{C}},
\label{eq5}
\end{equation}
and the inner product of two tensor has the following property
\begin{equation}
\<\ten{A}, \ten{B}\> = \frac{1}{n_3}\<\widebar{\mtx{A}}, \widebar{\mtx{B}}\> \in \R,
\label{eq6}
\end{equation}
where $1/n_3$ comes from the normalization constant of the FFT. The inner product produces a real-valued scalar due to the conjugate symmetric property of the FFT.

\begin{definition}[Identity tensor\cite{Kilmer2011}]
The identity tensor $\ten{I} \in \R^{n \times n \times n_3}$ is defined to be a tensor whose first frontal slice $\ten{I}^{(1)}$ is the $n \times n$ identity matrix and whose other frontal slices $\ten{I}^{(i)}, i = 2, \dots, n_3$ are zero matrices.
\label{def4}
\end{definition}

\begin{definition}[Orthogonal tensor\cite{Kilmer2011}]
A tensor $\ten{Q} \in \R^{n \times n \times n_3}$ is orthogonal if it satisfies
\begin{equation}
\ten{Q}^{H} \ast \ten{Q} = \ten{Q} \ast \ten{Q}^{H} =\ten{I},
\label{eq7}
\end{equation}
where $\ten{I}$ is the identity tensor of size $n \times n \times n_3$.
\label{def5}
\end{definition}

\begin{definition}[f-diagonal tensor\cite{Kilmer2011}]
A tensor $\ten{A}$ is called f-diagonal if each frontal slice $\ten{A}^{(i)}$ is a diagonal matrix.
\label{def6}
\end{definition}

\begin{figure}[!t]
\centering
\includegraphics[width=0.6\textwidth]{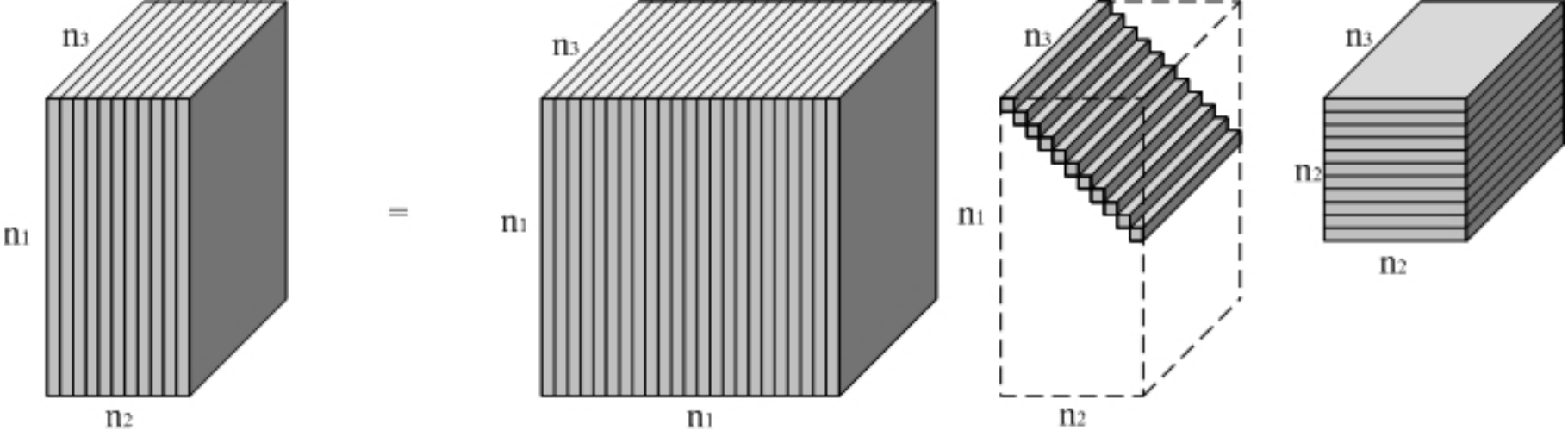}
\caption{Illustration of the t-SVD for a $n_1 \times n_2 \times n_3$ tensor~\cite{Zhang2017}.}
\label{fig1}
\end{figure}

\begin{algorithm}[!t]
\caption{t-SVD for Third-Order Tensors~\cite{Kilmer2011}}
\label{alg1}
\textbf{Input:} $\ten{A} \in \R^{n_1\times n_2 \times n_3}$.\\
\textbf{Output:} $\ten{U} \in \R^{n_1\times n_1 \times n_3}$, $\ten{S} \in \R^{n_1\times n_2 \times n_3}$, \\
$\ten{V} \in \R^{n_2 \times n_2 \times n_3}$.
\\\vspace{-0.4cm}
\begin{algorithmic}[1]
\STATE {$\hat{\ten{A}} = \fft(\ten{A}, [], 3)$;}
\FOR{$i = 1, \dots, n_3$}
\STATE{$[\mtx{U}, \mtx{S}, \mtx{V}] = \svd(\hat{\ten{A}}^{(i)})$};
\STATE{$\hat{\ten{U}}^{(i)} = \mtx{U}$, $\hat{\ten{S}}^{(i)} = \mtx{S}$, $\hat{\ten{V}}^{(i)} = \mtx{V}$};
\ENDFOR
\STATE{$\ten{U} = \ifft(\hat{\ten{U}}, [], 3)$, $\ten{S} = \ifft(\hat{\ten{S}}, [], 3)$,\\ $\ten{V} = \ifft(\hat{\ten{V}}, [], 3)$}
\end{algorithmic}
\end{algorithm}

The aforementioned notions allow us to propose the following tensor factorization.
\begin{definition}[Tensor Singular Value Decomposition: t-SVD\cite{Kilmer2011}]
For $\ten{A} \in \R^{n_1 \times n_2 \times n_3}$, the t-SVD of $\ten{A}$ is given by
\begin{equation}
\ten{A} = \ten{U} \ast \ten{S} \ast \ten{V}^{H},
\label{eq8}
\end{equation}
where $\ten{U} \in \R^{n_1 \times n_1 \times n_3}$ and $\ten{V} \in \R^{n_2 \times n_2 \times n_3}$ are orthogonal tensors, and $\ten{S} \in \R^{n_1 \times n_2 \times n_3}$ is a f-diagonal tensor, respectively. The entries in $\ten{S}$ are called the singular tubes of $\ten{A}$.
\label{def7}
\end{definition}
Figure~\ref{fig1} illustrates the t-SVD for a $n_1 \times n_2 \times n_3$ tensor, which can be obtained by computing matrix SVDs in the Fourier domain as shown in Algorithm~\ref{alg1}. Based on the t-SVD, we can derive the following notion of tensor rank.
\begin{definition}[Tubal multi-rank and tubal rank\cite{Kilmer2013}]
The \textbf{tubal multi-rank} of a tensor $\ten{A} \in \R^{n_1 \times n_2 \times n_3}$ is a vector $\vct{r} \in \R^{n_3}$ with its $i$-th entry as the rank of the $i$-th frontal slice, i.e., $r_i = \rank(\hat{\ten{A}}^{(i)})$. The tensor \textbf{tubal rank}, denoted as $\rank_t(\ten{A})$, is defined as the number of nonzero singular tubes of $\ten{S}$, where $\ten{S}$ comes from the t-SVD of $\ten{A} = \ten{U} \ast \ten{S} \ast \ten{V}^{H}$. That is
\begin{equation}
\rank_t(\ten{A}) = \#\{i: \ten{S}(i, i, :) \neq \vct{0}\} = \max_{i} r_i.
\label{eq9}
\end{equation}
\label{def8}
\vspace{-2em}
\end{definition}

\begin{remark}
The tubal rank has some interesting properties that are similar to the matrix rank, that is, $\rank_t(\ten{A}) \leq \min(n_1, n_2)$ for $\ten{A} \in \R^{n_1 \times n_2 \times n_3}$, and $\rank_t(\ten{A} \ast \ten{B}) \leq \min(\rank_t(\ten{A}), \rank_t(\ten{B}))$.
\label{remark1}
\end{remark}
It is usually sufficient to compute the skinny version of t-SVD using the tensor tubal rank. In details, suppose $\ten{A} \in \R^{n_1 \times n_2 \times n_3}$ has tensor tubal rank $r$, then the skinny t-SVD of $\ten{A}$ is given by
\begin{equation}
\ten{A} = \ten{U} \ast \ten{S} \ast \ten{V}^{H}
\label{eq10}
\end{equation}
where $\ten{U} \in \R^{n_1 \times r \times n_3}$ and $\ten{V} \in \R^{n_2 \times r \times n_3}$ satisfying $\ten{U}^{H} \ast \ten{U} = \ten{I}$, $\ten{V}^{H} \ast \ten{V} = \ten{I}$, and $\ten{S} \in \R^{r \times r \times n_3}$ is a f-diagonal tensor. This skinny t-SVD will be used throughout the paper unless otherwise stated.

\begin{definition}[Tubal nuclear norm\cite{Lu2016}]
The tubal nuclear norm of a tensor $\ten{A} \in \R^{n_1 \times n_2 \times n_3}$, denoted as $\|\ten{A}\|_{\textup{TNN}}$, is the average of the nuclear norm of all the frontal slices of $\widehat{\ten{A}}$, i.e., $\|\ten{A}\|_{\textup{TNN}} = \frac{1}{n_3} \sum_{i=1}^{n_3} \|{\widehat{\ten{A}}}^{(i)}\|_{\ast}$.
\label{def9}
\end{definition}

\begin{remark}
The norm defined above is also named tensor nuclear norm in~\cite{Zhang2014,Lu2016}. But there is another norm with the same name proposed in~\cite{Liu2013,Tomioka2010}. To differentiate these two tensor norms, we refer to Definition~\ref{def8} as tubal nuclear norm in this paper. With the factor $1/n_3$, it is different from the earlier definition given by~\cite{Zhang2014,Zhang2017}, and is important for our analysis in theory.
\label{remark2}
\end{remark}

\begin{figure}[!t]
\centering
\includegraphics[width=0.3\textwidth]{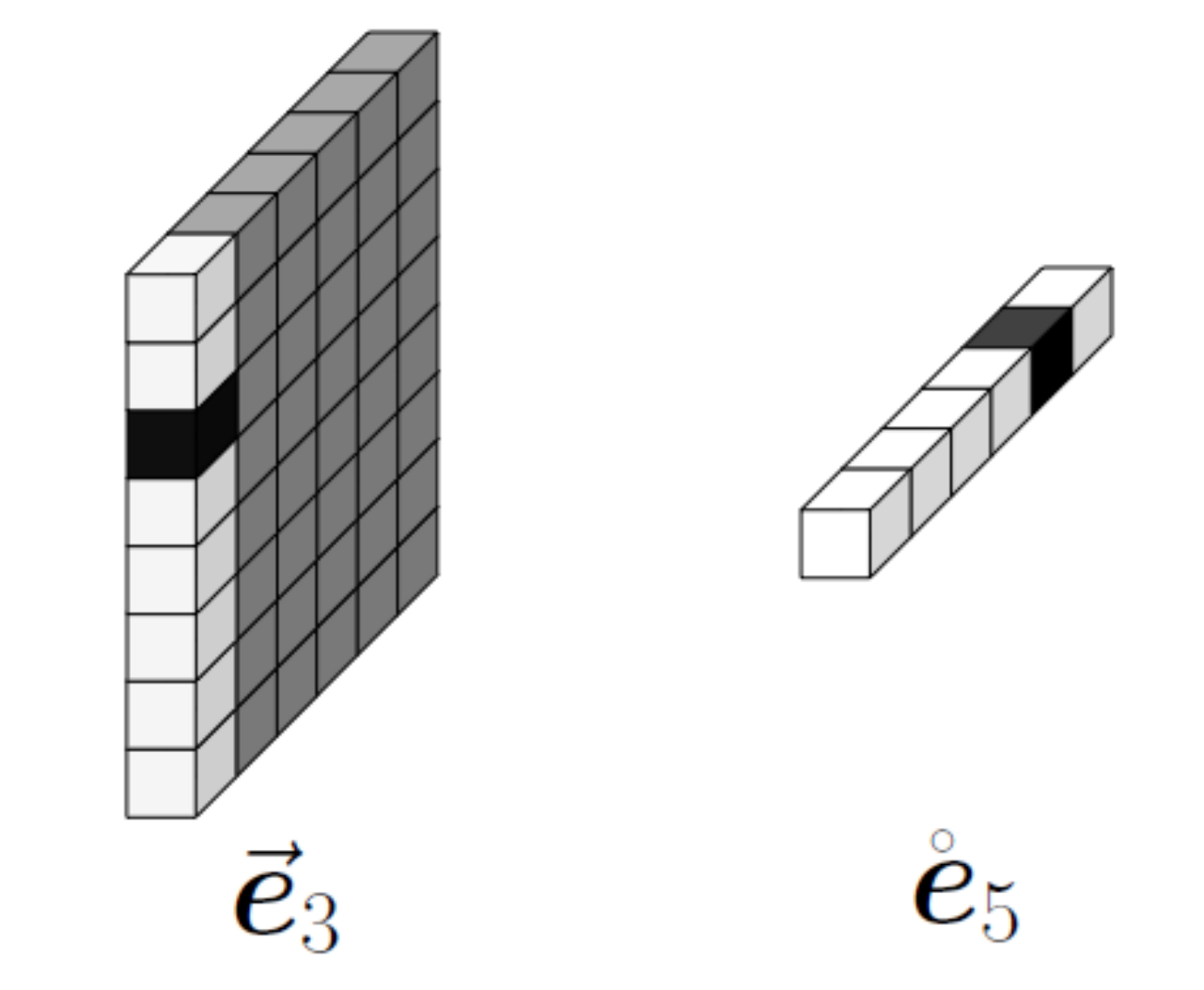}
\caption{The column basis $\tc{e}_3$ and tube basis $\tub{e}_5$~\cite{Zhang2017}. The black cubes are 1, gray and white cubes are 0. The white cubes stand for the potential entries that could be 1.}
\label{fig2}
\end{figure}

We introduce two kinds of tensor basis that are illustrated in Figure~\ref{fig2} and will be exploited to derive our main results.

\begin{definition}[Tensor basis\cite{Zhang2017}]
The \textbf{column basis}, denoted as $\tc{e}_i$, is a tensor of size $n_1 \times 1 \times n_3$ with its $(i, 1, 1)$th entry equaling to 1 and the rest equaling to 0. The nonzero entry 1 will only appear at the first front slice of $\tc{e}_i$. Naturally its conjugate transpose $\tc{e}_i^H$ is called \textbf{row basis}. The \textbf{tube basis}, denoted as $\tub{e}_k$, is a tensor of size $1 \times 1 \times n_3$ with its $(1,1,k)$th entry equaling to 1 and the rest equaling to 0.
\label{def10}
\end{definition}
One can obtain a unit tensor $\ten{E}_{ijk}$ with only the $(i,j,k)$th entry equaling to 1 through $\ten{E}_{ijk} = \eijk$. For a third-ord tensor $\ten{A} \in \R^{n_1 \times n_2 \times n_3}$, we can decompose it as $ \ten{A} = \sum_{ijk} \<\ten{E}_{ijk}, \ten{A}\>\ten{E}_{ijk} = \sum_{ijk} \ten{A}_{ijk} \ten{E}_{ijk}$.

\begin{definition}[Tensor spectral norm~\cite{Zhang2017}]
The tensor spectral norm of $\ten{A} \in \R^{n_1 \times n_2 \times n_3}$, denoted as $\|\ten{A}\|$, is defined as $\|\ten{A}\| = \|\widebar{\mtx{A}}\|$. In other words, the tensor spectral norm of $\ten{A}$ equals to the matrix spectral norm of its block diagonal form $\widebar{\mtx{A}}$.
\label{def11}
\end{definition}

\begin{remark}
If we define the tubal average rank as $\rank_{a}(\ten{A}) = \frac{1}{n_3} \sum^{n_3}_{i=1} \rank(\widehat{\ten{A}}^{(i)})$, it can be proved that the tubal nuclear norm is the convex envelop of the tubal average rank within the unit ball of the tensor spectral norm.
\label{remark3}
\end{remark}

\begin{definition}[Tensor operator norm\cite{Zhang2017}]
Suppose $\mathcal{L}$ is a tensor operator, then its operator norm is defined as
\begin{equation}
\|\mathcal{L}\|_{\textup{op}} = \sup_{\|\ten{X}\|_F \leq 1} \|\mathcal{L}(\ten{A})\|_F.
\label{eq11}
\end{equation}
\label{def12}
\end{definition}

\begin{remark}
This definition is consistent with the matrix case. Spectral norm is equivalent to the operator norm if the tensor operator $\mathcal{L}$ can be represented as a tensor $\ten{L}$ t-product $\ten{A}$. In other words, $\|\mathcal{L}\|_{\textup{op}} = \|\ten{L}\|$ if $\mathcal{L}(\ten{A}) = \ten{L} \ast \ten{A}$.
\label{remark4}
\end{remark}

\section{Theoretical Analysis}
\label{sec4}

Let us consider the RTC problem formally. Suppose we are given a third-order tensor $\ten{L}_0$ having low tubal rank and corrupted by a sparse term $\ten{E}_0$. Here, both $\ten{L}_0$ and $\ten{E}_0$ are of arbitrary magnitude. We do not know the tubal rank of $\ten{L}_0$. Furthermore, we have no idea about the locations of the nonzero entries of $\ten{E}_0$, not even how many there are. Can we recover $\ten{L}_0$ accurately (perhaps even exactly) and efficiently from an observed subset\footnote{In this situation, it is impossible to exactly recover $\ten{E}_0$ (some of its entries are simply not observed!), unless the observed set is identical to the support of $\ten{E}_0$.} of the noisy data $\ten{X} = \ten{L}_0 + \ten{E}_0$?

Mathematically, the problem can be represented by
\begin{equation}
\min_{\ten{L},\,\ten{E}} \rank_t(\ten{L}) + \lambda\|\ten{E}\|_{0},\,\,\, \textup{s.t.},\,  \mathcal{P}_\Omega(\ten{L} + \ten{E}) = \mathcal{P}_\Omega(\ten{X}),
\label{eq12}
\end{equation}
where $\lambda$ is a penalty parameter and $\PO$ is a linear projection such that the entries in the set $\Omega$ are given while the remaining entries are missing.
The optimization problem in (\ref{eq12}) is generally NP-hard due to the discrete nature of the tubal-rank function and the $\ell_0$ pseudo-norm which counts the number of nonzero entries of $\ten{E}$. Replacing these two terms by their convex surrogates, namely, tubal nuclear norm and $\ell_1$-norm respectively, leads to the following convex optimization problem
\begin{equation}
\min_{\ten{L},\,\ten{E}} \|\ten{L}\|_{\textup{TNN}} + \lambda\|\ten{E}\|_{1},\,\,\, \textup{s.t.},\,  \mathcal{P}_\Omega(\ten{L} + \ten{E}) = \mathcal{P}_\Omega(\ten{X}).
\label{eq13}
\end{equation}
Our model (\ref{eq13}) is equivalent to the following TC problem when there is no corruption, i.e., $\ten{E} = \mtx{0}$,
\begin{equation}
\min_{\ten{L}} \|\ten{L}\|_{\textup{TNN}} \,\,\, \textup{s.t.},\,  \mathcal{P}_\Omega(\ten{L}) = \mathcal{P}_\Omega(\ten{X}),
\label{eq14}
\end{equation}
and it reduces to the following TRPCA problem when $\Omega$ is the entire set of indices,
\begin{equation}
\min_{\ten{L},\,\ten{E}} \|\ten{L}\|_{\textup{TNN}} + \lambda\|\ten{E}\|_{1},\,\,\, \textup{s.t.},\,  \ten{L} + \ten{E} = \ten{X}.
\label{eq15}
\end{equation}

\subsection{Tensor Incoherence Conditions}
\label{sec4:sub1}

As discussed in~\cite{Huang2014,Lu2016,Zhang2017}, exact recovery is hopeless if most entries of $\ten{X}$ are equal to zero. Suppose $\ten{X}$ is both low-rank and sparse, e.g., $\ten{X} = \ten{E}_{111}$ ($\ten{X}_{ijk} = 1$ when $i = j= k= 1$ and zeros everywhere else), then we are not able to identify the low-rank tensor $\ten{L}_0$ in these cases. To make the problem meaningful, we need some incoherence conditions on $\ten{L}_0$ to ensure that it is not sparse.

\begin{definition}[Tensor Incoherence Conditions] Assume that $\rank_t(\ten{L}_0) = r$ and its skinny t-SVD is $\ten{L}_0 = \ten{U} \ast \ten{S} \ast \ten{V}^{H}$. $\ten{L}_0$ is said to satisfy the tensor incoherence conditions with parameter $\mu > 0$ if
\begin{align}
\max_{i=1, \dots, n_1} \|\ten{U}^{H} &\ast \tc{e}_i\|_F \leq \sqrt{\frac{\mu r}{n_1}}, \label{eq16}\\
\max_{j=1, \dots, n_2} \|\ten{V}^{H} &\ast \tc{e}_j\|_F \leq \sqrt{\frac{\mu r}{n_2}}, \label{eq17}
\end{align}
and \begin{equation}
\|\ten{U} \ast \ten{V}^{H} \|_{\infty} \leq \sqrt{\frac{\mu r}{n_1 n_2 n_3}}.
\label{eq18}
\end{equation}
\label{def13}
\end{definition}
Note that the smallest $\mu$ is equal to 1, which is achieved when each tensor column $\tc{u}_i = \ten{U}(:, i, :)$ or  $\tc{v}_j = \ten{V}(:, j, :)$ has entries with magnitude $1/\sqrt{n_1n_3}$. The largest possible value of $\mu$ is $\min(n_1, n_2)/r$ when one of the tensor columns of $\ten{U}$ (or $\ten{V}$ respectively) is equal to the standard tensor column basis $\tc{e}_i$ (or $\tc{e}_j$ respectively). With low $\mu$, the tensor columns $\ten{U}(:, i, :)$ and $\ten{V}(:, i, :), i = 1, 2, \dots, n_3$ are supposed to be sufficiently spread, i.e., be uncorrelated with the tensor basis, which means that each entry of $\ten{L}_0$ carries approximately same amount of information.

It is not difficult to discover that the incoherence conditions (\ref{eq16})-(\ref{eq18}) reduce to the regular matrix incoherence conditions~\cite{Candes2009,Recht2011,Candes2011,Li2013} when $n_3=1$. According to~\cite{Chen2013}, we name (\ref{eq16}) and (\ref{eq17}) as tensor standard incoherence conditions, and (\ref{eq18}) as tensor joint incoherence condition, respectively. With a factor $1/\sqrt{n_3}$, our incoherence conditions are much weaker than those given by~\cite{Lu2016}. Zhang and Aeron~\cite{Zhang2017} indicate that the joint incoherence condition (\ref{eq18}) is unnecessary for tensor completion, while we get the same conclusion in an alternative way (see the proof of Corollary~\ref{cor1}). However, it is unavoidable for obtaining exact solution to the TRPCA and RTC problems, as shown in our analysis.

Another identifiability issue arises if the corruption term $\ten{E}_0$ has low tubal rank. This can be avoided by assuming that the support of $\ten{E}_0$ is distributed uniformly at random.

\subsection{Main Results}
\label{sec4:sub2}

Now we present our main results. For convenience, we denote $n_{(1)} = \max(n_1, n_2)$ and $n_{(2)} = \min(n_1, n_2)$.

\begin{theorem}
Suppose $\ten{L}_0 \in \R^{n_1 \times n_2 \times n_3}$ obeys (\ref{eq16})-(\ref{eq18}), and the observation set $\Omega$ is uniformly distributed among all sets of cardinality $m = \rho n_1 n_2 n_3$. Also suppose that each observed entry is independently corrupted with probability $\gamma$. Then, there exist universal constants $c_1, c_2 > 0$ such that with probability at least $1 - c_1(n_{(1)}n_3)^{-c_2}$, the recovery of $\ten{L}_0$ with $\lambda = 1/\sqrt{\rho n_{(1)}n_3}$ is exact, provided that
\begin{equation}
r \leq \frac{c_rn_{(2)}}{\mu (\log(n_{(1)}n_3))^2}\, \,\,\, \textup{and} \,\,\, \,\gamma \leq c_\gamma
\label{eq19}
\end{equation}
where $c_r$ and $c_\gamma$ are two positive constants.
\label{the1}
\end{theorem}
The proof of Theorem~\ref{the1} will be given in Section~\ref{sec5}. The theorem tells us that the target tensor $\ten{L}_0$ whose singular vectors $\ten{U}(:, i, :)$ and $\ten{V}(:, j, :)$ are reasonably spread, can be exactly recovered with probability nearly one from a subset of its entries even if they are arbitrarily corrupted. All we require is that the tubal rank of the tensor $\ten{L}_0$ is not too large, to be exact, on the order of $n_{(2)}/(\mu(\log(n_{(1)}n_3))^2)$ and the corruption term $\ten{E}_0$ is sufficiently sparse. We would like to emphasize that the only \lq\lq random distribution\rq\rq~in our assumptions concerns the locations of the nonzero entries of $\ten{E}_0$, but not on their magnitudes or signs. Another remarkable fact is that there is no tuning parameter in our model.

As special cases of problem (\ref{eq13}), the recovery guarantees for problems (\ref{eq14}) and (\ref{eq15}) are naturally implied by Theorem~\ref{the1} as in the following corollaries
\begin{corollary}
Suppose $\ten{L}_0 \in \R^{n_1 \times n_2 \times n_3}$ obeys (\ref{eq16}) and (\ref{eq17}) and $m$ entries of $\ten{L}_0$ are observed with locations sampled uniformly at random, then there exist universal constants $c_0, c_1, c_2 > 0$ such that if
\begin{equation}
m \geq c_0 \mu r n_{(1)} n_3(\log(n_{(1)}n_3))^2,
\label{eq20}
\end{equation}
$\ten{L}_0$ is the unique minimizer to (\ref{eq14}) with probability at east $1 - c_1(n_{(1)}n_3)^{-c_2}$.
\label{cor1}
\end{corollary}

\begin{corollary}
Suppose $\ten{L}_0 \in \R^{n_1 \times n_2 \times n_3}$ obeys (\ref{eq16})-(\ref{eq18}) and $\ten{E}_0$ has support uniformly distributed with probability $\gamma$. Then, there exist universal constants $c_1, c_2 > 0$ such that with probability at least $1 - c_1(n_{(1)}n_3)^{-c_2}$, $(\ten{L}_0, \ten{E}_0)$ is the unique minimizer to (\ref{eq15}) with $\lambda = 1/\sqrt{n_{(1)}n_3}$, provided that
\begin{equation}
r \leq \frac{c_rn_{(2)}}{\mu (\log(n_{(1)}n_3))^2}\, \,\,\, \textup{and} \,\,\, \,\gamma \leq c_\gamma
\label{eq21}
\end{equation}
where $c_r$ and $c_\gamma$ are two positive constants.
\label{cor2}
\end{corollary}

\subsection{Connections with Prior Work}
\label{sec4:sub3}

Since the t-product for third-order tensors and the tubal nuclear norm reduce to the standard matrix multiplication and the matrix nuclear norm respectively when $n_3 = 1$, our model is a simple and elegant extension of the RMC formulation~\cite{Shang2014,Li2013}. In this sense, Theorem 1.3 in~\cite{Li2013} can be viewed as a special case of Theorem~\ref{the1}.

We mention earlier two related works~\cite{Lu2016,Zhang2017}, which are the most similar to our study. They also define the notion of tensor incoherence using the t-SVD algebraic framework and propose sufficient conditions for convex programming to achieve guaranteed recovery. However, they simply focus on problem (\ref{eq14}) and (\ref{eq15}), which are both special cases of our model (\ref{eq13}).

Moreover, the number of observed entries $m$ required by Corollary~\ref{cor1} differs slightly from that suggested by Theorem III.1 in~\cite{Zhang2017} with a logarithm factor $O(\log(n_{(1)}n_3))$. For low-rank matrix recovery, the sampling complexity and recovery guarantees derived from different sampling models are consistent, with only a slight change of the constant factor~\cite{Candes2010,Gross2010}. We expect them to be the same in the tensor case as well. So, we can think that the two results are in good agreement with each other, seeing that the Bernoulli sampling is exploited in~\cite{Zhang2017} and the random sampling without replacement is adopted in this paper. Although Corollary~\ref{cor2} is almost identical to Theorem 3.1 proposed in~\cite{Lu2016}, we get the result under the tensor incoherence conditions that are much weaker.

\section{Proof of Theorem~\ref{the1}}
\label{sec5}

In this section, we provide the detailed proof of Theorem~\ref{the1}. The idea is to employ convex analysis to derive conditions in which one can check whether the pair $(\ten{L}, \ten{E})$ is the unique minimum solution to (\ref{eq13}), and to explicity show that such conditions are met with overwhelming probability in the conditions of Theorem~\ref{the1}.

Our proof follows closely the proofs presented in~\cite{Li2013,Huang2014}, where the main tools, such as the Non-commutative Bernstein Inequality (NBI) and the \textit{golfing scheme}, are also helpful in the derivation of our results. As there are subtle differences, our proof is not a fairly straightforward adaption. In problem (\ref{eq13}), the corrupted observations are randomly sampled in the original domain, while the tubal nuclear norm is defined in the Fourier domain. Therefore, the proofs of Lemma~\ref{lem1}-\ref{lem4} and Theorem~\ref{the3} need to additionally consider the properties of Fourier transformation and block circulant matrix.

\subsection{Sampling Schemes and Model Randomness}
\label{sec5:sub1}

The sampling strategy used in Theorem~\ref{the1} is the uniform sampling without replacement. There are other widely used sampling models, e.g., Bernoulli sampling, adaptive sampling and random sampling with replacement. To facilitate our proof, we will consider \textit{i.i.d. Bernoulli-Rademacher model}. More precisely, we assume $\Omega = \{(i, j, k)|\delta_{ijk} = 1\}$ where the $\delta_{ijk}$'s are i.i.d. Bernoulli variables taking value one with probability $\rho$ and zero with probability $1-\rho$. Such a Bernoulli sampling is denoted by $\Omega \sim \textup{Ber}(\rho)$ for short. As a proxy for uniform sampling, the probability of failure under Bernoulli sampling with $\rho = \frac{m}{n_1n_2n_3}$ closely approximates the probability of failure under uniform sampling.

Let a subsect $\Lambda \subset \Omega$ be the corrupted entries of $\ten{L}_0$ and $\Gamma \subset \Omega$ be locations where data are available and clean. In a standard Bernoulli model, we suppose that
\begin{displaymath}
\Omega \sim \textup{Ber}(\rho), \,\,\,
\Lambda \sim \textup{Ber}(\gamma\rho), \,\,\,
\Gamma \sim \textup{Ber}((1 - \gamma)\rho),
\end{displaymath}
and that the signs of the nonzero entries of $\ten{E}_0$ are deterministic. It has been shown to be much easier to work with a stronger assumption that the signs of the nonzero entries of $\ten{E}_0$ are independent symmetric $\pm1$ random variables (i.e., Rademacher random variables). We introduce two independent random subsets of $\Omega$
\begin{displaymath}
\Lambda' \sim \textup{Ber}(2\gamma\rho), \,\,\,
\Gamma' \sim \textup{Ber}((1 - 2\gamma)\rho),
\end{displaymath}
and it is convenient to think that $\ten{E}_0 = \Pg(\ten{E})$ for some fixed tensor $\ten{E}$. Consider a random sign tensor $\ten{M}$ with i.i.d. entries such that for any index $(i, j, k)$, $\PP(\ten{M}_{ijk} = 1) = \PP(\ten{M}_{ijk} = -1) = \frac{1}{2}$. Then $|\ten{E}|\circ \ten{M}$ has components with symmetric random signs. By introducing a new noise tensor $\ten{E}'_0 = \Pgp(|\ten{E}|\circ\ten{M})$ and using the standard derandomization theory (e.g., Theorem 2.3 in~\cite{Candes2011}), we can assert that

\begin{theorem}
Suppose $\ten{L}_0$ obeys the conditions of Theorem~\ref{the1} and that $\ten{E}_0$ and $\ten{E}'_0$ are given as above. If the recovery of $(\ten{L}_0, \ten{E}'_0)$ is exact with high probability, it is also exact with at least the same probability for the model with input data $(\ten{L}_0, \ten{E}_0)$.
\label{the2}
\end{theorem}
Therefore from now on, we can equivalently consider
\begin{equation}
\Lambda \sim \textup{Ber}(2\gamma\rho), \,\,\,
\Gamma \sim \textup{Ber}((1 - 2\gamma)\rho),
\label{eq22}
\end{equation}
for the locations of nonzero and zero entries of $\ten{E}_0$ respectively, and assume that the nonzero entries have symmetric random signs.

\subsection{Supporting Lemmas}
\label{sec5:sub2}

Denote $T$ by the set
\begin{equation}
T = \big\{\ten{U} \ast \ten{Y}^{H} + \ten{W} \ast \ten{V}^{H} | \ten{Y} \in \R^{n_2 \times r \times n_3}, \ten{W} \in \R^{n_1 \times r \times n_3} \big\}, \nonumber
\end{equation}
and introduce two projections $\PT$ and $\PTc$ as follows,
\begin{equation}
\PT(\ten{Z}) =  \ten{U}\ast\ten{U}^{H}\ast\ten{Z} +\ten{Z}\ast\ten{V}\ast\ten{V}^{H} - \ten{U}\ast\ten{U}^{H}\ast\ten{Z}\ast\ten{V}\ast\ten{V}^{H}, \nonumber
\end{equation}
\begin{equation}
\PTc(\ten{Z}) =  (\ten{I} - \ten{U}\ast\ten{U}^{H})\ast\ten{Z}\ast(\ten{I} - \ten{V}\ast\ten{V}^{H}),~~~~~~~~~~~~~~~~~~~~~\nonumber
\end{equation}
where $\ten{I}$ is the identity tensor of appropriate size. It is easy to verify that $\<\PT(\ten{A}), \PTc(\ten{B})\> = 0$ for any $\ten{A}, \ten{B} \in \R^{n_1 \times n_2 \times n_3}$. Another observation is that $\|\PT(\eijk)\|_F \leq \sqrt{\frac{2\mu r}{n_{(2)}}}$ for any tensor of the form $\eijk$, a fact that we will use several times in the sequel.

Now we list several key lemmas which play a crucial role in the proof of Theorem~\ref{the1}.
\begin{lemma}
Suppose $\Omega \sim \textup{Ber}(\rho)$. Then with high probability,
\begin{equation}
\|\rho^{-1}\PT\PO\PT - \PT\|_{\textup{op}} \leq \epsilon,
\label{eq23}
\end{equation}
provided that $\rho \geq C_0 \epsilon^{-2} \frac{\mu r \log(n_{(1)}n_3)}{n_{(2)}}$ for some numerical constant $C_0 > 0$.
\label{lem1}
\end{lemma}

\begin{lemma}
Suppose $\ten{Z} \in \R^{n_1 \times n_2 \times n_3}$ is a fixed tensor, and $\Omega \sim \textup{Ber}(\rho)$. Then with high probability,
\begin{equation}
\|(\rho^{-1}\PT\PO\PT - \PT)\ten{Z}\|_{\infty} \leq \epsilon \|\ten{Z}\|_{\infty},
\label{eq24}
\end{equation}
provided that $\rho \geq C_0 \epsilon^{-2} \frac{\mu r \log(n_{(1)}n_3)}{n_{(2)}}$ for some numerical constant $C_0 > 0$.
\label{lem2}
\end{lemma}

\begin{lemma}
Suppose $\ten{Z} \in \R^{n_1 \times n_2 \times n_3}$ is a fixed tensor, and $\Omega \sim \textup{Ber}(\rho)$. Then with high probability,
\begin{equation}
\|(\OpId - \rho^{-1}\PO)\ten{Z}\| \leq C'_0\sqrt{\frac{n_{(1)}n_3\log(n_{(1)}n_3)}{\rho}}\|\ten{Z}\|_{\infty},
\label{eq25}
\end{equation}
provided that $\rho \geq C_0  \frac{\log(n_{(1)}n_3)}{n_{(2)}n_3}$ for some numerical constant $C_0, C'_0 > 0$.
\label{lem3}
\end{lemma}

\begin{lemma}~\cite{Lu2016}
For the $n_1 \times n_2 \times n_3$ Bernoulli sign tensor $\ten{M}$ whose entries are distributed as
\begin{equation}
\ten{M}_{ijk} = \left\{
\begin{array}{lll}
1, & \textup{w.p.}\,\,\rho/2,\\
0, & \textup{w.p.}\,\,1-\rho,\\
-1, & \textup{w.p.}\,\,\rho/2,
\end{array}
\right.
\label{eq26}
\end{equation}
there exists a function $\varphi(\rho)$ satisfying $\lim\limits_{\rho \rightarrow 0^{+}} \varphi(\rho) = 0$, such that the following statement holds with large probability
\begin{equation}
\|\ten{M}\| \leq \varphi({\rho}) \sqrt{n_{(1)}n_3}.
\label{eq27}
\end{equation}
\label{lem4}
\end{lemma}
The proofs of the first three Lemmas\footnote{Note that Lemma~\ref{lem1} and~\ref{lem2} have been proved in~\cite{Zhang2017} when $\epsilon = \frac{1}{2}$. Here, we provide the proofs for the general case.} can be found in Appendix~\ref{app1}, \ref{app2} and \ref{app3}. Note that Lemma~\ref{lem1}-\ref{lem3} reduce to their corresponding matrix versions when $n_3 = 1$. Lemma~\ref{lem4} recently introduced in~\cite{Lu2016}, which provides a upper bound for the spectral norm of the tensors consisting of Bernoulli sign variables, is of great importance in our later analysis.

\subsection{Dual Certificates}
\label{sec5:sub3}

We propose a simple condition for the pair $(\ten{L}_0, \ten{E}_0)$ to be the unique optimal solution to problem (\ref{eq13}). These conditions are stated in terms of a dual variable $\ten{Y}$, the existence of which certifies optimality.

\begin{theorem}
If there is a tensor $\ten{Y} \in \R^{n_1 \times n_2 \times n_3}$ obeying
\begin{equation}
\left\{
\begin{array}{lll}
\|\PT(\ten{Y} + \lambda\sgn(\ten{E}_0) - \ten{U} \ast \ten{V}^{H})\|_F \leq \frac{\lambda}{n_1n_2n^2_3} \\
\|\PTc(\ten{Y} + \lambda\sgn(\ten{E}_0))\| \leq \frac{1}{2} \\
\|\Pg(\ten{Y})\|_{\infty} \leq \frac{\lambda}{2}\\
\Pgc(\ten{Y}) = \mtx{0},
\end{array}
\right.
\label{eq28}
\end{equation}
where $\lambda = 1/\sqrt{\rho n_{(1)}n_3}$, then $(\ten{L}_0, \ten{E}_0)$ is the unique optimal solution to (\ref{eq13}) when $n_1, n_2, n_3$ are large enough.
\label{the3}
\end{theorem}

\begin{proof}
Let $f(\ten{L}, \ten{E}) = \|\ten{L}\|_{\textup{TNN}} + \lambda \|\ten{E}\|_1$. Given a feasible perturbation $(\ten{L}_0 + \ten{Z}, \ten{E}_0 - \PO(\ten{Z}))$, we will show that the objective value $f(\ten{L}_0 + \ten{Z}, \ten{E}_0 - \PO(\ten{Z}))$ is strictly greater than $f(\ten{L}_0, \ten{E}_0)$ unless $\ten{Z} = \mtx{0}$. Define the skinny t-SVD of $\PTc(\ten{Z})$ to be $\PTc(\ten{Z}) = \ten{U}_\perp \ast \ten{S}_\perp \ast \ten{V}^{H}_\perp$ and then we have $\|\ten{U} \ast \ten{V}^{H} + \ten{U}_\perp \ast \ten{V}^{H}_\perp\| = 1$. Note that
\begin{align}
\|\ten{L}_0 + \ten{Z}\|_{\text{TNN}} & \geq  \<\ten{U} \ast \ten{V}^{H} + \ten{U}_\perp \ast \ten{V}^{H}_\perp, \ten{L}_0 + \ten{Z}\> \nonumber \\
& =  \<\ten{U} \ast \ten{V}^{H}, \ten{L}_0\> + \<\ten{U}_\perp \ast \ten{V}^{H}_\perp, \PTc(\ten{Z})\> + \<\ten{U} \ast \ten{V}^{H}, \ten{Z}\> \nonumber\\
& =  \|\ten{L}_0\|_{\textup{TNN}} + \|\PTc(\ten{Z})\|_{\textup{TNN}} + \<\ten{U} \ast \ten{V}^{H}, \ten{Z}\>, \nonumber
\end{align}
and
\begin{align}
\|\ten{E}_0 - \PO(\ten{Z})\|_1 & = \|\PO(\ten{E}_0 -\ten{Z})\|_1  = \|\Pg(\ten{E}_0 - \ten{Z})\|_1 + \|\Pl(\ten{E}_0 - \ten{Z})\|_1  \nonumber\\
& =  \|\Pg(\ten{Z})\|_1 + \|\ten{E}_0 -\Pl(\ten{Z})\|_1 \nonumber\\
& \geq  \|\Pg(\ten{Z})\|_1 + \|\ten{E}_0\|_1 -\<\sgn(\ten{E}_0), \Pl(\ten{Z})\> \nonumber\\
& \geq  \|\Pg(\ten{Z})\|_1 + \|\ten{E}_0\|_1 -\<\sgn(\ten{E}_0), \ten{Z}\>, \nonumber
\end{align}
we have
\begin{align}
\Delta f & = f(\ten{L}_0 + \ten{Z}, \ten{E}_0 - \PO(\ten{Z})) - f(\ten{L}_0, \ten{E}_0) \nonumber \\
& = \|\ten{L}_0 + \ten{Z}\|_{\textup{TNN}} + \lambda\|\ten{E}_0 - \PO(\ten{Z})\|_1 - \|\ten{L}_0\|_{\textup{TNN}} - \lambda\|\ten{E}_0\|_1 \nonumber\\
& \geq  \|\PTc(\ten{Z})\|_{\textup{TNN}} + \lambda \|\Pg(\ten{Z})\|_1 - \<\lambda \sgn(\ten{E}_0) - \ten{U} \ast \ten{V}^{H}, \ten{Z}\> \nonumber\\
%& \geq  \|\PTc(\ten{Z})\|_{\textup{TNN}} + \lambda \|\Pg(\ten{Z})\|_1 - |\<\lambda \sgn(\ten{E}_0) - \ten{U} \ast \ten{V}^{H}, \ten{Z}\>| \nonumber\\
& \geq \|\PTc(\ten{Z})\|_{\textup{TNN}} + \lambda \|\Pg(\ten{Z})\|_1 -|\<\ten{Y} + \lambda \sgn(\ten{E}_0) - \ten{U} \ast \ten{V}^{H}, \ten{Z}\> - \<\ten{Y}, \ten{Z}\>| \nonumber\\
&\geq \|\PTc(\ten{Z})\|_{\textup{TNN}} + \lambda \|\Pg(\ten{Z})\|_1-\<\PT(\ten{Y} + \lambda \sgn(\ten{E}_0) - \ten{U} \ast \ten{V}^{H}), \PT(\ten{Z})\> \nonumber\\
&~~~~-\<\PTc(\ten{Y} + \lambda \sgn(\ten{E}_0)), \PTc(\ten{Z})\> - \<\Pg(\ten{Y}), \Pg(\ten{Z})\> \nonumber\\
& =  \|\PTc(\ten{Z})\|_{\textup{TNN}} + \lambda \|\Pg(\ten{Z})\|_1-\frac{1}{n_3}\<\overline{\PT(\ten{Y} + \lambda \sgn(\ten{E}_0) - \ten{U} \ast \ten{V}^{H})}, \overline{\PT(\ten{Z})}\> \nonumber\\
&~~~~- \frac{1}{n_3}\<\overline{\PTc(\ten{Y} + \lambda \sgn(\ten{E}_0))}, \overline{\PTc(\ten{Z})}\> - \<\Pg(\ten{Y}), \Pg(\ten{Z})\> \nonumber\\
& \geq  \|\PTc(\ten{Z})\|_{\textup{TNN}} + \lambda \|\Pg(\ten{Z})\|_1-\frac{1}{n_3}\|\overline{\PT(\ten{Y} + \lambda \sgn(\ten{E}_0) - \ten{U} \ast \ten{V}^{H})}\|_F \|\overline{\PT(\ten{Z})}\|_F \nonumber\\
&~~~~- \frac{1}{n_3}\|\overline{\PTc(\ten{Y} + \lambda \sgn(\ten{E}_0))}\| \|\overline{\PTc(\ten{Z})}\|_{\ast} - \|\Pg(\ten{Y})\|_{\infty}\|\Pg(\ten{Z})\|_1 \nonumber\\
& =  \|\PTc(\ten{Z})\|_{\textup{TNN}} + \lambda \|\Pg(\ten{Z})\|_1-\|\PT(\ten{Y} + \lambda \sgn(\ten{E}_0) - \ten{U} \ast \ten{V}^{H})\|_F \|\PT(\ten{Z})\|_F \nonumber\\
&~~~~- \|\PTc(\ten{Y} + \lambda \sgn(\ten{E}_0))\| \|\PTc(\ten{Z})\|_{\ast} - \|\Pg(\ten{Y})\|_{\infty}\|\Pg(\ten{Z})\|_1 \nonumber\\
& \geq   \frac{1}{2}\|\PTc(\ten{Z})\|_{\ast} + \frac{\lambda}{2}\|\Pg(\ten{Z})\|_1 - \frac{\lambda}{n_1n_2n^2_3} \|\PT(\ten{Z})\|_F. \label{eq29}
\end{align}
The inequality (\ref{eq29}) is due to (\ref{eq28}). Recall that we have $\|\frac{1}{(1-2\gamma)\rho}\PT\Pg\PT - \PT\|_{\textup{op}} \leq \frac{1}{2}$ which implies $\|\frac{1}{\sqrt{(1-2\gamma)\rho}}\PT\Pg\|_{\textup{op}} \leq \sqrt{3/2}$, then
\begin{align}
\|\PT(\ten{Z})\|_F  & = \frac{1}{\sqrt{n_3}} \|\overline{\PT(\ten{Z})}\|_F \leq \frac{2}{\sqrt{n_3}} \Big\|\frac{1}{(1-2\gamma)\rho}\overline{\PT\Pg\PT(\ten{Z})}\Big\|_F \nonumber \\
& \leq  \frac{2}{\sqrt{n_3}} \Big\|\frac{1}{(1-2\gamma)\rho}\overline{\PT\Pg\PTc(\ten{Z})}\Big\|_F + \frac{2}{\sqrt{n_3}} \Big\|\frac{1}{(1-2\gamma)\rho}\overline{\PT\Pg(\ten{Z})}\Big\|_F \nonumber\\
& \leq  \sqrt{\frac{6}{(1-2\gamma)\rho n_3}}\|\overline{\PTc(\ten{Z})}\|_F + \sqrt{\frac{6}{(1-2\gamma)\rho n_3}}\|\overline{\Pg(\ten{Z})}\|_F \nonumber\\
& \leq  \sqrt{\frac{6}{(1-2\gamma)\rho}}\|\PTc(\ten{Z})\|_F + \sqrt{\frac{6}{(1-2\gamma)\rho}}\|\Pg(\ten{Z})\|_F.\label{eq30}
\end{align}

It is easy to check that
\begin{equation}
\|\PTc(\ten{Z})\|_{\textup{TNN}} = \frac{1}{n_3}\|\overline{\PTc(\ten{Z})}\|_{\ast} \geq \frac{1}{n_3}\|\overline{\PTc(\ten{Z})}\|_F = \frac{1}{\sqrt{n_3}} \|\PTc(\ten{Z})\|_F
\label{eq31}
\end{equation}
and $\|\Pg(\ten{Z})\|_1 \geq \|\Pg(\ten{Z})\|_F$. Substituting (\ref{eq30}) and (\ref{eq31}) into (\ref{eq32}), we have
\begin{eqnarray}
\Delta f \geq \Bigg(\frac{1}{2\sqrt{n_3}} - \frac{\lambda}{n_1n_2n^2_3}\sqrt{\frac{6}{(1-2\gamma)\rho}}\Bigg)\|\PTc(\ten{Z})\|_F + \Bigg(\frac{\lambda}{2} - \frac{\lambda}{n_1n_2n^2_3}\sqrt{\frac{6}{(1-2\gamma)\rho}}\Bigg)\|\Pg(\ten{Z})\|_F. \label{eq32}
\end{eqnarray}
When $n_1, n_2, n_3$ are sufficiently large such that
\begin{equation}
\frac{1}{2\sqrt{n_3}} - \frac{\lambda}{n_1 n_2 n^2_3}\sqrt{\frac{6}{(1-2\gamma)\rho}} > 0, \qquad \frac{\lambda}{2} - \frac{\lambda}{n_1 n_2 n^2_3}\sqrt{\frac{6}{(1-2\gamma)\rho}} > 0, \nonumber
\end{equation}
the inequality (\ref{eq32}) holds if and only if $\PT(\ten{Z}) = \Pg(\ten{Z}) = \mtx{0}$. On the other hand, when $\rho$ is sufficiently large and $\gamma$ is sufficiently small (which are bounded by two constants $c_\rho$ and $c_\gamma$),
\begin{equation}
\|\PT\Pg\|_{\textup{op}} \leq \sqrt{\frac{3(1 - 2\gamma)\rho}{2}} < 1, \nonumber
\end{equation}
which implies that $\PT\Pg$ is injective. As a result, (\ref{eq32}) holds if and only if $\ten{Z} = \mtx{0}$.
\end{proof}

We apply the golfing scheme similar to those used in~\cite{Li2013,Huang2014} to construct the dual tensor $\ten{Y}$, which is supported on $\Gamma$, by gradually increasing the size of $\Gamma$. Now think of $\Gamma \sim \textup{Ber}((1-2\gamma)\rho)$ as a union of sets of support $\Gamma_j$, namely, $\Gamma = \bigcup_{j=1}^p \Gamma_j$ where $\Gamma_j \sim \textup{Ber}(q_j)$. Let $q_1 = q_2 = \frac{(1-2\gamma)\rho}{6}$ and $q_3 = \dots = q_p = q$, which implies $q \geq C_0\rho/\log(n_{(1)}n_3)$. Hence we have
\begin{equation}
1 - (1 - 2\gamma)\rho = \Big(1 - \frac{(1-2\gamma)\rho}{6}\Big)^2 (1 - q)^{p-2},
\label{eq33}
\end{equation}
where $p = \lfloor 5\log(n_{(1)}n_3) + 1\rfloor$. Starting from $\ten{Z}_0 = \PT(\ten{U} \ast \ten{V}^{H} - \lambda \sgn(\ten{E}_0))$, we define inductively
\begin{equation}
\ten{Z}_j = \Big(\PT - \frac{1}{q_j}\PT\Pgs{j}\PT\Big)\ten{Z}_{j-1}.  \label{eq34}
\end{equation}
Then it follows from Lemma~\ref{lem1}-\ref{lem3} that
\begin{equation}
\|\ten{Z}_j\|_F \leq \frac{1}{2}\|\ten{Z}_{j-1}\|_F, \,\,\,j = 1, \dots,p,
\label{eq35}
\end{equation}
\begin{align}
\|\ten{Z}_1\|_{\infty} &\leq \frac{1}{2\sqrt{\log(n_{(1)}n_3)}}\|\ten{Z}_0\|_{\infty},\label{eq36}\\
\|\ten{Z}_j\|_{\infty} &\leq \frac{1}{2^j\log(n_{(1)}n_3)}\|\ten{Z}_0\|_{\infty}, \,\,\,j = 2, \dots, p,\label{eq37}
\end{align}
and
\begin{align}
 &\|(\OpId - q^{-1}_j\Pgs{j}) \ten{Z}_{j-1}\| \nonumber \\
\leq & C'_0 \sqrt{\frac{n_{(1)}n_3\log(n_{(1)}n_3)}{q_j}}\|\ten{Z}_{j-1}\|_{\infty},\,\,\,j = 2, \dots, p,
\label{eq38}
\end{align}
with high probability provided $c_r$ and $c_\gamma$ are small enough.

We set the dual tensor $\ten{Y}$ as
\begin{equation}
\ten{Y} = \sum_{j=1}^p \frac{1}{q_j}\Pgs{j}(\ten{Z}_{j-1}),
\label{eq39}
\end{equation}
and attempt to show that it satisfies (\ref{eq28}). Obviously, $\Pgc(\ten{Y}) = \mtx{0}$ and it suffices to prove
\begin{equation}
\left\{
\begin{array}{llll}
\|\PT(\ten{Y} + \lambda\sgn(\ten{E}_0) - \ten{U} \ast \ten{V}^{H})\|_F \leq \frac{\lambda}{n_1n_2n^2_3} \\
\|\PTc(\ten{Y})\| \leq \frac{1}{4}\\
\lambda\|\PTc(\sgn(\ten{E}_0))\| \leq \frac{1}{4}\\
\|\Pg(\ten{Y})\|_{\infty} \leq \frac{\lambda}{2},
\end{array}
\right.
\label{eq40}
\end{equation}
where $\lambda = 1/\sqrt{\rho n_{(1)}n_3}$ and $n_1, n_2, n_3$ are large enough.

First, let us bound $\|\ten{Z}_0\|_F$ and $\|\ten{Z}_0\|_{\infty}$. By the triangle inequality, we have $\|\ten{Z}_0\|_{\infty} \leq \|\ten{U} \ast \ten{V}^{H}\|_{\infty} + \lambda\|\PT(\sgn(\ten{E}_0))\|_{\infty}$. Noting that
\begin{equation}
\sgn(\ten{E}_0) = \sum_{i,j,k} \big[\sgn(\ten{E}_0)\big]_{ijk}\eijk, \nonumber
\end{equation}
we have
\begin{equation}
\PT(\sgn(\ten{E}_0)) = \sum_{i,j,k} \big[\sgn(\ten{E}_0)\big]_{ijk}\PT(\eijk). \nonumber
\end{equation}
Hence, the $(a,b,c)$th entry of $\PT(\sgn(\ten{E}_0))$ can be represented by
\begin{equation}
\<\PT(\sgn(\ten{E}_0)), \eabc\> = \sum_{ijk} \big[\sgn(\ten{E}_0)\big]_{ijk}\<\PT(\eijk), \eabc\>. \nonumber
\end{equation}
By Bernstein's inequality, we further have
\begin{equation}
\PP(|\<\PT(\sgn(\ten{E}_0)), \eabc\>| \geq \tau) \leq 2\exp\bigg(-\frac{\tau^2/2}{N + M\tau/3}\bigg), \nonumber
\end{equation}
where
\begin{equation}
M = \big|\big[\sgn(\ten{E}_0)\big]_{ijk}\big|\|\PT(\eijk)\|_F \|\PT(\eabc)\|_F \leq \frac{2\mu r}{n_{(2)}}, \nonumber
\end{equation}
and
\begin{equation}
N = 2\gamma\rho \|\PT(\eijk)\|^2_F \leq 4\gamma\rho\frac{\mu r}{n_{(2)}}. \nonumber
\end{equation}
Considering that the entries of $\PT(\sgn(\ten{E}_0))$ can be understood as i.i.d. copies of the $(a, b, c)$ th entry, we have by the union bound
\begin{equation}
\|\PT(\sgn(\ten{E}_0))\|_{\infty} \leq C'' \sqrt{\frac{\rho \mu r \log(n_{(1)}n_3)}{n_{(2)}}} \nonumber
\end{equation}
with high probability for some numerical constant $C''$. From the joint incoherence condition (\ref{eq18}), we know
\begin{equation}
\|\ten{U} \ast \ten{V}^{H}\|_{\infty} \leq \sqrt{\frac{\mu r}{n_1 n_2 n_3}} = \lambda \sqrt{\frac{\rho \mu r}{n_{(2)}}}, \nonumber
\end{equation}
and thus we have
\begin{align}
\|\ten{Z}_0\|_{\infty} & \leq C \lambda \sqrt{\frac{\rho \mu r\log(n_{(1)}n_3)}{n_{(2)}}}, \label{eq41}\\
\|\ten{Z}_0\|_{F} & \leq \sqrt{n_1n_2n_3}\|\ten{Z}_0\|_{\infty} \leq C \lambda \sqrt{\rho \mu rn_{(1)}n_3\log(n_{(1)}n_3)}, \label{eq42}
\end{align}
where $C = \max\big\{\frac{1}{\log(n_{(1)}n_3)}, C''\big\}$. Now, let us turn to the proof of (\ref{eq40}).

\begin{proof}
From (\ref{eq39}), we deduce
\begin{align}
&\|\PT(\ten{Y}) + \PT(\lambda\sgn(\ten{E}_0) - \ten{U} \ast \ten{V}^{H})\|_F \nonumber\\
=& \Big\|\ten{Z}_0 -\sum_{j=1}^p \frac{1}{q_j}\PT\Pgs{j}(\ten{Z}_{j-1})\Big\|_F\nonumber\\
=&  \Big\|\PT(\ten{Z}_0) -\sum_{j=1}^p \frac{1}{q_j}\PT\Pgs{j}(\ten{Z}_{j-1})\Big\|_F\nonumber\\
=&  \Big\|(\PT - \frac{1}{q_1}\PT\Pgs{1}\PT)\ten{Z}_0 - \sum_{j=2}^p \frac{1}{q_j}\PT\Pgs{j}\PT(\ten{Z}_{j-1}) \Big\|_F \nonumber\\
= & \Big\|\PT(\ten{Z}_1)-\sum_{j=2}^p \frac{1}{q_j}\PT\Pgs{j}\PT(\ten{Z}_{j-1})\Big\|_F \nonumber\\
=& \dots = \|\ten{Z}_p\|_F \leq \Big(\frac{1}{2}\Big)^p \|\ten{Z}_0\|_F \nonumber\\
\leq & C \Big(n_{(1)}n_3\Big)^{-5} \lambda \sqrt{\rho \mu r n_{(1)} n_3 \log(n_{(1)}n_3)} \nonumber\\
\leq &\frac{\lambda}{n_1n_2n^2_3}. \label{eq43}
\end{align}
The fifth step follows from (\ref{eq35}) and the sixth from (\ref{eq42}).
$\\\\$
Furthermore, we have
\begin{align}
\|\PTc(\ten{Y})\| & = \Big\|\PTc \sum_{j=1}^p \frac{1}{q_j}\Pgs{j}(\ten{Z}_{j-1})\Big\| \nonumber\\
& \leq \sum_{j=1}^p \Big\|\frac{1}{q_j} \PTc\Pgs{j}(\ten{Z}_{j-1})\Big\| \nonumber\\
& = \sum_{j=1}^p \Big\|\PTc\Big(\frac{1}{q_j}\Pgs{j}(\ten{Z}_{j-1}) - \ten{Z}_{j-1}\Big)\Big\| \nonumber\\
& \leq  \sum_{j=1}^p \Big\|\frac{1}{q_j}\Pgs{j}(\ten{Z}_{j-1}) - \ten{Z}_{j-1}\Big\| \nonumber\\
& \leq \sum_{j=1}^p C'_0 \sqrt{\frac{n_{(1)}n_3\log(n_{(1)}n_3)}{q_j}}\|\ten{Z}_{j-1}\|_{\infty}  \nonumber\\
& \leq C'_0 \sqrt{n_{(1)}n_3\log(n_{(1)}n_3)}\Big(\sum_{j=3}^p \frac{1}{2^{j-1}\log(n_{(1)}n_3)\sqrt{q_j}} + \frac{1}{2\sqrt{\log(n_{(1)}n_3)}\sqrt{q_2}} + \frac{1}{\sqrt{q_1}}\Big)\|\ten{Z}_0\|_{\infty} \nonumber\\
& \leq C' \lambda \sqrt{\frac{ \rho \mu r n_{(1)}n_3(\log(n_{(1)}n_3))^2}{\rho n_{(2)}}}\label{eq44}\\
& \leq C' \sqrt{c_r} \leq \frac{1}{4}, \label{eq45}
\end{align}
provided $\lambda = 1/\sqrt{\rho n_{(1)}n_3}$ and $c_r$ is sufficiently small. The fifth step is from Lemma~\ref{lem3}, the sixth from (\ref{eq36})-(\ref{eq38}), and the seventh from (\ref{eq41}) respectively.
$\\\\$
Third, the sign tensor $\sgn(\ten{E}_0)$ is distributed as
\begin{equation}
\big[\sgn(\ten{E}_0)\big]_{ijk} = \left\{
\begin{array}{lll}
1, & \textup{w.p.}\,\,\gamma\rho\\
0, & \textup{w.p.}\,\,1-2\gamma\rho\\
-1, & \textup{w.p.}\,\,\gamma\rho
\end{array}
\right. \nonumber.
\end{equation}
As proved by Lemma~\ref{lem4}, there exists a function $\varphi(\gamma\rho)$ satisfying $\lim\limits_{\gamma\rho \rightarrow 0^{+}} \varphi(\gamma\rho) = 0$, such that
\begin{equation}
\|\sgn(\ten{E}_0)\| \leq \varphi(\gamma\rho) \sqrt{n_{(1)}n_3} \nonumber
\end{equation}
with large probability, which gives
\begin{equation}
\lambda\|\PTc(\sgn(\ten{E}_0))\| \leq \lambda\|\sgn(\ten{E}_0)\| \leq \varphi(\gamma\rho)/\sqrt{\rho} \leq \frac{1}{4},
\label{eq46}
\end{equation}
as long as $c_r$ and $c_\gamma$ is sufficiently small.
$\\\\$
Last, we observe that
\begin{align}
\|\Pg(\ten{Y})\|_{\infty} &= \Big\|\PT\sum_{j=1}^p \frac{1}{q_j}\Pgs{j}(\ten{Z}_{j-1})\Big\|_{\infty} \nonumber\\
&\leq \sum_{j=1}^p \frac{1}{q_j}\|\ten{Z}_{j-1}\|_{\infty}\nonumber\\
& \leq \Big(\sum_{j=3}^p \frac{1}{2^{j-1}\log(n_{(1)}n_3)\sqrt{q_j}} + \frac{1}{2\sqrt{\log(n_{(1)}n_3)q_2}} +\frac{1}{\sqrt{q_1}}\Big)\|\ten{Z}_0\|_{\infty}\nonumber\\
& \leq C \sqrt{\frac{\mu r\log(n_{(1)}n_3)}{\rho n_{(2)}}}\lambda \nonumber\\
& \leq C \sqrt{\frac{c_r}{\log(n_{(1)}n_3)}} \lambda \leq \frac{\lambda}{2} \label{eq47}
\end{align}
when $c_r$ is small enough. The second step follows from Lemma~\ref{lem2}, the third from (\ref{eq36})-(\ref{eq38}), and the fourth from (\ref{eq41}) respectively.
\end{proof}

\subsection{Proofs of Two Corollaries}
\label{sec5:sub4}

Corollary~\ref{cor2} is obvious, as problem (\ref{eq13}) reduce to (\ref{eq15}) in the event of the whole entries available, i.e., $\rho = 1$. The proof of Corollary~\ref{cor1} is also straightforward. Remember that the sparse term $\ten{E}$ in (\ref{eq13}) should vanish for the TC problems. We can achieve this goal by setting $\lambda \rightarrow \infty$. In this situation, the first, third and fourth inequalities in (\ref{eq40}) hold automatically. This says that the joint incoherence condition (\ref{eq18}) is unnecessary and can be successfully excluded. From (\ref{eq44}), we have
\begin{equation}
\frac{\mu r (\log(n_{(1)}n_3))^2}{\rho n_{(2)}} \leq c_r, \nonumber
\end{equation}
which implies
\begin{equation}
m = \rho n_{(1)} n_{(2)} n_3 \geq c_0 \mu r n_{(1)} n_3(\log(n_{(1)}n_3))^2, \nonumber
\end{equation}
where $c_0 = 1/c_r$ is sufficiently large positive constant.

\section{Optimization Algorithm}
\label{sec6}

In this work, we use the alternating direction method of multiplier (ADMM) method to solve the convex problem~(\ref{eq13}). ADMM decomposes a large global problem into a series of smaller subproblems, and coordinates the solutions of subproblems
to compute the globally optimal solution. It has received renewed interest in recent years due to the fact that it is very efficient to tackle large-scale problems and solve optimization problems with multiple non-smooth terms in the objective function. We also refer to~\cite{Candes2011,Shang2014,Zhang2014,Lu2016} for some exploited applications of ADMM to the similar problems.

\begin{algorithm}[!t]
\caption{RTC: Solving (\ref{eq13}) via ADMM}
\label{alg2}
\textbf{Input:} $\ten{X}$, $\Omega$ and $\lambda$.\\
\textbf{Initialize:} $\ten{L}^0 = \ten{E}^0 = \ten{Y}^0 = \mtx{0}$, $\rho = 1.1$, $\mu^0$ = 1e-4, $\mu_{\max}$ = 1e8, $\varepsilon$ = 1e-6. \\\vspace{-0.4cm}
\begin{algorithmic}[1]
\WHILE {not converged}
\STATE {Update $\ten{L}^{k+1}$ by
\begin{displaymath}
\min\limits_{\ten{L}} \|\ten{L}\|_{\textup{TNN}} + \frac{\mu^k}{2}\Big\|\ten{L} + \ten{E}^k - \ten{X} + \frac{\ten{Y}^k}{\mu^k}\Big\|^2_F;
\end{displaymath}}
\STATE {Update $\PO(\ten{E}^{k+1})$ by
\begin{displaymath}
\min\limits_{\ten{E}} \lambda\|\PO(\ten{E})\|_{1} + \frac{\mu^k}{2}\Big\|\PO\Big(\ten{E} + \ten{L}^{k+1} - \ten{X} + \frac{\ten{Y}^k}{\mu^k}\Big)\Big\|^2_F;
\end{displaymath}}
\STATE {Update $\POc(\ten{E}^{k+1})$ by
\begin{displaymath}
\POc(\ten{E}^{k+1}) = \POc(\ten{X} - \ten{L}^{k+1} - \ten{Y}^k/\mu^k)
\end{displaymath}}
\STATE {Update the multipliers $\ten{Y}^{k+1}$ by\\ $\ten{Y}^{k+1} =\ten{Y}^{k} +\mu^{k}(\ten{L}^{k+1}+\ten{E}^{k+1} - \ten{X})$;}
\STATE {Update $\mu^{k+1}$ by $\mu^{k+1}=\textup{min}(\rho\mu^{k},\,\mu_{\max})$;}
\STATE {Check the convergence condition, \\
$\|\ten{L}^{k+1} -\ten{L}^{k}\|_{\infty} <\varepsilon, \,\,\, \|\ten{E}^{k+1} -\ten{E}^{k}\|_{\infty} <\varepsilon,$\\
$\|\ten{X}-\ten{L}^{k+1} -\ten{E}^{k+1}\|_{\infty} <\varepsilon$.}
\ENDWHILE
\end{algorithmic}
\textbf{Output:} $\ten{L}$.
\end{algorithm}

As mentioned above, the $\ell_1$-norm in (\ref{eq13}) forces any entry of an optimal solution $\ten{E}$ in the unobserved set $\Omega^\perp$ to be zero. Without loss of generality, we can assume that the unobserved data may be appropriate values such that $\POc(\ten{X}) = \POc(\ten{L}) + \POc(\ten{E})$. Then, the linear projection operator constraint
in (\ref{eq13}) is simply replaced by an equation $\ten{X} = \ten{L} + \ten{E}$. Thus, problem (\ref{eq13}) can be rewritten as
\begin{equation}
\min_{\ten{L},\,\ten{E}} \|\ten{L}\|_{\ast} + \lambda\|\PO(\ten{E})\|_{1},\quad \textup{s.t.},\,  \ten{X} = \ten{L} + \ten{E}, \label{eq48}
\end{equation}
which is exactly a high-order version of the RMC problem~\cite{Shang2014} and we have good reason to believe that it can be solved in a similar way. Algorithm~\ref{alg2} summarize the optimization details. In step 2 and 3, the updates of $\ten{L}^{k+1}$ and $\PO(\ten{E}^{k+1})$ both have closed-form solutions~\cite{Shang2014,Zhang2014}. It is easy to find that the computational cost of this algorithm is dominated by step 2 in each iteration, which requires computing FFT and $n_3$ SVDs of $n_1 \times n_2$ matrices. Hence the complexity is $O(t(n_1n_2n_3\log(n_3) + n_{(1)}n^2_{(2)}n_3))$ where $t$ is the number of iterations. We can resort to the conjugate symmetry of the Fourier transform to further reduce the computational burden (see~\cite{Kilmer2013} for more details).

\section{Experiments}
\label{sec7}

We conduct a series of experiments to demonstrate the validity of our theorem, and show possible applications of our model and algorithm. As suggested by Theorem~\ref{the1}, the parameter $\lambda$ is set to be $\lambda = 1/\sqrt{\rho n_{(1)}n_3}$ in all the experiments unless otherwise specified. For practical problems, it is possible to further improve the performance by turning $\lambda$ cautiously. Nevertheless, the default value is often a good rule of thumb.

\subsection{Synthetic Tensor Recovery}
\label{sec7:sub1}

\subsubsection{Validity of Exact Recovery}

We first verify the correct recovery phenomenon of Theorem~\ref{the1} by synthetic problems. For simplicity, we consider the tensors of size $n \times n \times n$ with varying dimension $n = 100, 200$ and $300$. We generate the clean tensor $\ten{L}_0 = \ten{P} \ast \ten{W}$ with tubal rank $\rank_t(\ten{L}_0) = r$, where the entries of $\ten{P} \in \R^{n \times r \times n}$ and $\ten{W} \in \R^{r \times n \times n}$ are independently sampled from a standard Gaussian distribution $\mathcal{N}(0, 1)$\footnote{We also consider the situations in which the entries of tensors $\ten{P}$, $\ten{W}$ and $\ten{E}_0$ are sampled from other different distributions, such as uniform distribution and Bernoulli distribution. Similar results are obtained and we do not report them here due to page limit.}. In addition, a fraction $\gamma$ of its entries are uniformly corrupted by additive i.i.d. noise from a standard Gaussian distribution $\mathcal{N}(0, 1)$ at random. Finally, we randomly choose a percentage $\rho$ of the noisy tensor entries as our observations.

\begin{table}[!t]
\centering
\caption{Exact recovery on random data with different sizes. In the two scenarios, synthetic tensors are with different tubal ranks, percentage of missing entries and proportion of grossly corrupted observations.}
\label{tab1}
\vspace{0.5em}
\begin{tabular}{c|c|c|c|c|c|c|c}
\hline
\multicolumn{4}{c|}{$r = 0.05n, \,\,\, \rho = 0.9, \,\,\, \gamma = 0.1$} & \multicolumn{4}{c}{$r = 0.1n, \,\,\, \rho = 0.8, \,\,\, \gamma = 0.2$}\\
\hline
$n$ & $r$ & $\rank_t(\ten{L})$ & $\frac{\|\ten{L} - \ten{L}_0\|_F}{\|\ten{L}_0\|_F}$ & $n$ & $r$ & $\rank_t(\ten{L})$ & $\frac{\|\ten{L} - \ten{L}_0\|_F}{\|\ten{L}_0\|_F}$\\
\hline
\hline
100 & 5 & 5	 & $1.68 \times 10^{-10}$ & 100 & 10 & 10 & $8.53 \times 10^{-10}$ \\
\hline
200	& 10 & 10 & $5.25 \times 10^{-12}$ & 200 & 20 & 20 & $1.13 \times 10^{-11}$\\
\hline
300	& 15 & 15 &	$4.80 \times 10^{-12}$ & 300 & 30 & 30 & $2.26 \times 10^{-12}$\\
\hline
\end{tabular}
\end{table}

We test on two cases and summarize the results in Table~\ref{tab1}. We set $r = 0.05n$, $\gamma = 0.1$ and $\rho = 0.9$ for the first scenario, and choose a more challenging setting with $r = 0.1n$, $\gamma = 0.2$ and $\rho = 0.8$ for the second scenario. It is clear to see that our method gives the correct rank estimation of $\ten{L}_0$ and the negligible relative error $\|\ten{L} - \ten{L}_0\|_F/\|\ten{L}_0\|_F$ in all cases. These results verify the exact recovery phenomenon as claimed in Theorem~\ref{the1} pretty well.

\begin{table}[!t]
\centering
\caption{Exact recovery on $100 \times 100 \times 100$ random data with different corruption magnitudes.}
\label{tab2}
\vspace{0.5em}
\begin{tabular}{c|c|c|c}
\hline
\multicolumn{2}{c|}{$r = 0.05n, \,\,\, \rho = 0.9, \,\,\, \gamma = 0.1$} & \multicolumn{2}{c}{$r = 0.1n, \,\,\, \rho = 0.8, \,\,\, \gamma = 0.2$}\\
\hline
Magnitude & $\frac{\|\ten{L} - \ten{L}_0\|_F}{\|\ten{L}_0\|_F}$ & Magnitude & $\frac{\|\ten{L} - \ten{L}_0\|_F}{\|\ten{L}_0\|_F}$\\
\hline
\hline
$\mathcal{N}(0, 1/n)$ & $5.52 \times 10^{-11}$ & $\mathcal{N}(0, 1/n)$ & $3.52 \times 10^{-11}$\\
\hline
$\mathcal{N}(0, 1)$ & $1.68 \times 10^{-10}$ & $\mathcal{N}(0, 1)$ & $8.53 \times 10^{-10}$\\
\hline
$\mathcal{N}(0, n)$ & $6.60 \times 10^{-10}$ & $\mathcal{N}(0, n)$ & $2.80 \times 10^{-9}$\\
\hline
\end{tabular}
\end{table}

Theorem~\ref{the1} shows that the exact recovery is independent of the magnitudes of the corruption term $\ten{E}_0$. To verify this, Table~\ref{tab2} reports the relative error $\|\ten{L} - \ten{L}_0\|_F/\|\ten{L}_0\|_F$ under varying corruption magnitudes $\mathcal{N}(0, 1/n)$, $\mathcal{N}(0, 1)$ and $\mathcal{N}(0, n)$. It seems that our approach always succeeds, no matter what magnitudes of the corruptions are.

\subsubsection{Phase Transition}

\begin{figure}[!t]
\centering
\includegraphics[width=0.7\textwidth]{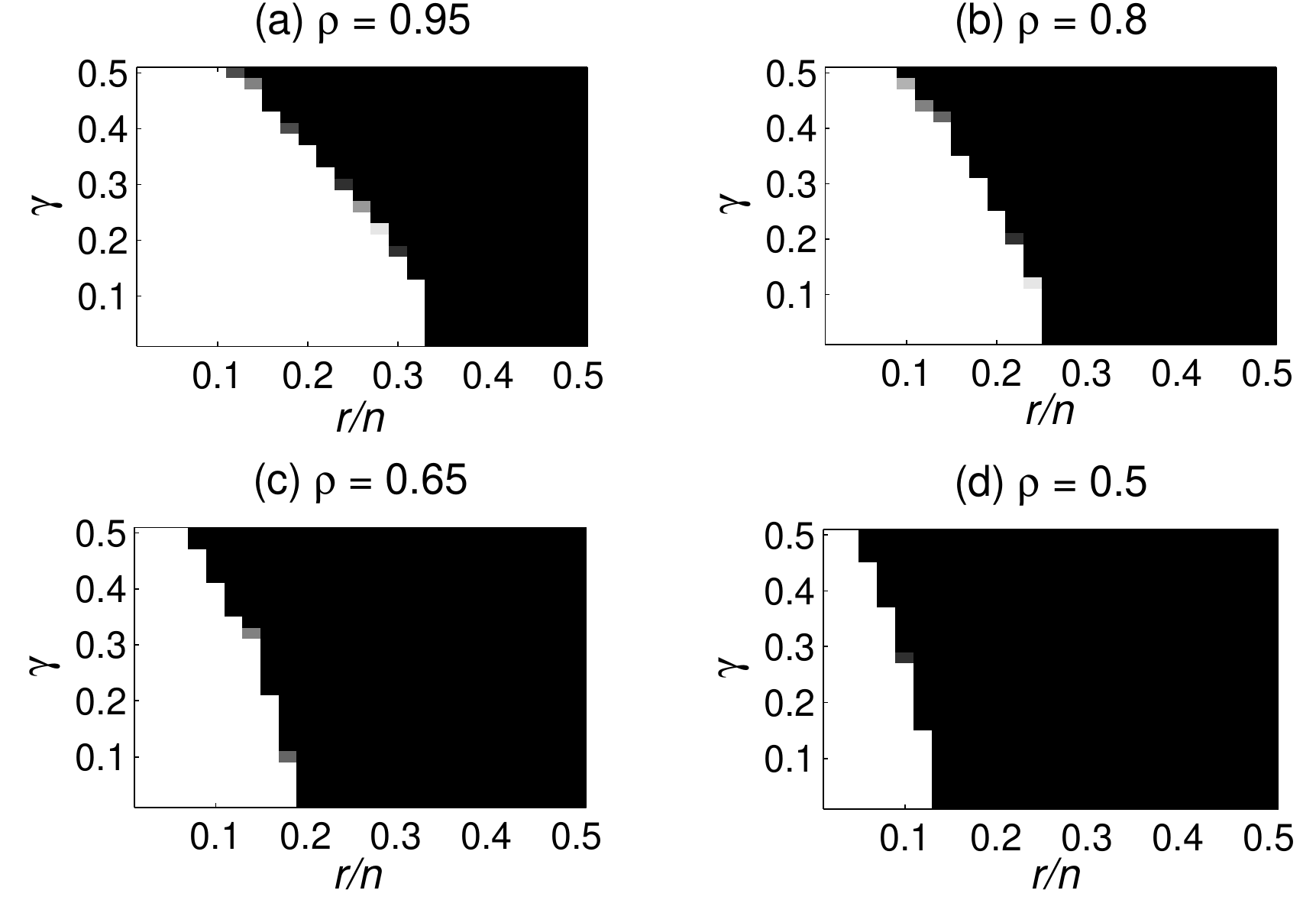}
\caption{Exact recovery for varying rank and gross corruptions under different proportions of observed entries. Fraction of perfect recoveries across 10 trials, as a function of tubal rank $\rank_t(\ten{L}_0)$ ($x$-axis) and proportion of gross corruptions $\ten{E}_0$ ($y$-axis).}
\label{fig3}
\end{figure}

To further corroborate our theoretical results, we check the recovery ability of our algorithm as a function of tubal rank $r$, fractions of gross corruptions $\gamma$ and proportion of observed entries $\rho$. The data are generated as the above-mentioned experiments, where the data size $n = 100$. We set $\rho$ to be different specified values, and vary $r$ and $\gamma$ to empirically investigate the probability of recovery success. For each pair $(r, \gamma)$, we simulate 10 test instances and declare a trial to be successful if the recovered tensor $\ten{L}$ satisfies $\|\ten{L} - \ten{L}_0\|_F/\|\ten{L}_0\|_F \leq 10^{-3}$. Figure~\ref{fig3} reports the fraction of perfect recovery for each pair (black = $0\%$ and white = $100\%$). We see clearly that there exists a big region in which the recovery is correct for all the cases. Moreover, the larger the percentage of missing values is, the smaller the region of correct recovery becomes.

\subsubsection{Comparison with Similar Methods}

Considering the connections among problems (\ref{eq13}), (\ref{eq14}) and (\ref{eq15}), we compare our algorithm with two most similar approaches\footnote{These two approaches are thought to be the most similar to our algorithm in the sense that they also employ the ADMM algorithm to solve corresponding convex optimization problem under the algebraic framework of t-SVD.}, namely, tubal nuclear norm minimization (TNNM)~\cite{Zhang2014} and TRPCA~\cite{Lu2016}, for three different settings. Once again, we fix $n = 100$ and generate random data as the prior experiments. In the first case, we set $\ten{E}_0$ to vanish and vary $r$ and $\rho$ to compare the recovery behaviors of the three methods for tensor completion. As shown in the first row of Figure~\ref{fig4}, RTC performs much better than the other methods. We then test all the three methods by varying $r$ and $\gamma$ when the entire entries are observed. Since problem (\ref{eq13}) reduces to problem (\ref{eq15}) in this settings, our algorithm and TRPCA obtain the identical results. In contrast, TNNM fails to recover the synthetic tensors in most cases, owing to that it is very fragile to the gross corruptions. Finally, we consider the robust tensor completion setting, i.e., fixing $\rho = 0.8$ and setting $r$ and $\gamma$ to be different values. From the bottom row of Figure~\ref{fig4}, we observe that RTC consistently and significantly outperforms the other two approaches, especially TNNM, which can not give desirable results.

\begin{figure}
\centering
\subfigure{\includegraphics[width=0.22\textwidth]{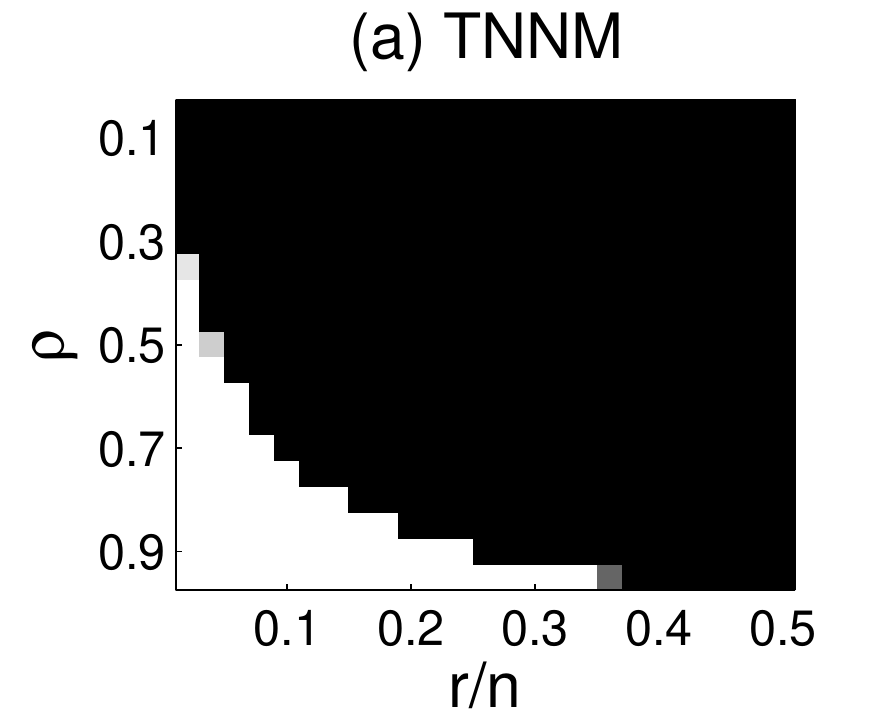}}
\hspace{-1em}
\subfigure{\includegraphics[width=0.22\textwidth]{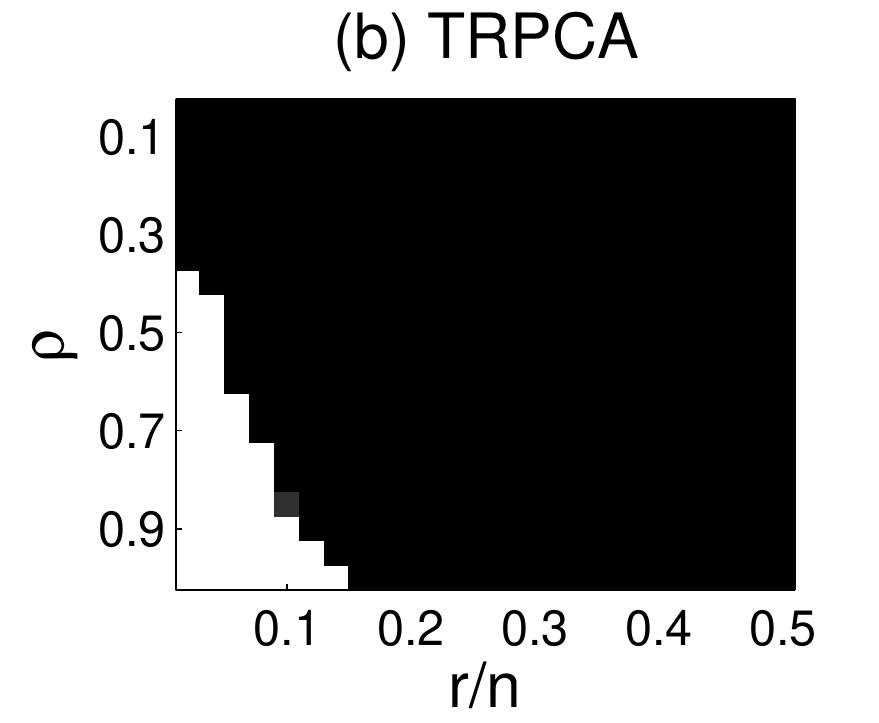}}
\hspace{-1em}
\subfigure{\includegraphics[width=0.22\textwidth]{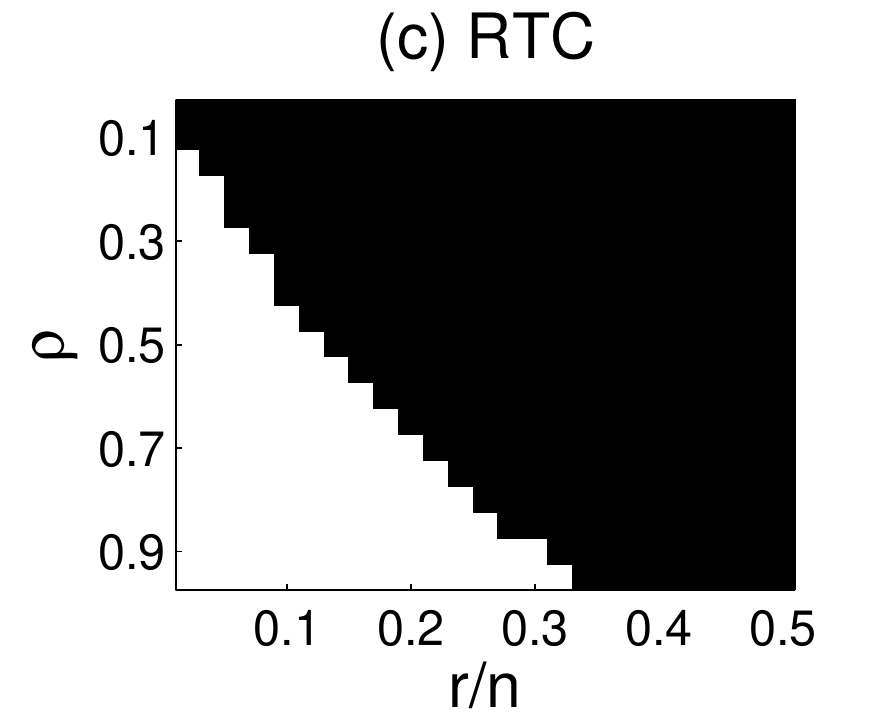}}
\\[-0.5em]
\subfigure{\includegraphics[width=0.22\textwidth]{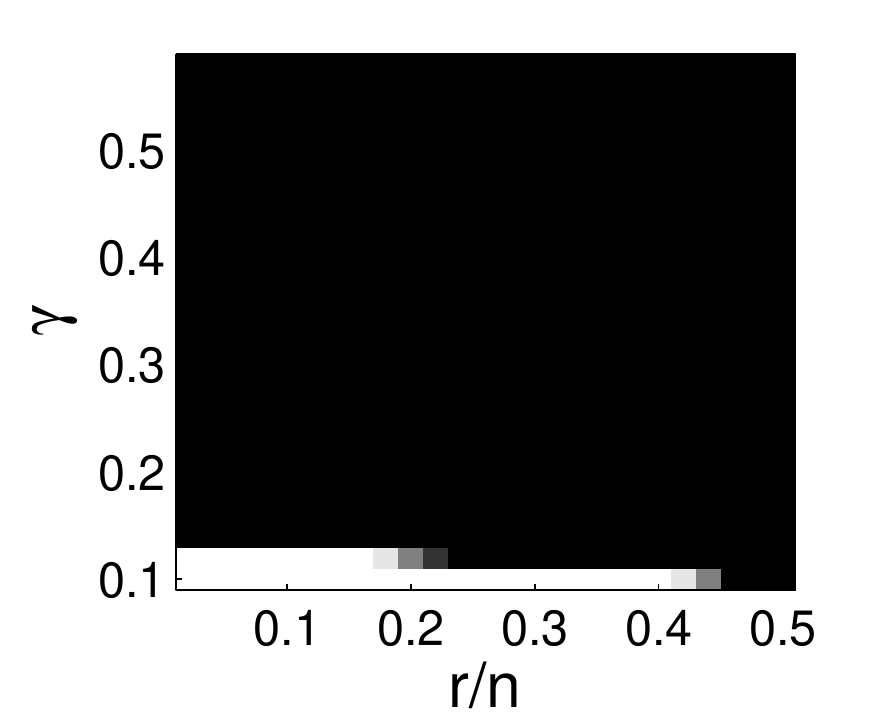}}
\hspace{-1em}
\subfigure{\includegraphics[width=0.22\textwidth]{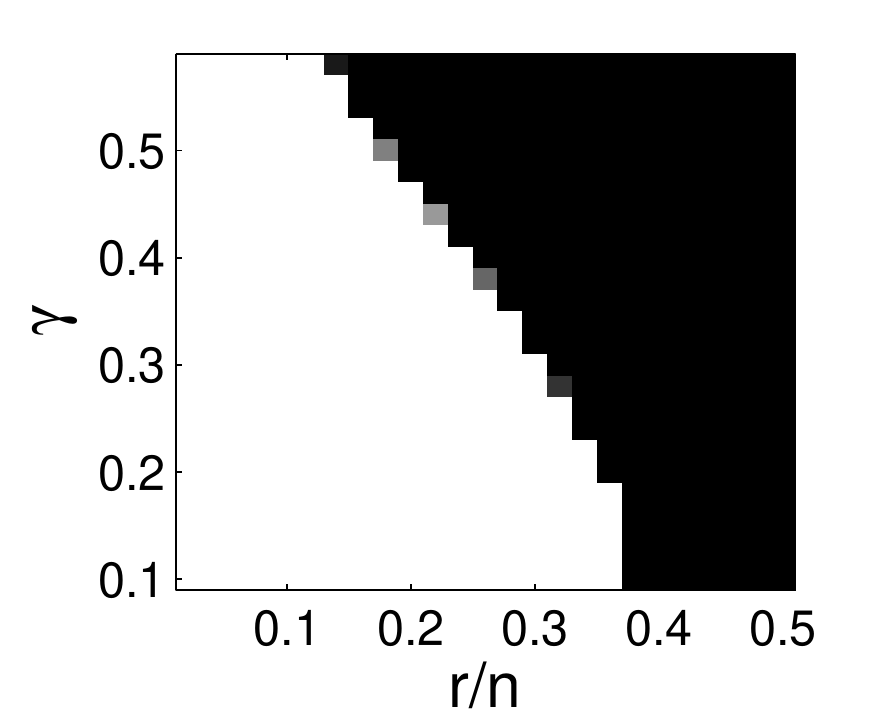}}
\hspace{-1em}
\subfigure{\includegraphics[width=0.22\textwidth]{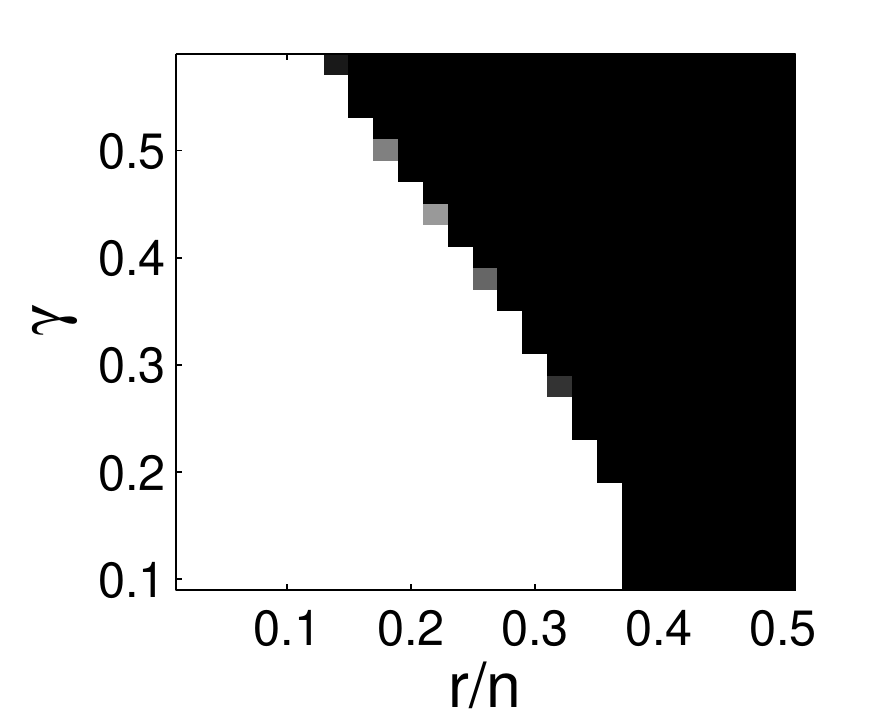}}
\\[-0.5em]
\subfigure{\includegraphics[width=0.22\textwidth]{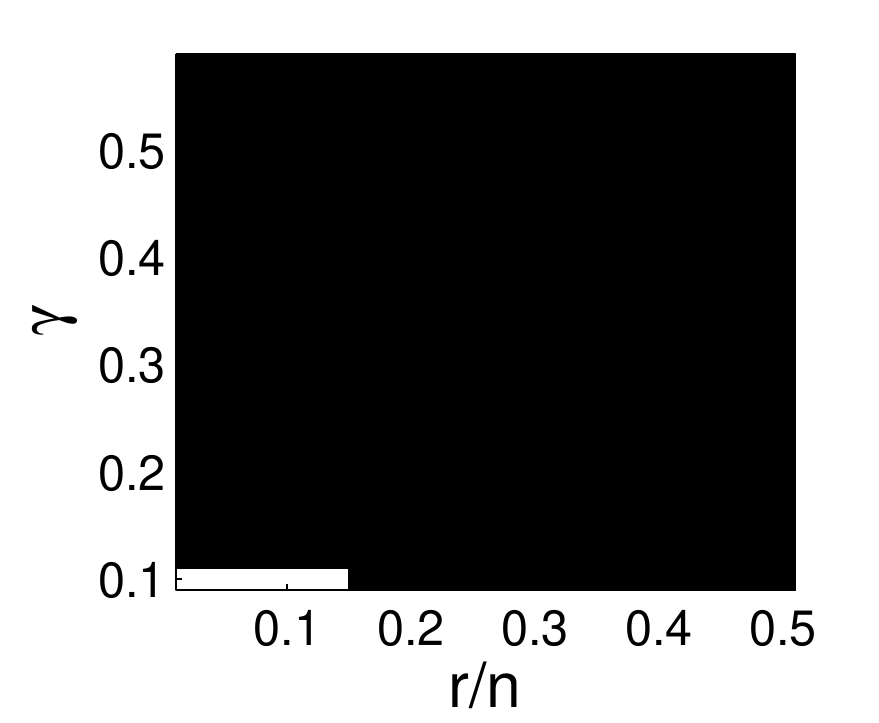}}
\hspace{-1em}
\subfigure{\includegraphics[width=0.22\textwidth]{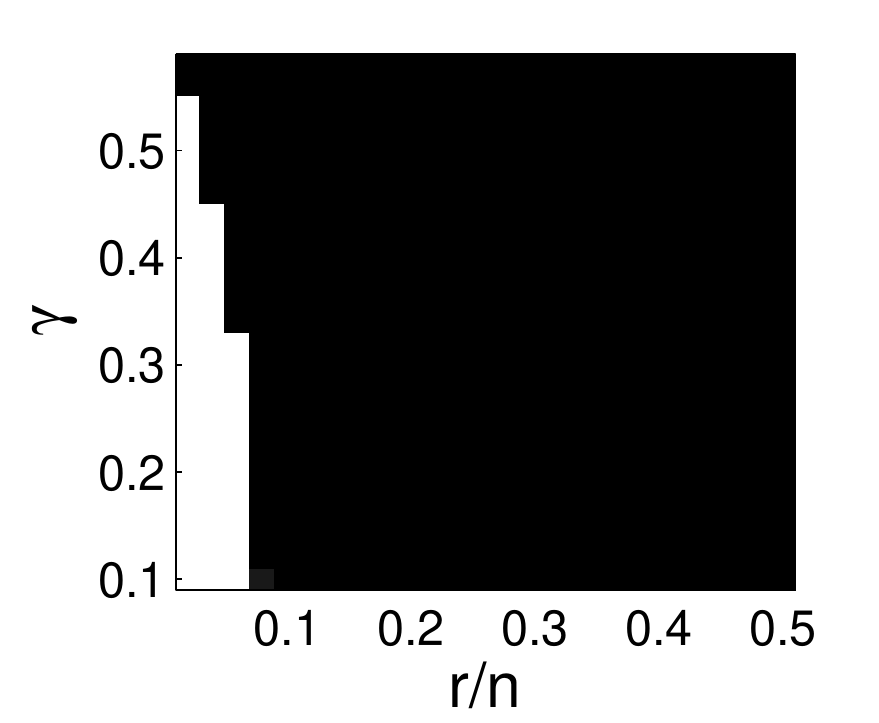}}
\hspace{-1em}
\subfigure{\includegraphics[width=0.22\textwidth]{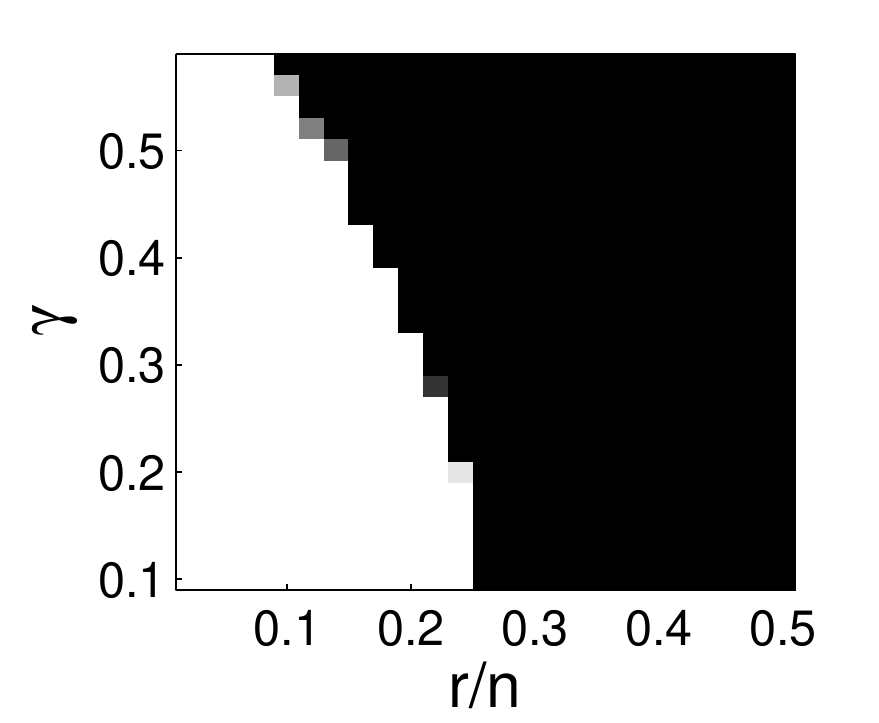}}
\caption{Comparison of our method and two similar approaches for three different problems of low-rank tensor recovery. Top: tensor completion. Middle: tensor robust principal component analysis. Bottom: robust tensor completion.}
\label{fig4}
\end{figure}

\subsection{Natural Image Restoration}
\label{sec7:sub2}

It is well known that a $n_1 \times n_2$ color image with red, blue and green channels can be naturally regarded as a third-order tensor $\ten{X} \in \R^{n_1 \times n_2 \times 3}$. Each frontal slice of $\ten{X}$ corresponds to a channel of the color image. Actually, each channel of a color image may not be low-rank, but their top singular values dominate the main information~\cite{Liu2013,Lu2016}. Hence, the image can be approximately reconstructed by a low-tubal-rank tensor.

\begin{table*}[!t]
\centering
\caption{Average PSNR and SSIM obtained by various methods on the BSD image set.}
\label{tab3}
\vspace{0.5em}
\scalebox{0.75}{
\begin{tabular}{c|cc|cc|cc|cc|cc|cc}
\hline
 & \multicolumn{6}{c|}{$\rho = 0.9$} & \multicolumn{6}{c}{$\rho = 0.7$}\\
\cline{2-13}
& \multicolumn{2}{c|}{$\gamma = 0.1$} & \multicolumn{2}{c|}{$\gamma = 0.2$} & \multicolumn{2}{c|}{$\gamma = 0.3$} & \multicolumn{2}{c|}{$\gamma = 0.1$} & \multicolumn{2}{c|}{$\gamma = 0.2$} & \multicolumn{2}{c}{$\gamma = 0.3$}\\
\cline{2-13}
& PSNR & SSIM & PSNR & SSIM & PSNR & SSIM & PSNR & SSIM & PSNR & SSIM & PSNR & SSIM\\
\hline
\hline
RPCA & 27.67 & 0.8535 & 27.30  & 0.8367 & 26.88 & 0.8122 & 21.69 & 0.5609 & 20.62 &    0.4744 & 19.64 & 0.4081\\
RMC & 28.11 & 0.8552 & 27.82 & 0.8423 & 27.53 & 0.8276 & 26.33 & 0.7865 & 26.09 & 0.7736 &	25.84 & 0.7599\\
TRPCA & \underline{32.31} & \underline{0.9457} & \underline{31.59} & \underline{0.9278} & \underline{30.91} & \underline{0.9037} & 29.25 & \underline{0.8608} & 28.69 & 0.8209 & 28.06	& 0.7678\\
%\hline
BM3D & 28.31 & 0.8008 & 28.25 &	0.8002 & 28.18 & 0.7994 & 20.91 & 0.3392 & 20.80 & 0.3311 & 20.70 & 0.3234\\
BM3D+ & 30.72 & 0.8289 & 30.51 & 0.8245 & 30.28 & 0.8203 & 29.75 & 0.8060 & 29.45 & 0.7993	& 29.18 & 0.7935\\
BM3D++ & 30.94 & 0.8338	& 30.74 & 0.8297 & 30.52 & 0.8257 & \underline{30.42} & 0.8221 & \underline{30.11} & 0.8152 & \underline{29.82} & \underline{0.8093}\\
%\hline
SNN & 30.14 & 0.9128 & 29.60 & 0.8972 & 29.11 & 0.8797 & 27.75 & 0.8426 & 27.35	& \underline{0.8248}	& 26.97 & 0.8063\\
RTC & {\bf 33.03} & {\bf 0.9566} & {\bf 32.10} & {\bf 0.9400} & {\bf 31.27} & {\bf 0.9185} & {\bf 31.30} & {\bf 0.9296} & {\bf 30.58} & {\bf 0.9091} &	{\bf 29.91} & {\bf 0.8831}\\
\hline
\end{tabular}}
\end{table*}

\begin{figure}[!t]
\centering
\subfigure{\includegraphics[width=0.9\textwidth]{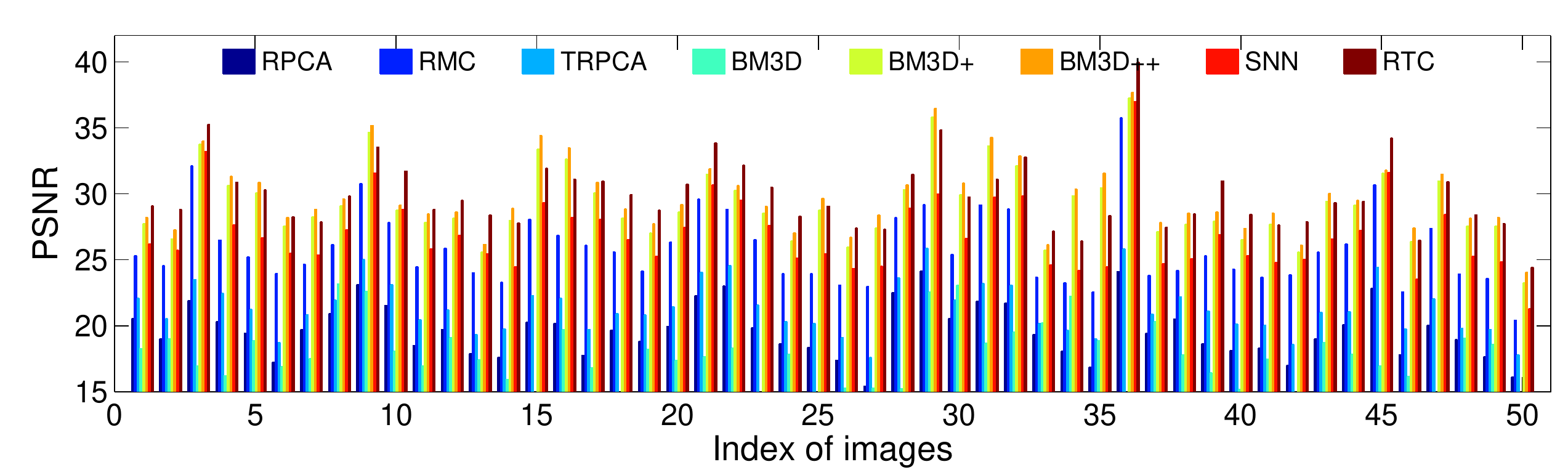}}
\subfigure{\includegraphics[width=0.9\textwidth]{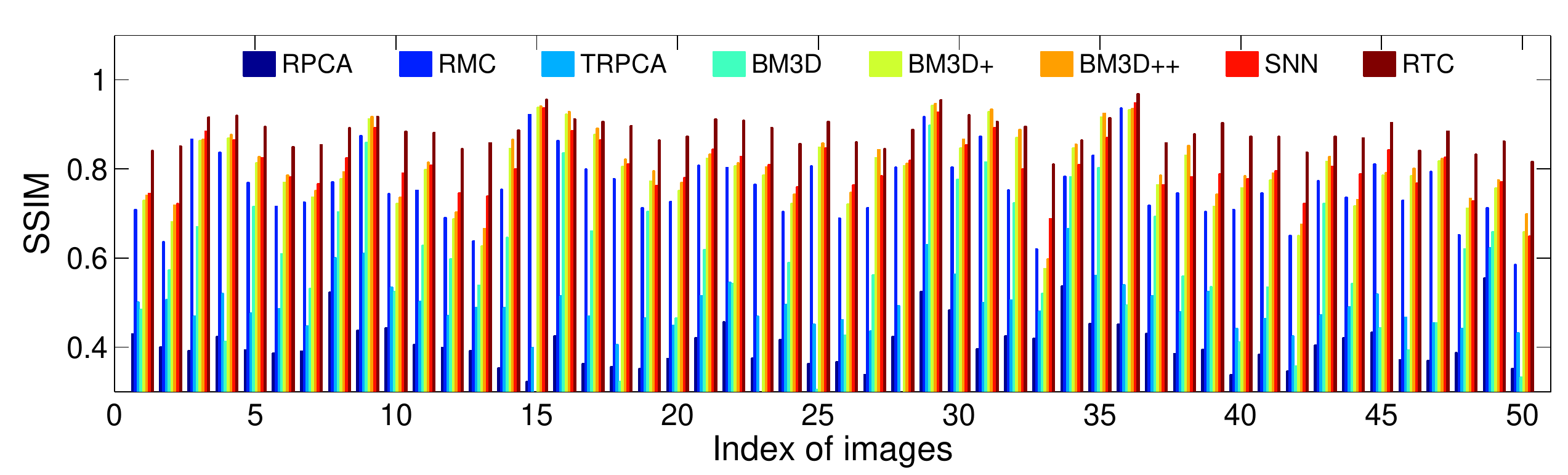}}
\caption{Comparison of the PSNR and SSIM values of various methods for image restoration on 50 images when $\rho = 0.7$ and $\gamma = 0.3$.}
\label{fig5}
\end{figure}

In this experiment, we focus on noisy image completion. This problem, unlike the traditional problems of image inpainting and image denoising, aims to simultaneously fill the missing pixels and remove the noise in an image. One typical example is the restoration of archived photographs and films~\cite{Kokaram2004,Subrahmanyam2010}. The archived materials are prone to be degraded due to physical processes or chemical decompositions, which may lead to various kinds of contaminations as well. So it is necessary to deal with corruptions and missing values jointly.

We download 50 color images at random from the Berkeley Segmentation Database (BSD)~\cite{Martin2001}. For each image, we randomly set $\gamma$ pixels corrupted with Gaussian noise $\mathcal{N}(0, \sigma)$, and choose $\rho$ entries to be observed. We compare our algorithm with several approaches for low-rank matrix/tensor recovery, including RPCA~\cite{Candes2011}, RMC~\cite{Shang2014}, SNN~\cite{Huang2014}, and TRPCA~\cite{Lu2016}. For RPCA and RMC, we apply them on each channel independently with $\lambda = 1/\sqrt{n_{(1)}}$. For SNN, we find that its performance is not satisfactory when the parameters $\lambda_i$'s are set to the default values~\cite{Huang2014}. As suggested by~\cite{Lu2016}, we empirically set $\lambda_1 = \lambda_2 = 15$ and $\lambda_3 = 1.5$ which make SNN work pretty well for most images. We run TRPCA with $\lambda = 1/\sqrt{n_{(1)}n_3}$. For a comprehensive comparison, we also test BM3D\footnote{\url{http://www.cs.tut.fi/~foi/GCF-BM3D/index.html}}~\cite{Dabov2007} on the BSD image set, which is usually referred to as the representative of state-of-the-art algorithms for image restoration. Considering that BM3D is originally proposed for image denoising, we further enhance it with the scheme of \lq\lq completion + denoising\rq\rq, which means that we first fill the missing pixels without considering the noise and then apply BM3D to the intermediate result. Here, HaLRTC\footnote{\url{http://www.cs.rochester.edu/u/jliu/publications.html}}~\cite{Liu2013} and TNNM~\cite{Zhang2014} are used in the completion step and the corresponding methods are denoted by BM3D+ and BM3D++, respectively.

\begin{figure*}[t!]
\centering
\subfigure[{\tiny Original (PSNR, SSIM)}]{\includegraphics[width=0.18\textwidth]{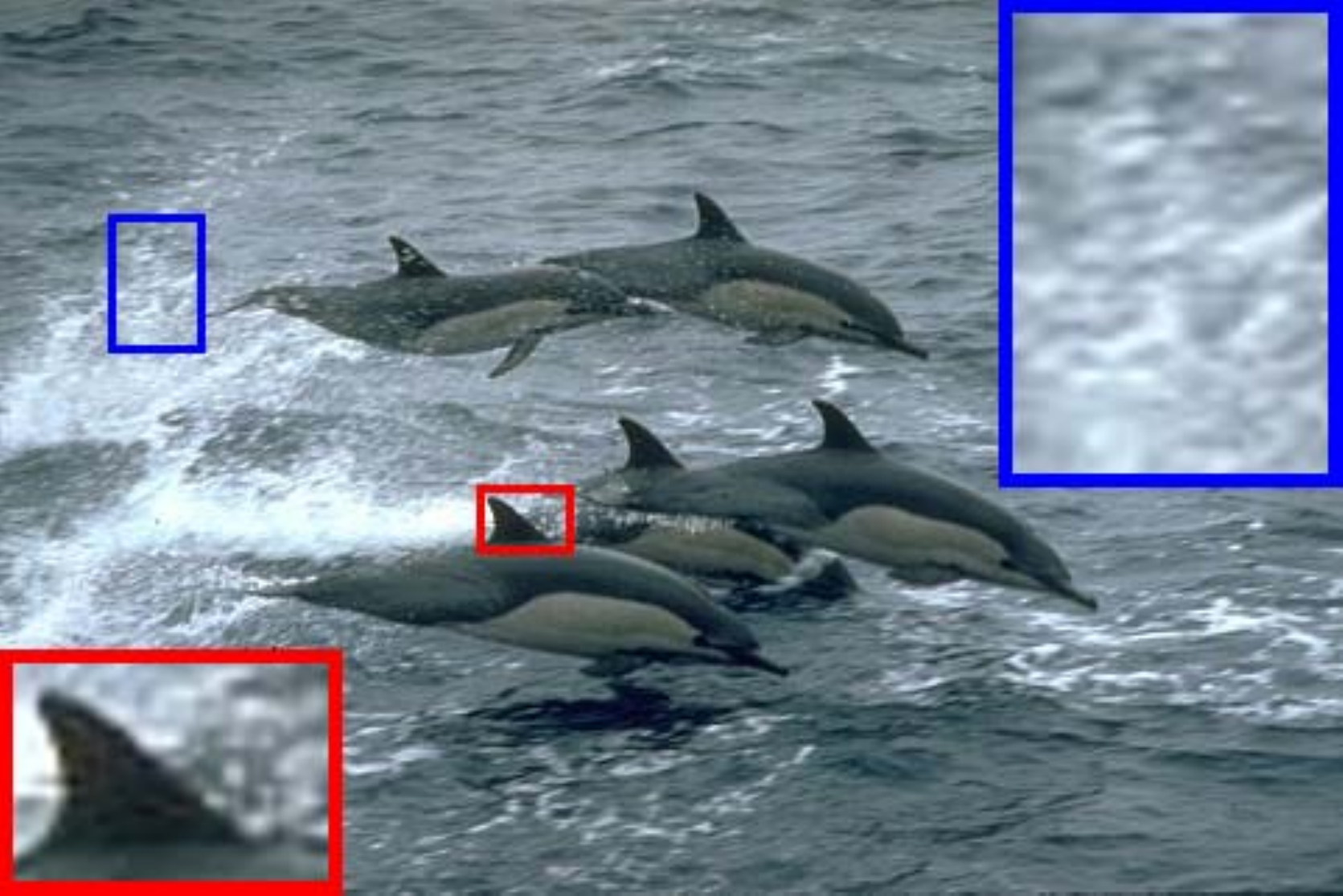}}
\subfigure[{\tiny Noisy (NA, NA)}]{\includegraphics[width=0.18\textwidth]{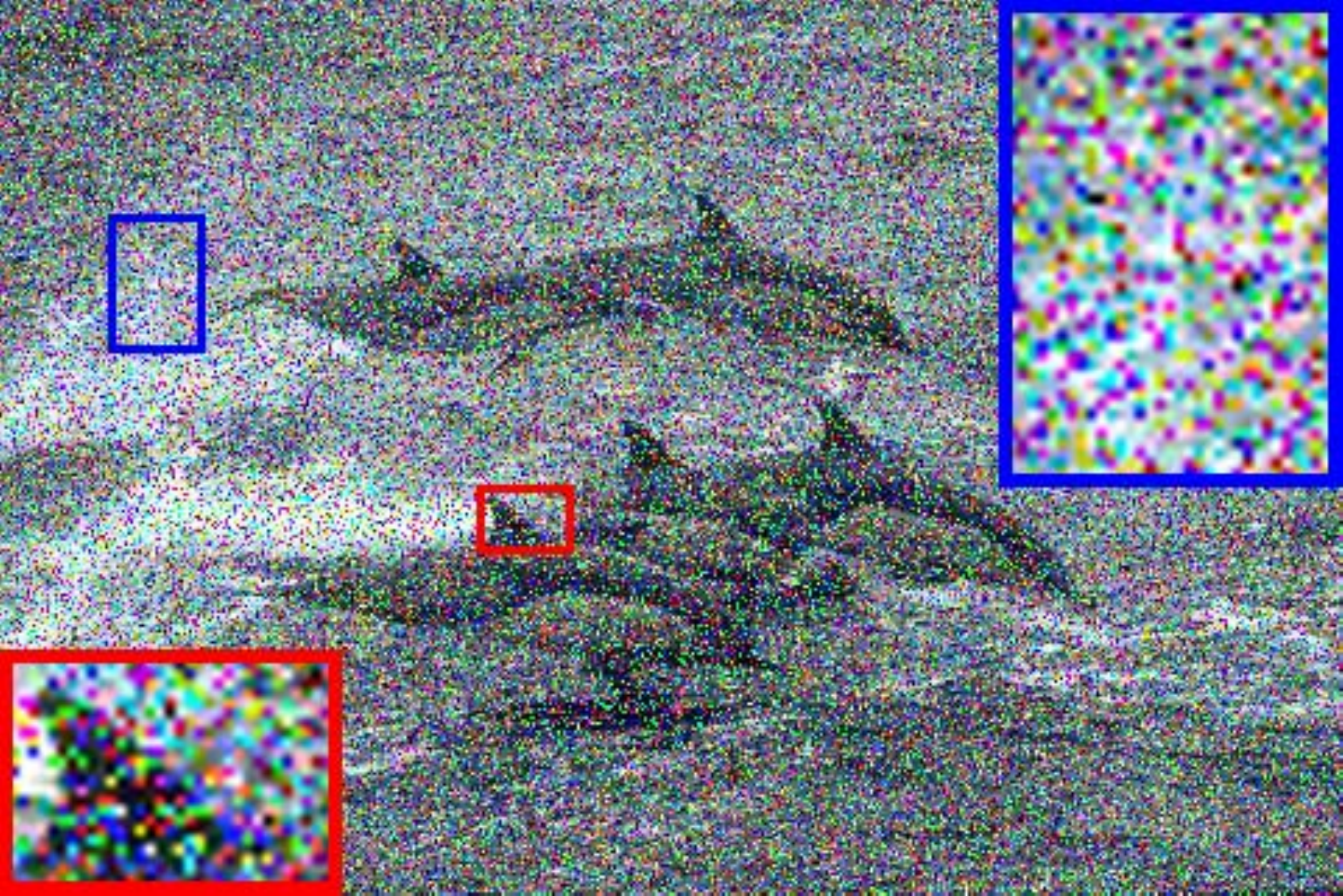}}
\subfigure[{\tiny RPCA (18.62,  0.3955)}]{\includegraphics[width=0.18\textwidth]{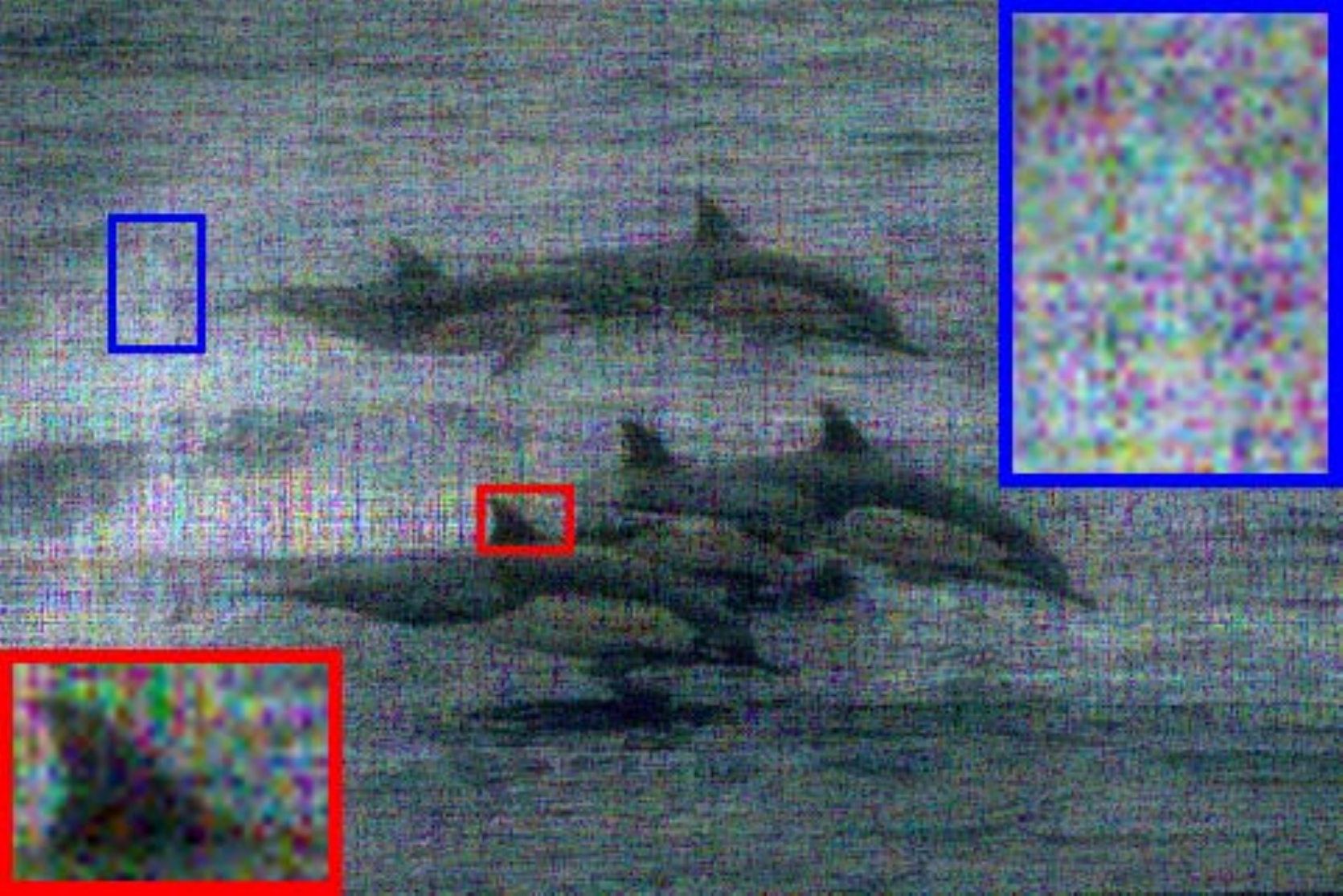}}
\subfigure[{\tiny RMC (25.31,    0.7048)}]{\includegraphics[width=0.18\textwidth]{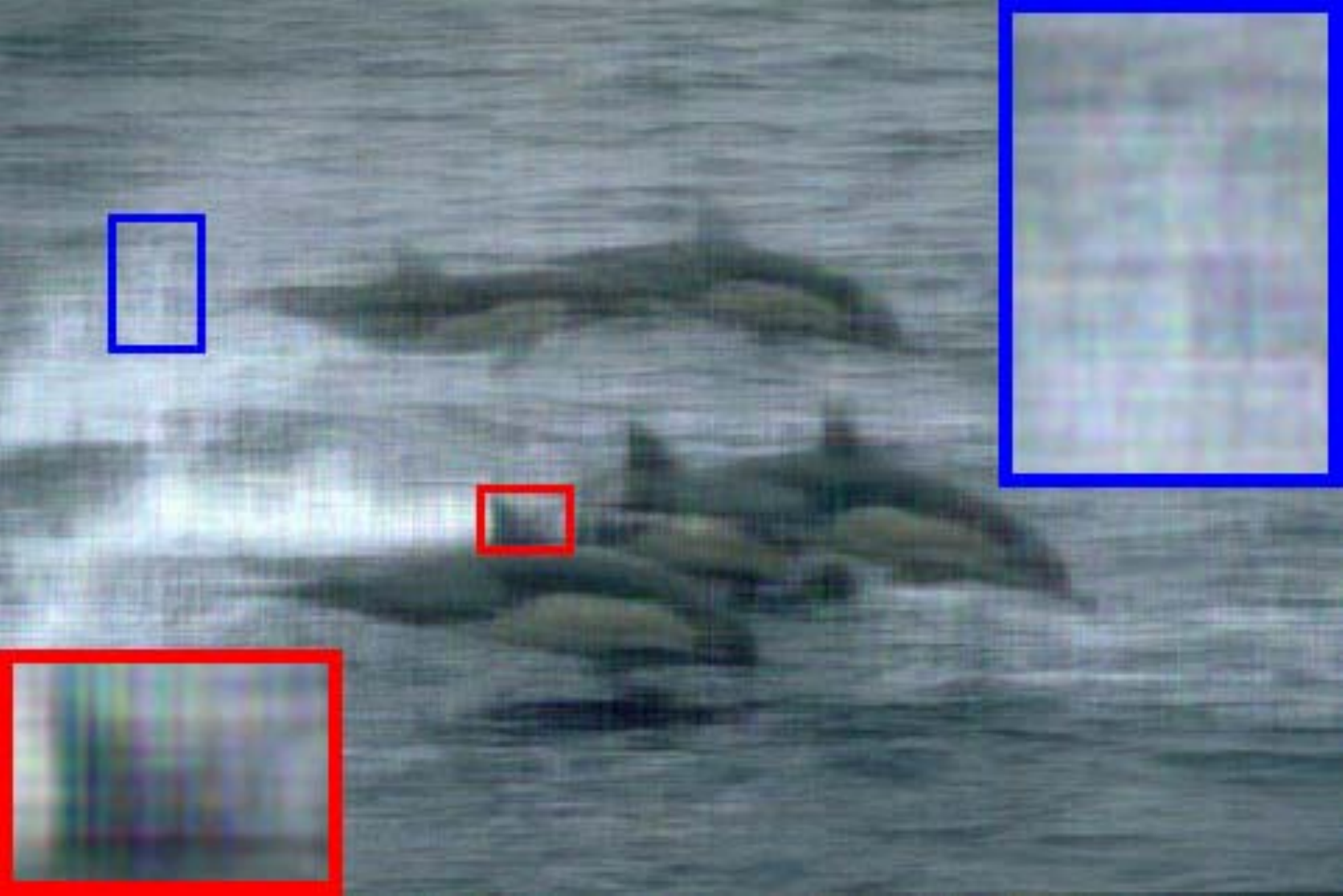}}
\subfigure[{\tiny TRPCA (28.28,   0.7948)}]{\includegraphics[width=0.18\textwidth]{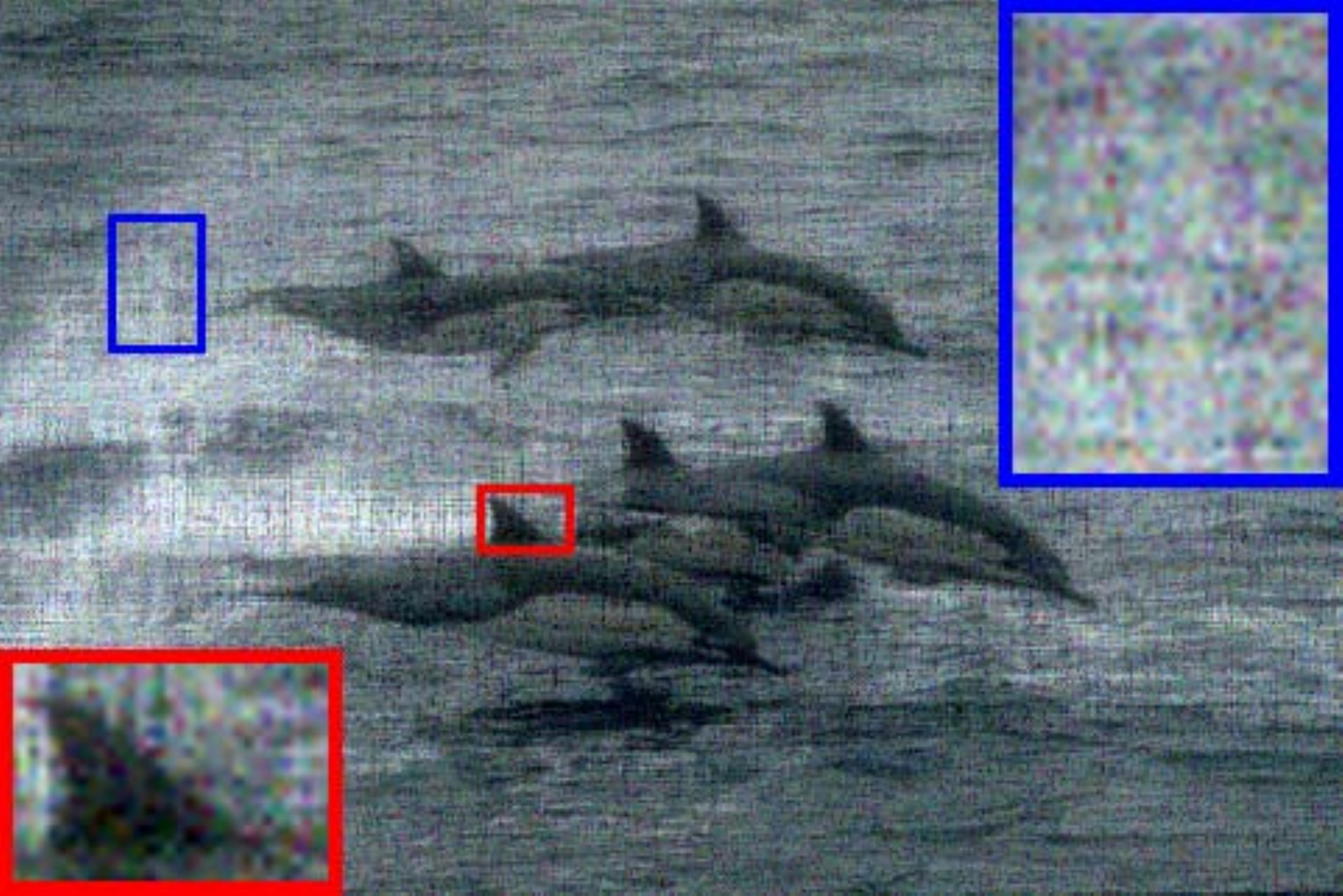}}
\\[-0.2em]
\subfigure[{\tiny BM3D (21.83,    0.4116)}]{\includegraphics[width=0.18\textwidth]{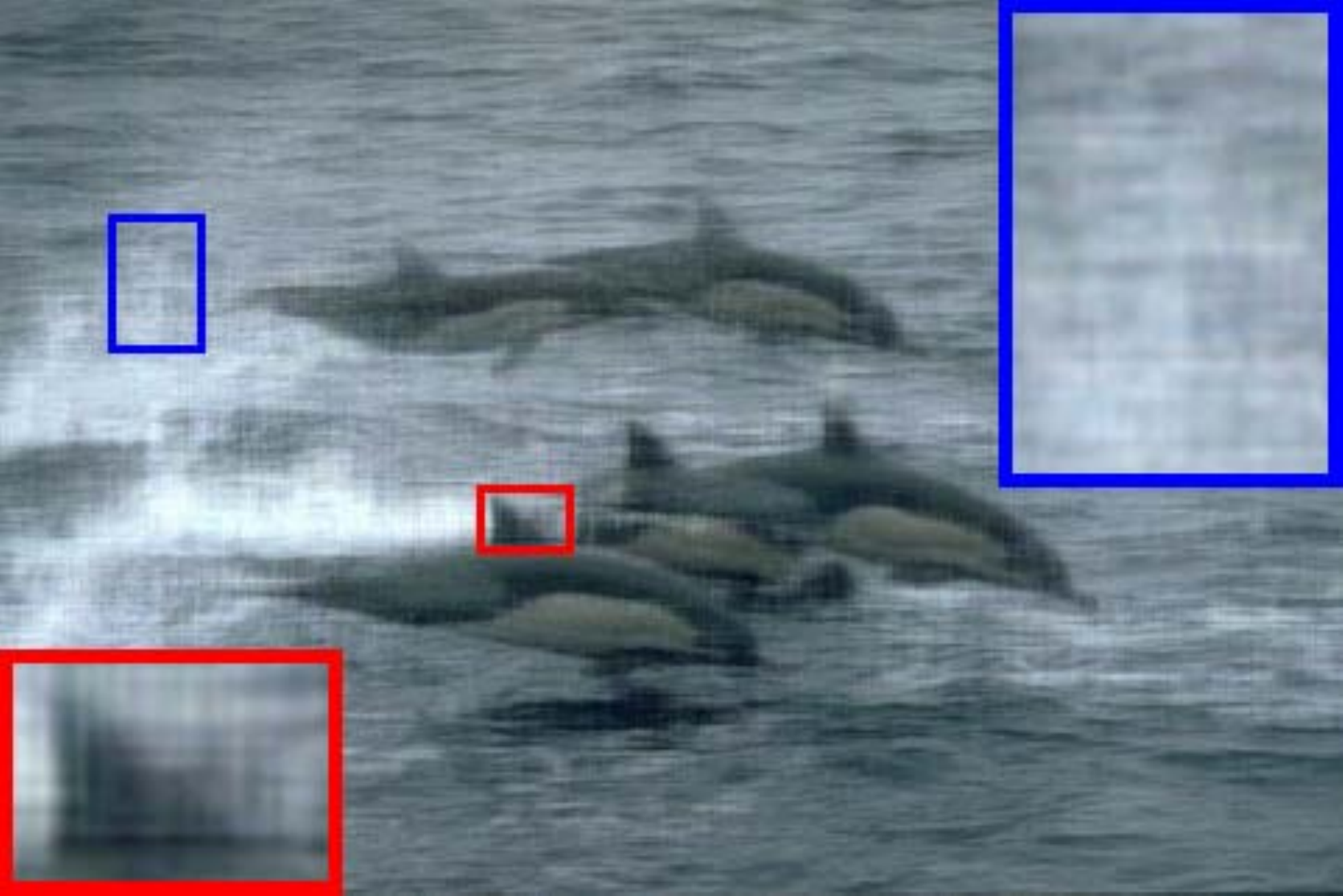}}
\subfigure[{\tiny BM3D+ (27.94,   0.7165)}]{\includegraphics[width=0.18\textwidth]{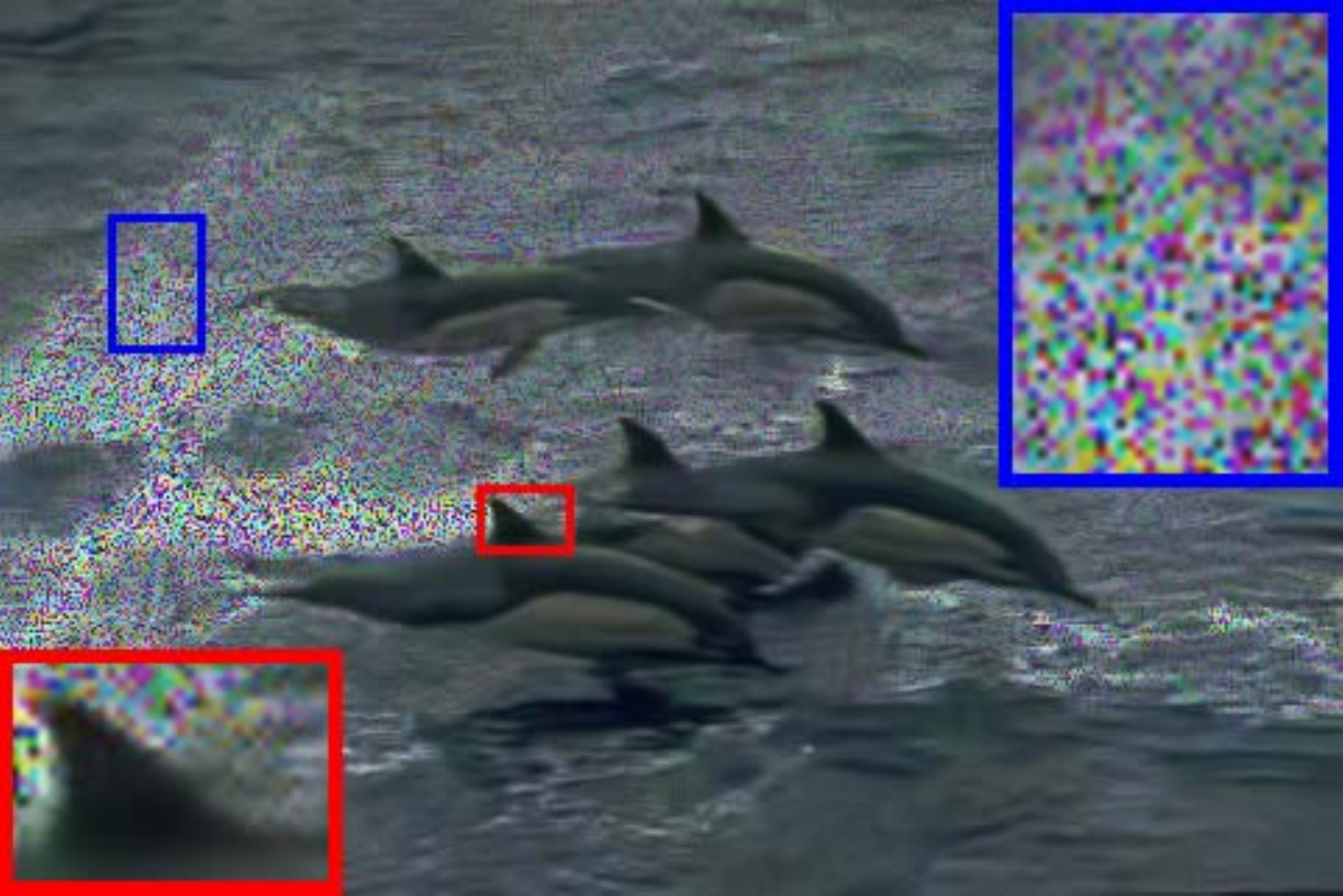}}
\subfigure[{\tiny BM3D++ (28.62,   0.7440)}]{\includegraphics[width=0.18\textwidth]{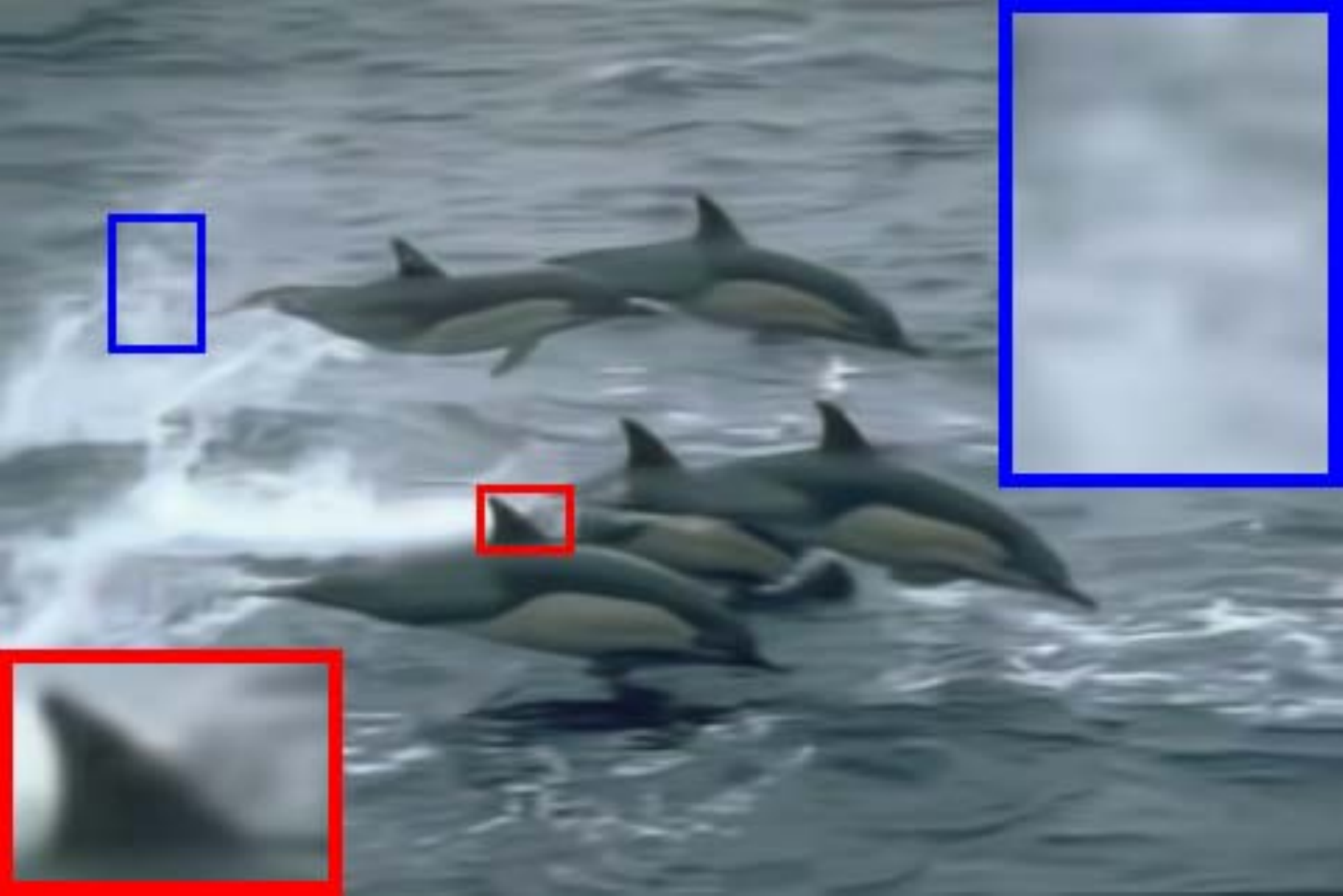}}
\subfigure[{\tiny SNN (26.91,   0.7898)}]{\includegraphics[width=0.18\textwidth]{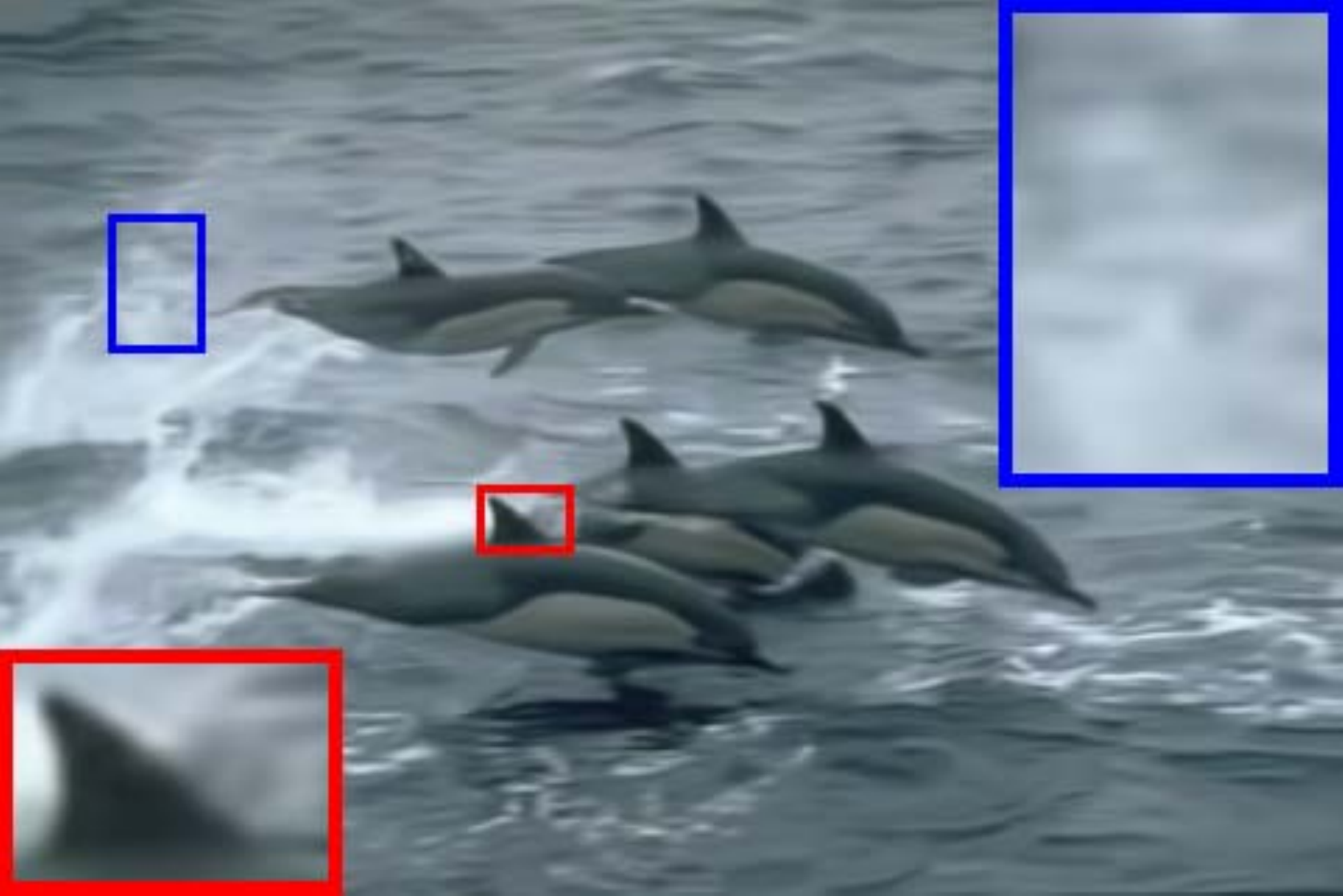}}
\subfigure[{\tiny RTC ({\bf 30.98},    {\bf 0.9044})}]{\includegraphics[width=0.18\textwidth]{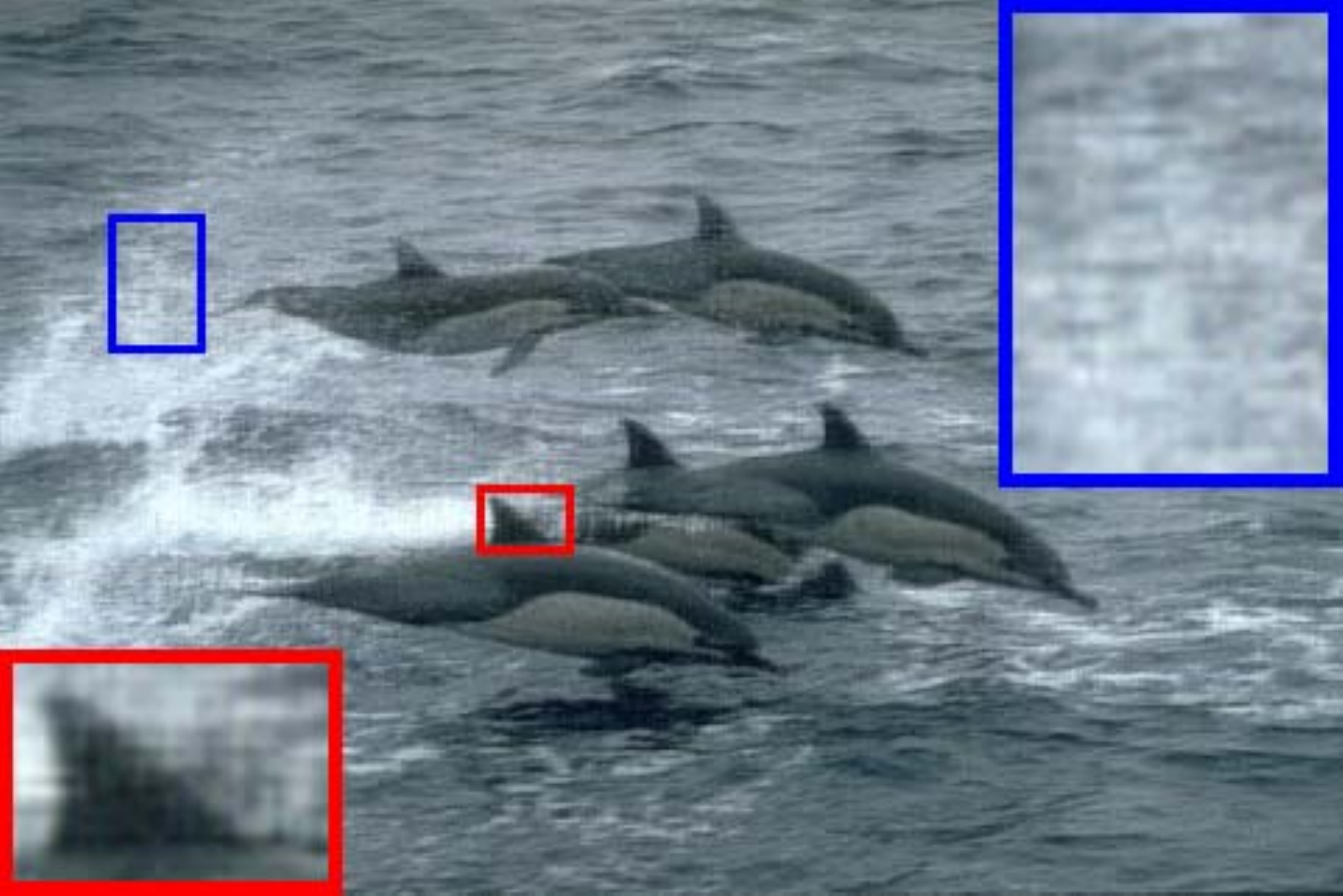}}
\\
\subfigure[{\tiny Original (PSNR, SSIM)}]{\includegraphics[width=0.18\textwidth]{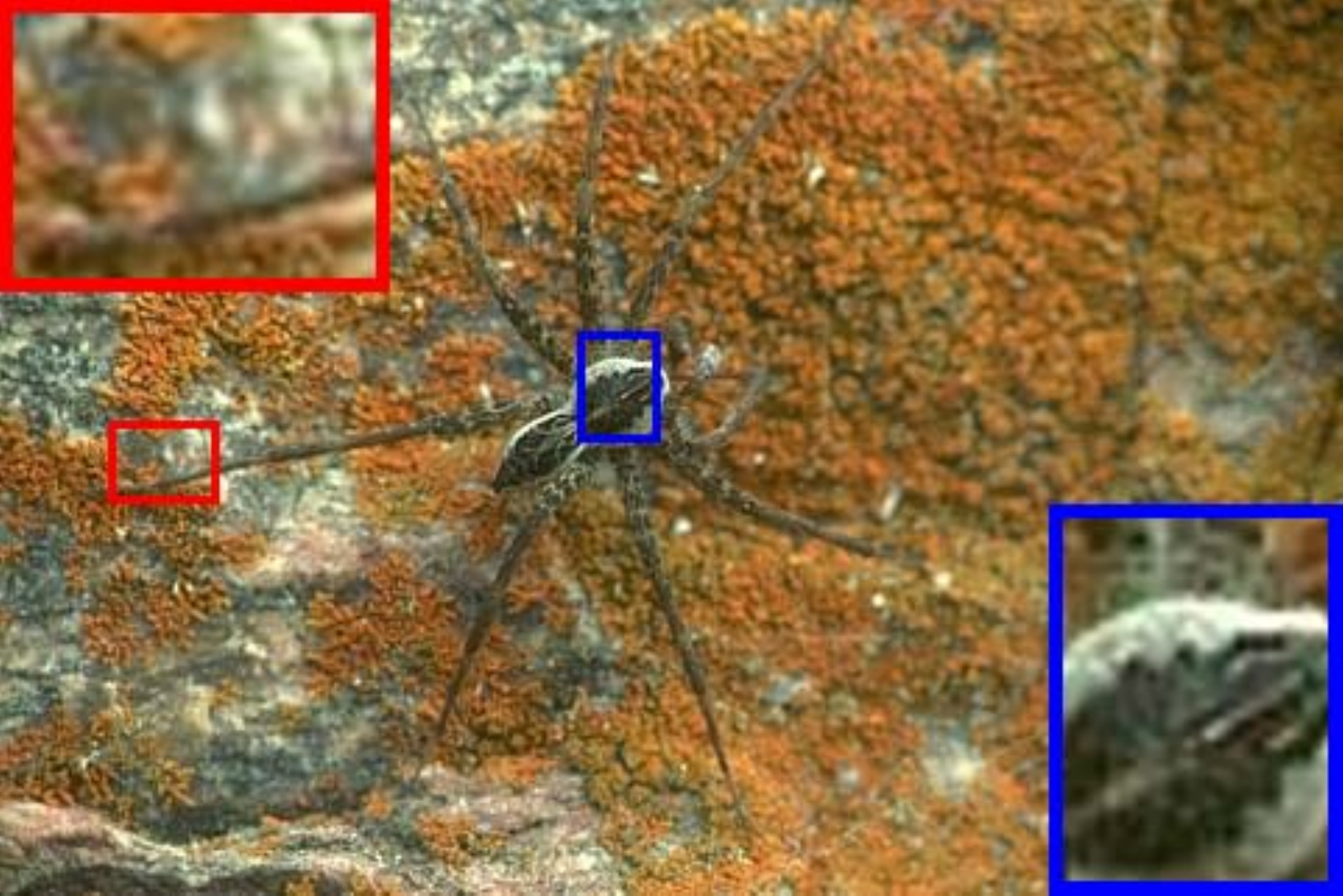}}
\subfigure[{\tiny Noisy (NA, NA)}]{\includegraphics[width=0.18\textwidth]{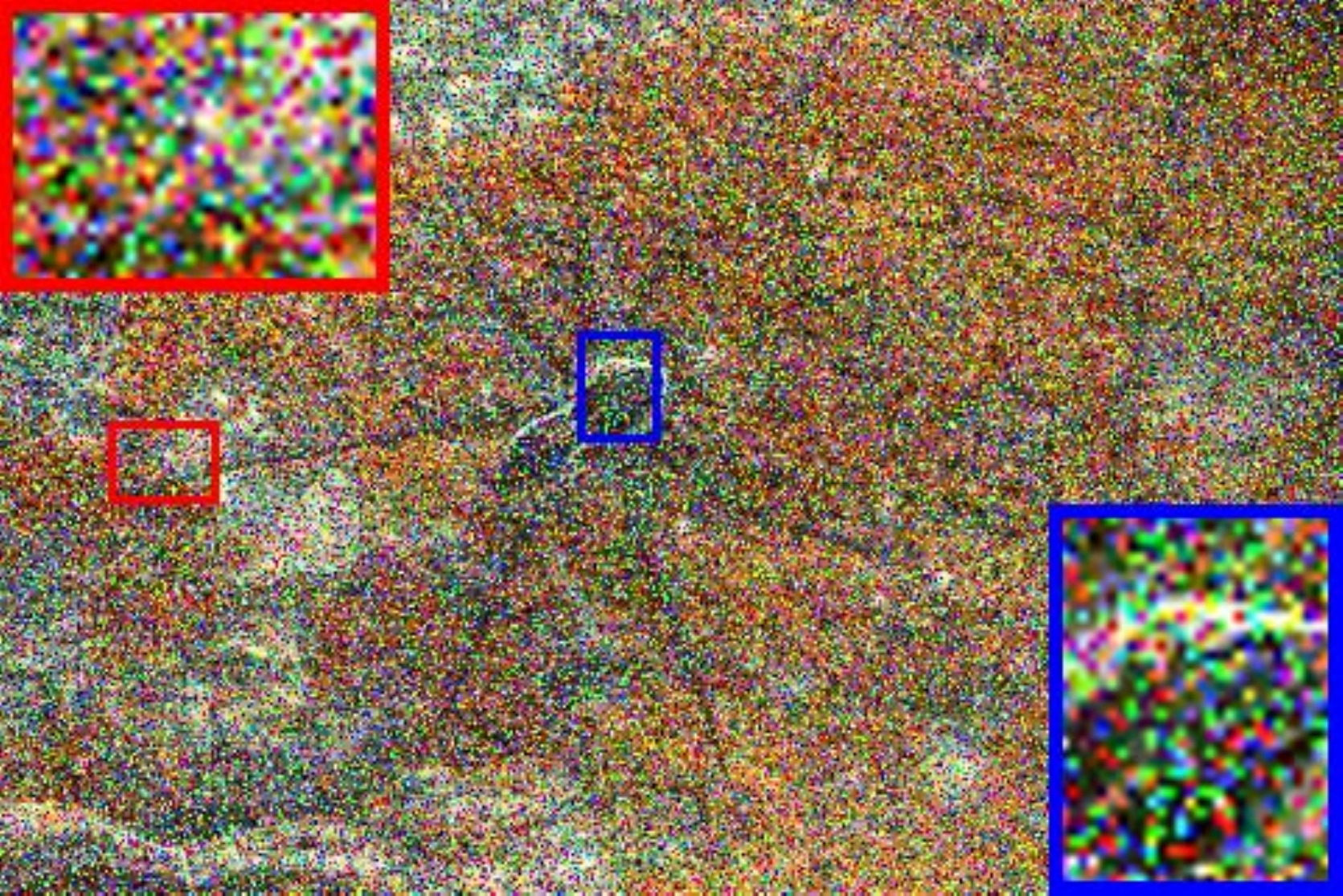}}
\subfigure[{\tiny RPCA (17.88,    0.3928)}]{\includegraphics[width=0.18\textwidth]{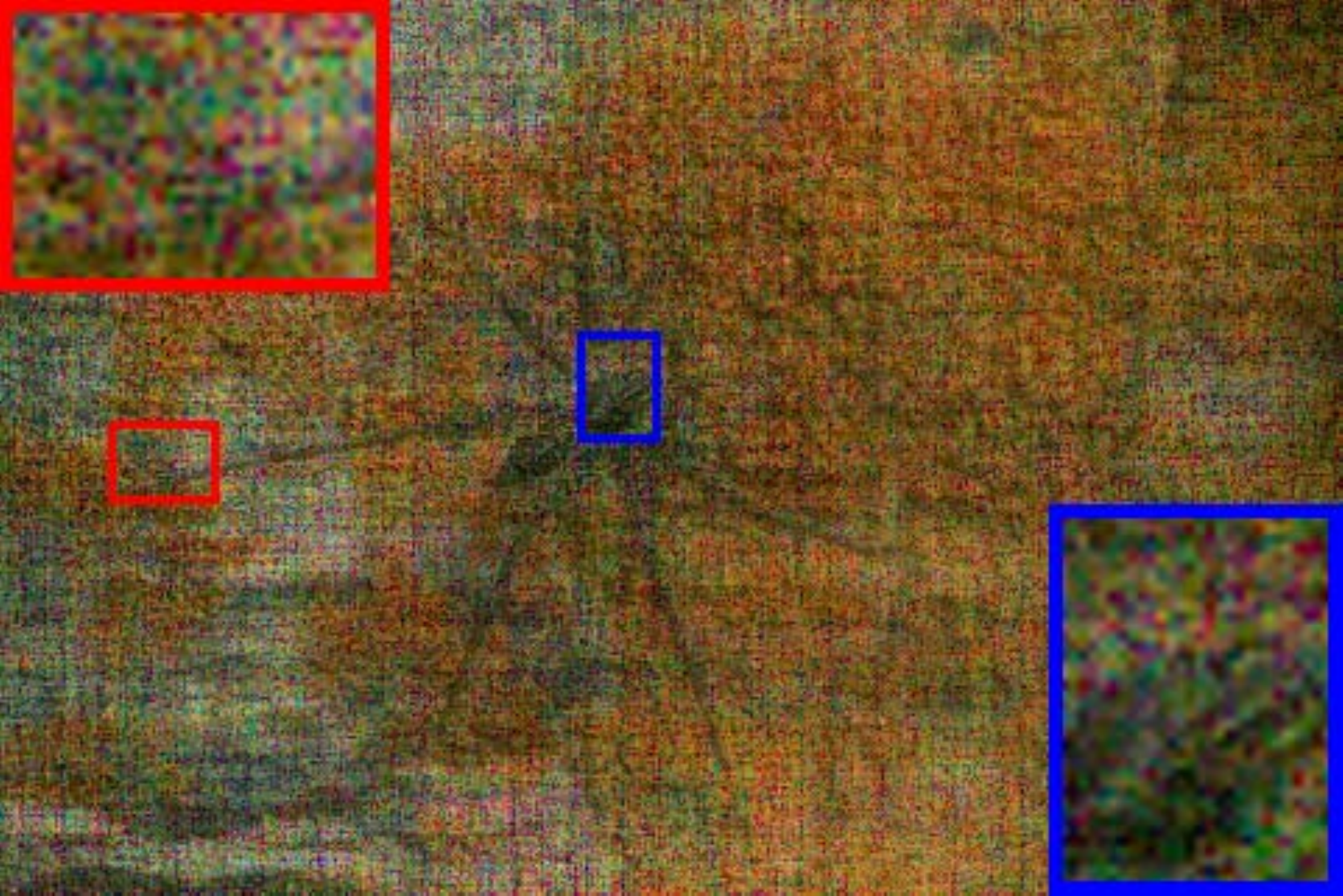}}
\subfigure[{\tiny RMC (24.03    0.6380)}]{\includegraphics[width=0.18\textwidth]{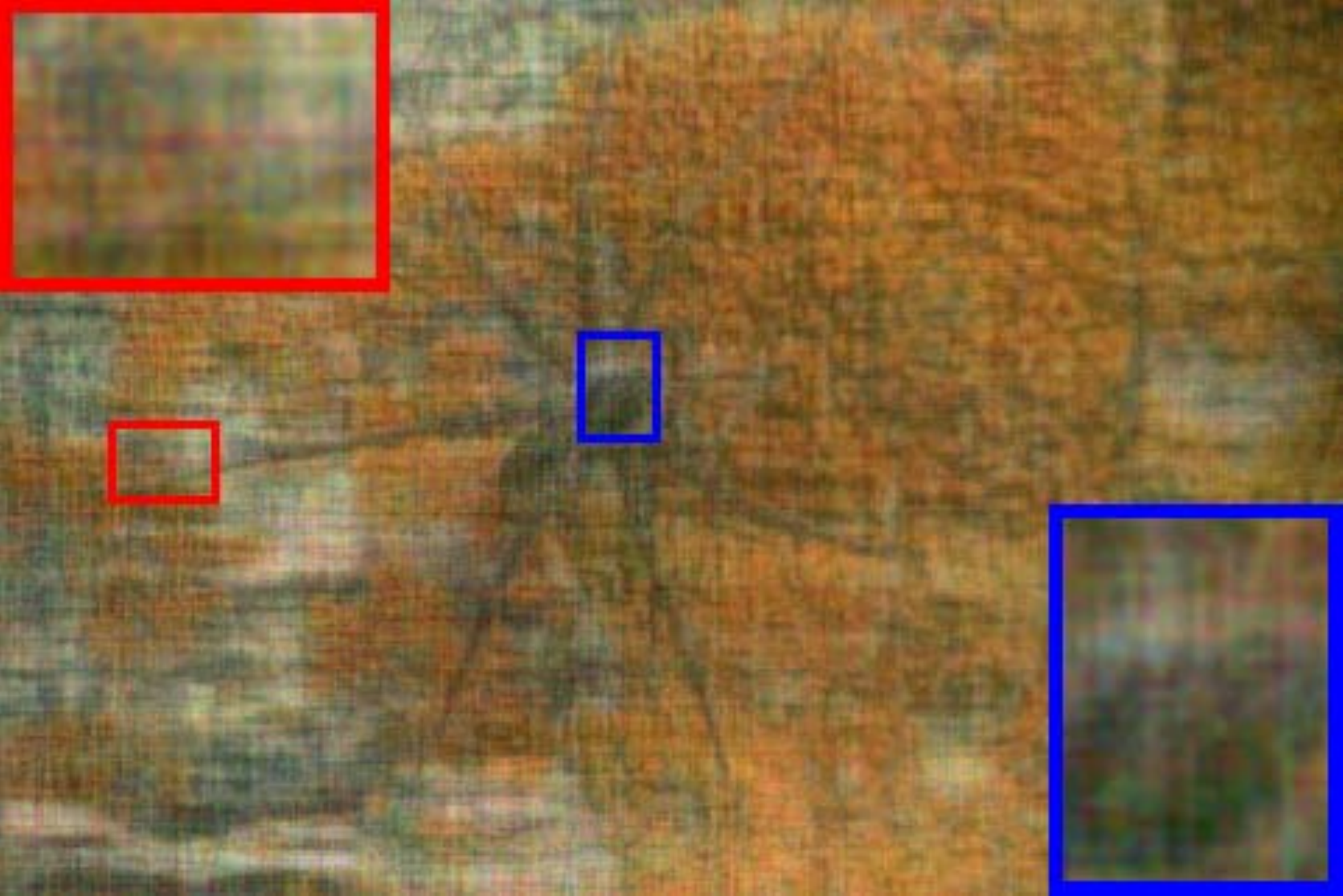}}
\subfigure[{\tiny TRPCA (25.60,    0.7192)}]{\includegraphics[width=0.18\textwidth]{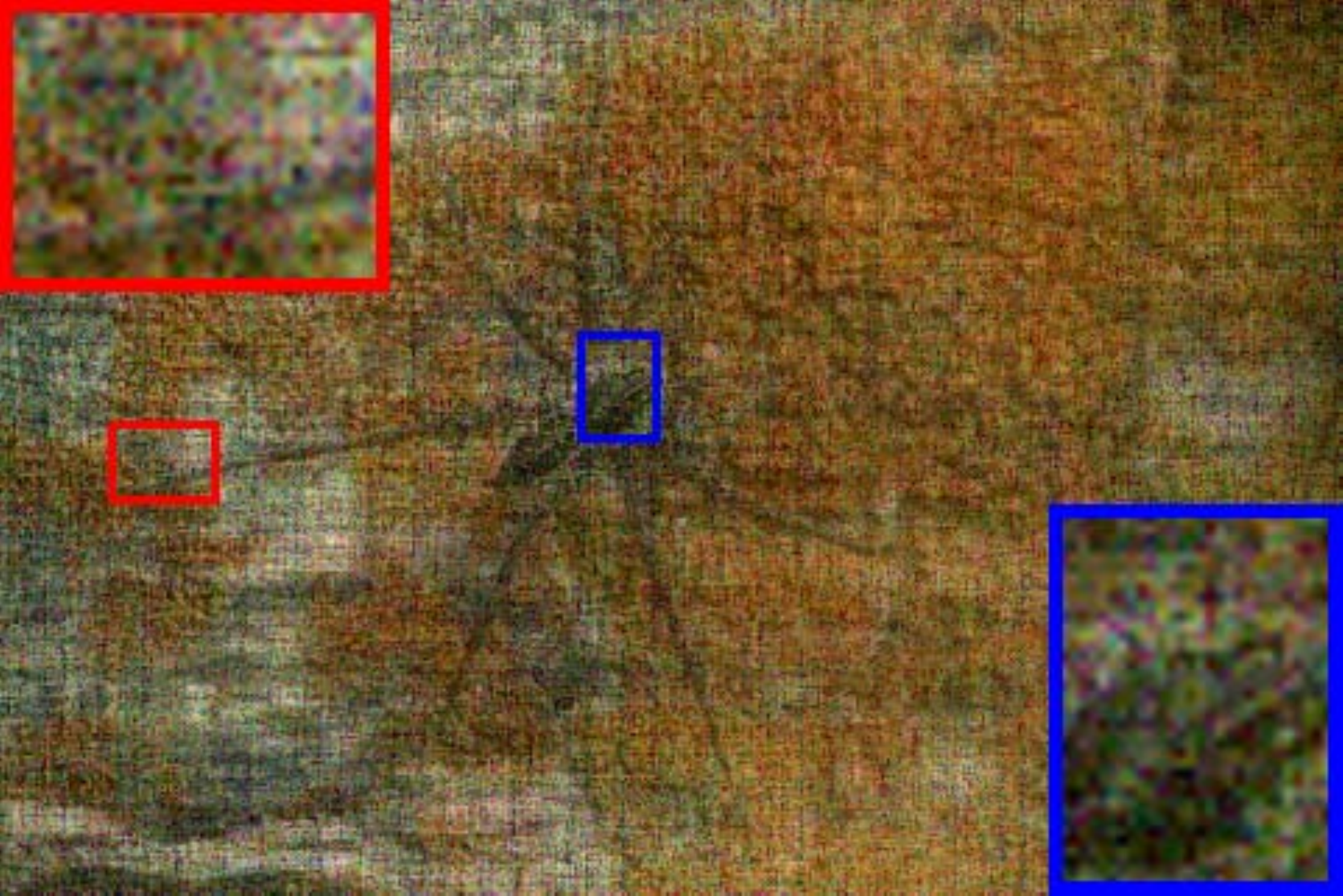}}
\\[-0.2em]
\subfigure[{\tiny BM3D (21.43,    0.4227)}]{\includegraphics[width=0.18\textwidth]{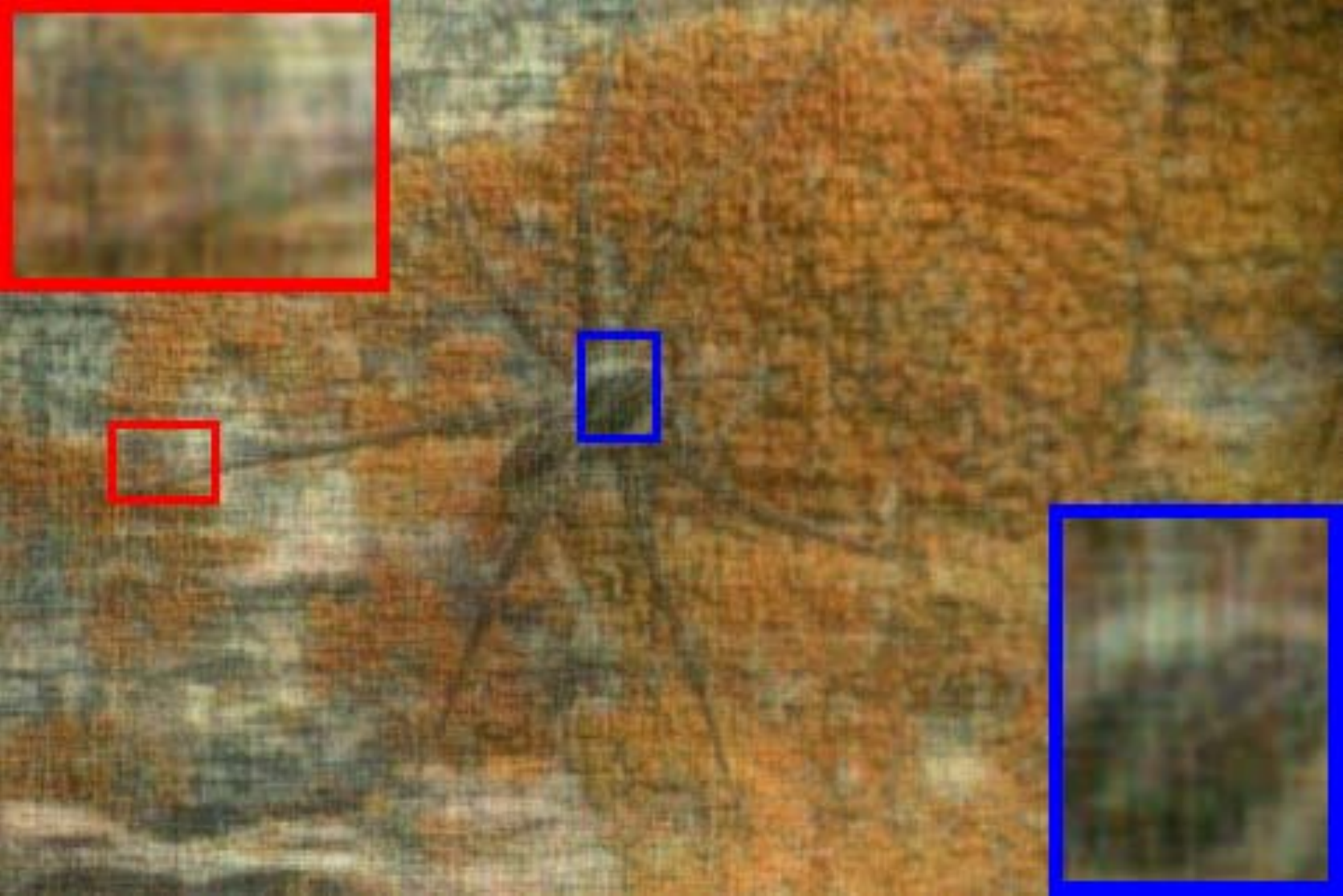}}
\subfigure[{\tiny BM3D+ (25.60,   0.6279)}]{\includegraphics[width=0.18\textwidth]{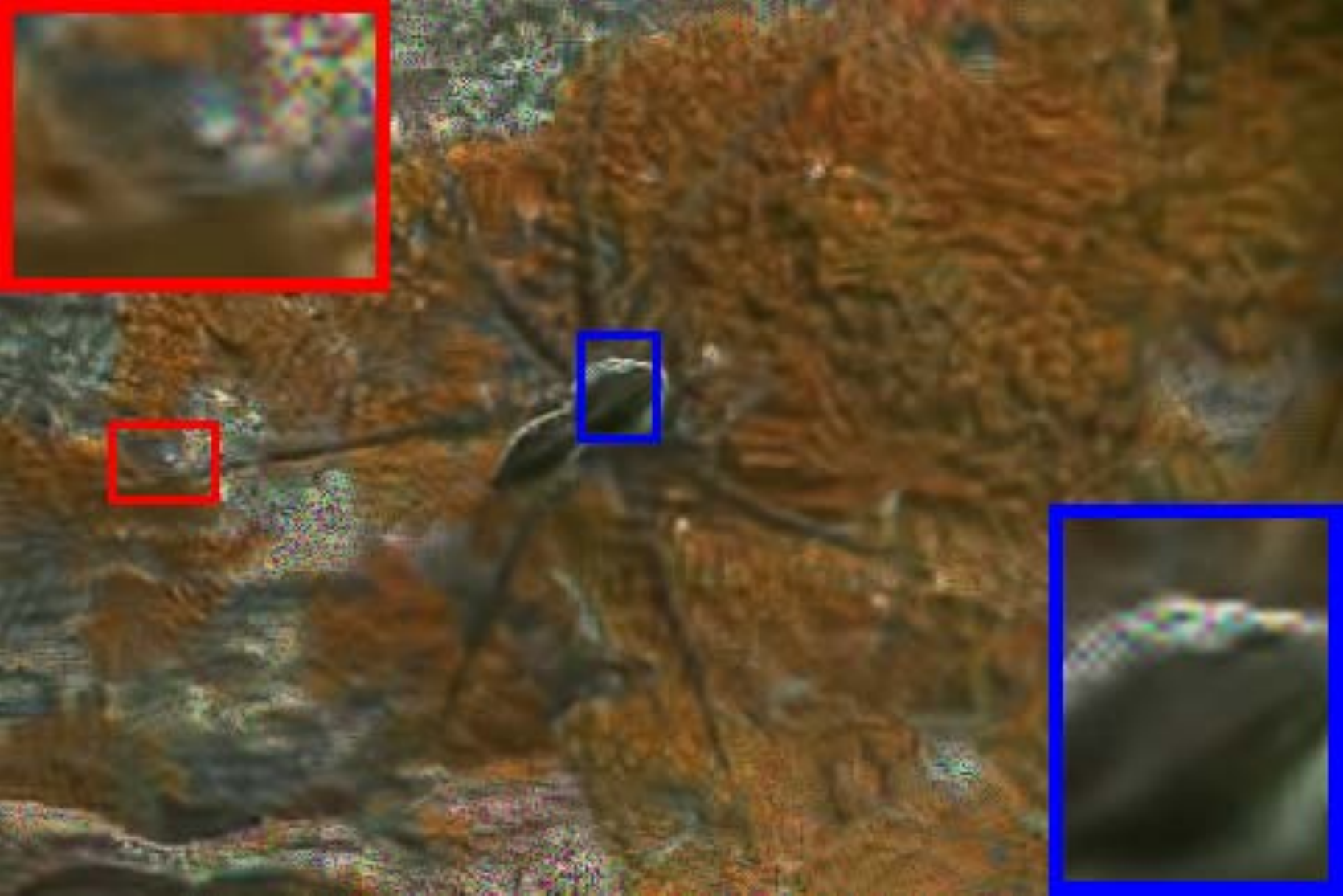}}
\subfigure[{\tiny BM3D++ (26.18,    0.6668)}]{\includegraphics[width=0.18\textwidth]{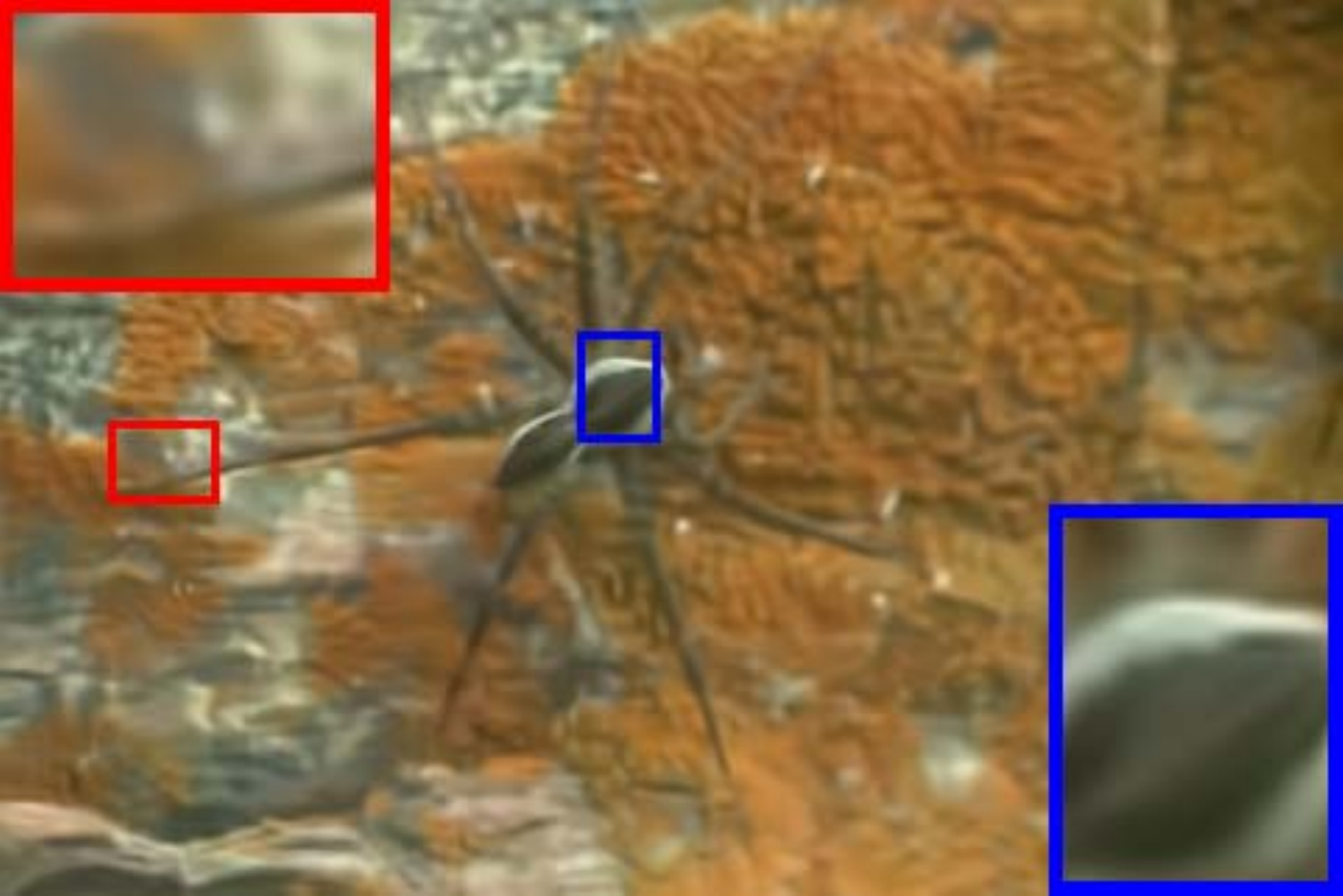}}
\subfigure[{\tiny SNN (25.47,    0.7399)}]{\includegraphics[width=0.18\textwidth]{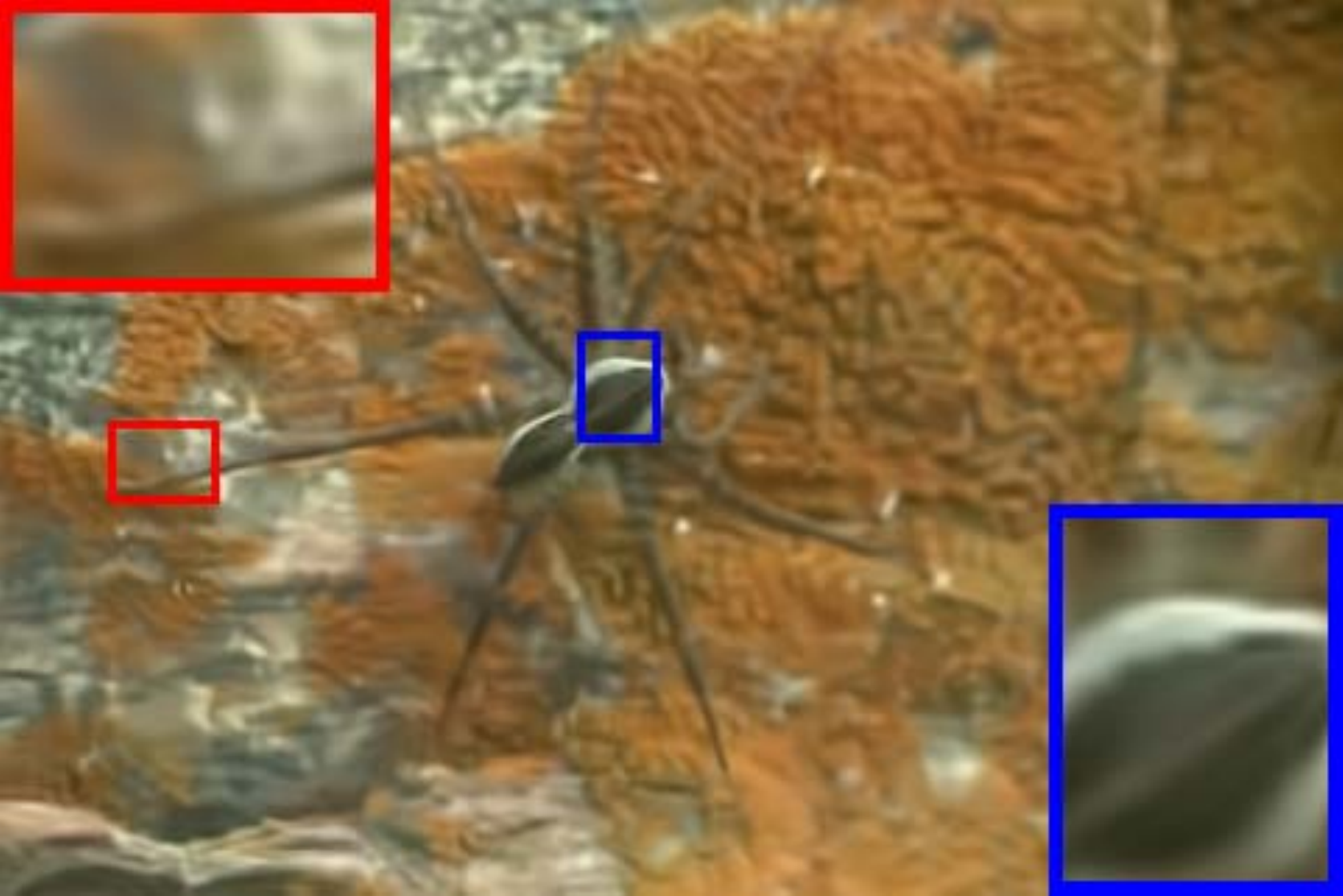}}
\subfigure[{\tiny RTC ({\bf 28.39},    {\bf 0.8593})}]{\includegraphics[width=0.18\textwidth]{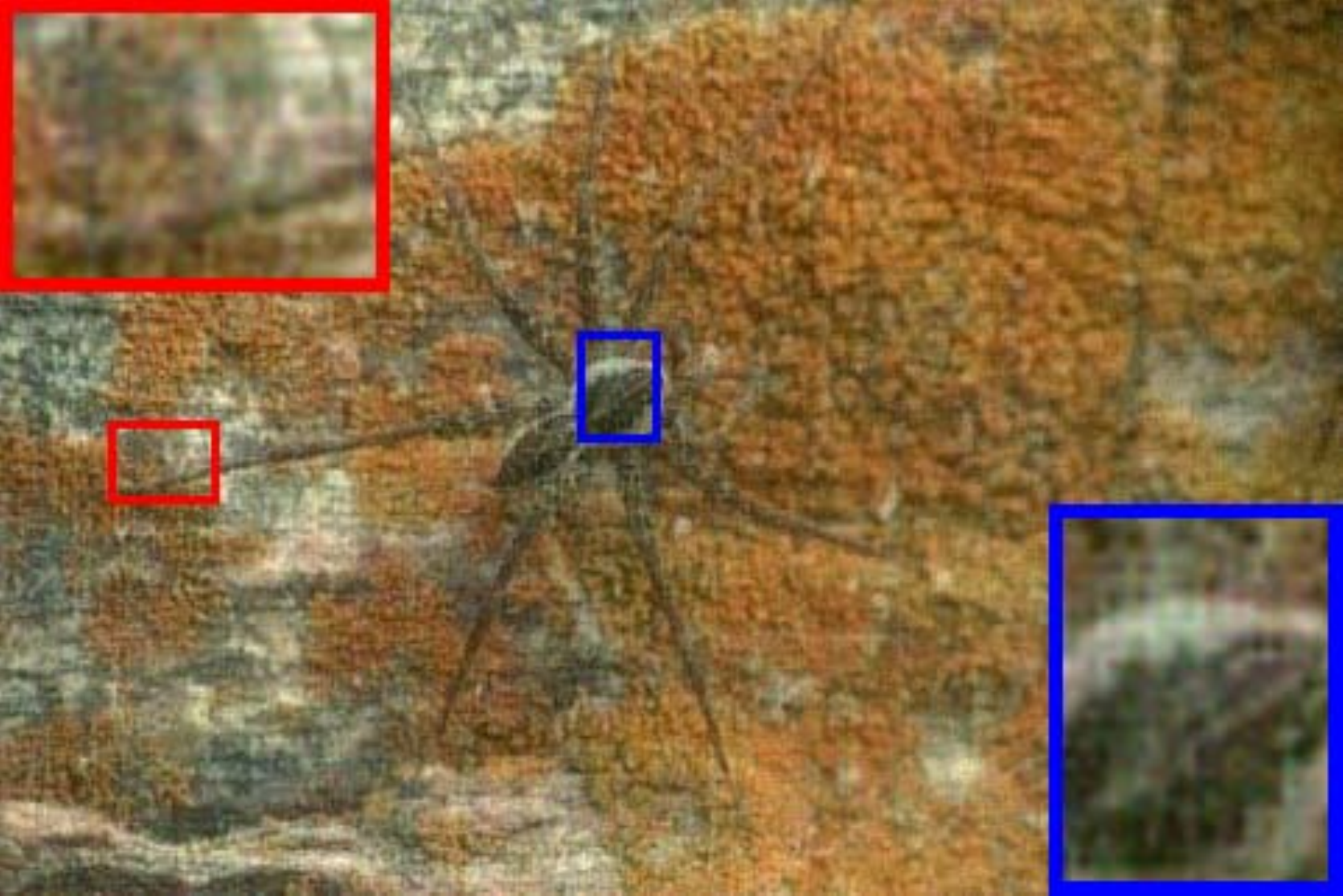}}
\caption{Visual comparison of Image Restoration when $\rho = 0.7$ and $\gamma = 0.3$. Our result contains slightly sharper edges and fewer artifacts.}
\label{fig6}
\end{figure*}

We consider two different situations in which $\rho$ is set to be 0.9 and 0.7 respectively, and change $\gamma$ from 0.01 to 0.03 for each case. Table~\ref{tab3} gives the results in terms of average Peak Signal-to-Noise Ratio (PSNR) and Structural Similarity index (SSIM) when $\sigma = 30$\footnote{Similar results are obtained when we try different values for $\sigma$.}. The best results is in bold text and the second one is underlined. Our algorithm yields the best quantitative results for all the cases, significantly better than the runner-up sometimes. The performance of TRPCA is comparable to our algorithm when $\rho = 0.9$, but it deteriorates dramatically as $\rho$ and $\gamma$ become larger. BM3D is not able to achieve acceptable results, especially when the fraction of missing pixels is relatively high, since it considers removing the noise from the images purely. With an additional completion step, BM3D+ and BM3D++ exhibit remarkably improved performance. Three tensor-based methods, TRPCA, SNN and RTC, perform much better than two matrix-based approaches, RPCA and RMC. The reason is that RPCA and RMC, which conduct the matrix recovery on each channel independently, are not capable of exploiting the information across channels, while the tensor-based methods can take the advantage of the multi-channel structure. We also see that our quantitative results are much better than those obtained by SNN, which verifies that t-SVD is more suitable for capturing the \lq\lq spatial-shifting\rq\rq~characteristics in natural images compared with Tucker decomposition.

In Figure~\ref{fig5}, we give the results obtained by various methods on all 50 images when $\rho = 0.7$ and $\gamma = 0.3$. Our algorithm outperforms the other methods quantitatively for most images. From the two examples in Figure~\ref{fig6}, we see that our recovered images contain slightly sharper edges and fewer artifacts, exhibited in the enlarged views of the corresponding areas in red and blue boxes.

\subsection{Video Background Modeling}
\label{sec7:sub3}

Another possible application of our algorithm is the background modeling problem, a crucial task in video surveillance, which is to estimate a good model for the background variations in a scene. Due to the correlation between frames, it is reasonably to believe that the background variations are approximately low-rank. Foreground objects generally occupy only a small fraction of the image pixels and hence can be naturally treated as sparse errors.

\begin{figure*}[!t]
\centering
\subfigure{\includegraphics[width = 0.18\textwidth]{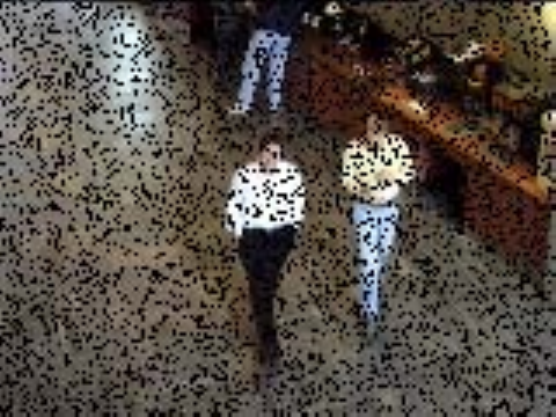}}  \subfigure{\includegraphics[width = 0.18\textwidth]{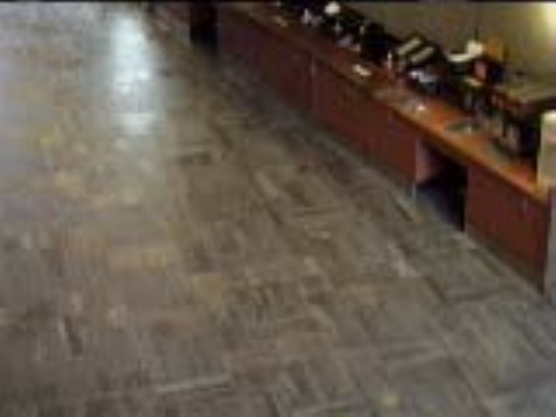}}  \subfigure{\includegraphics[width = 0.18\textwidth]{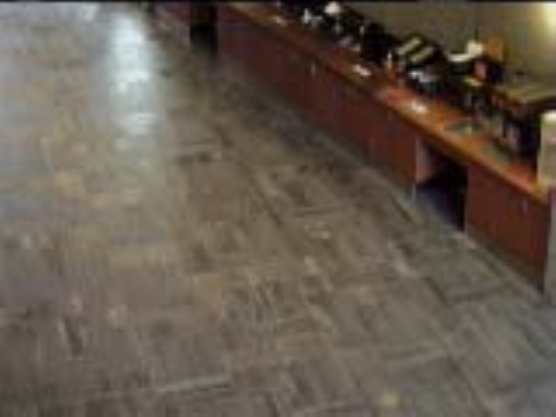}}
\subfigure{\includegraphics[width = 0.18\textwidth]{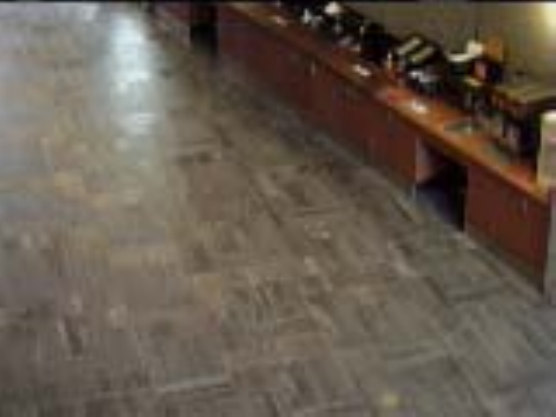}}
\subfigure{\includegraphics[width = 0.18\textwidth]{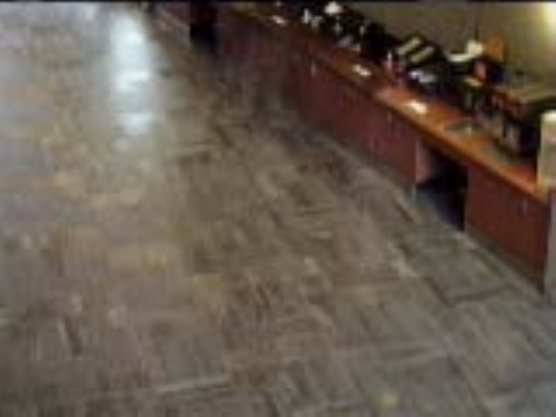}}
\\[-0.3em]
\subfigure{\includegraphics[width = 0.18\textwidth]{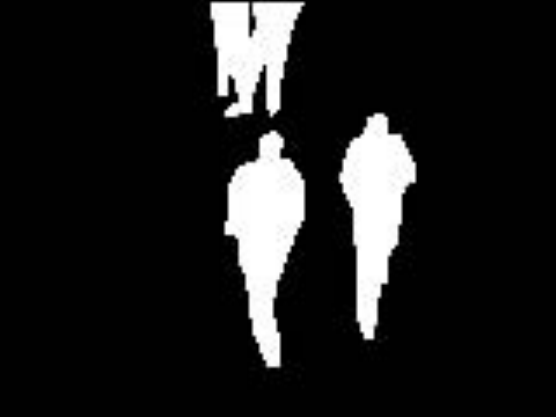}}
\subfigure{\includegraphics[width = 0.18\textwidth]{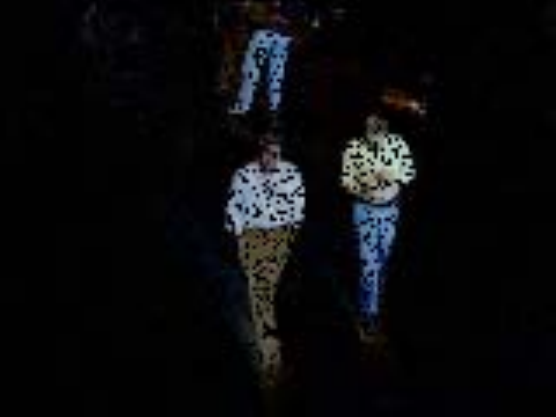}}  \subfigure{\includegraphics[width = 0.18\textwidth]{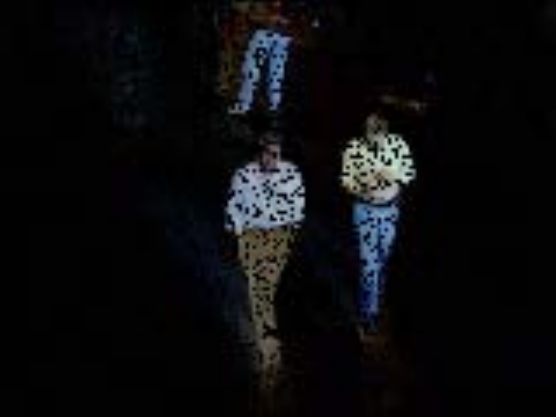}}
\subfigure{\includegraphics[width = 0.18\textwidth]{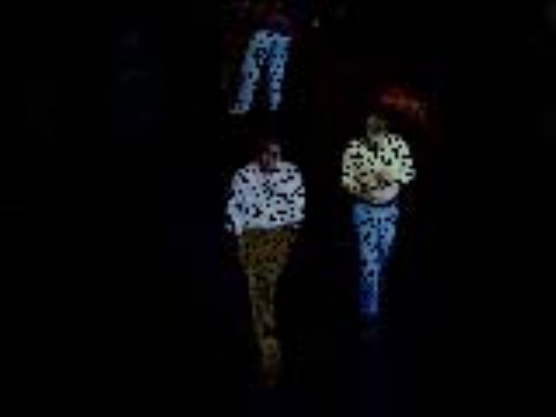}}
\subfigure{\includegraphics[width = 0.18\textwidth]{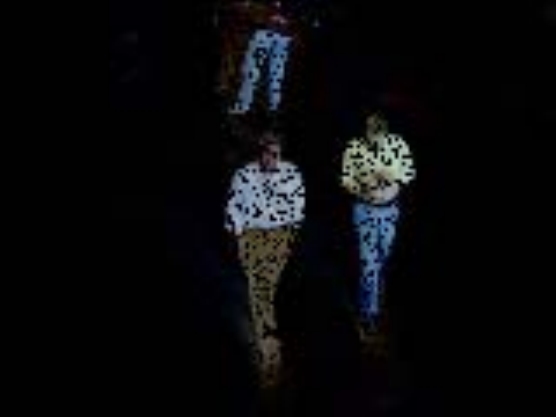}}
\\[-0.1em]
\subfigure{\includegraphics[width = 0.18\textwidth]{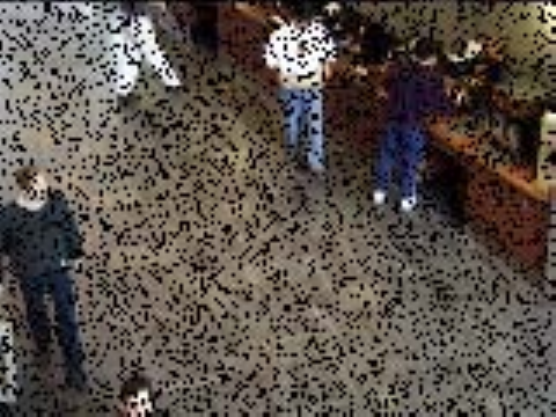}}  \subfigure{\includegraphics[width = 0.18\textwidth]{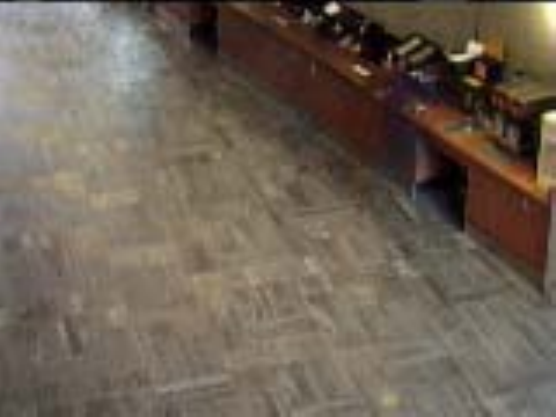}}  \subfigure{\includegraphics[width = 0.18\textwidth]{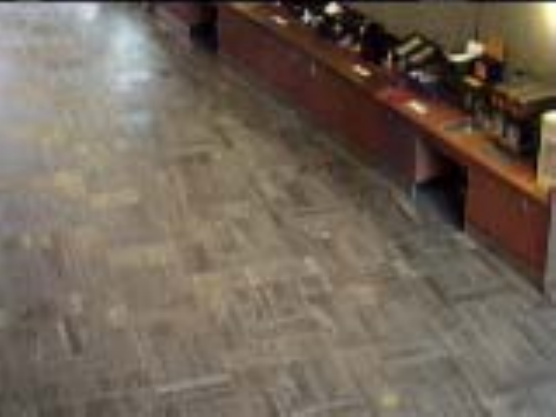}}
\subfigure{\includegraphics[width = 0.18\textwidth]{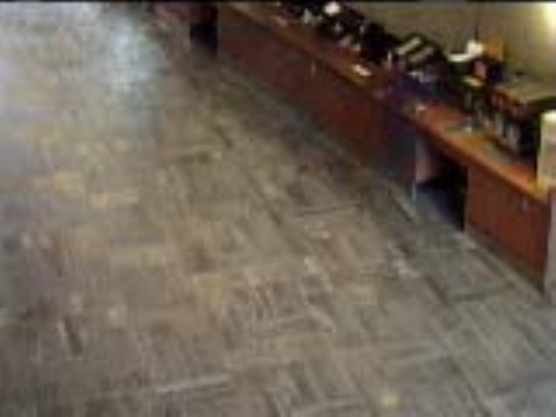}}
\subfigure{\includegraphics[width = 0.18\textwidth]{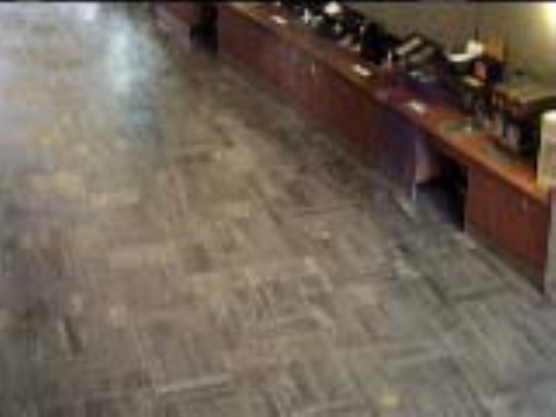}}
\\[-0.3em]
\subfigure{\includegraphics[width = 0.18\textwidth]{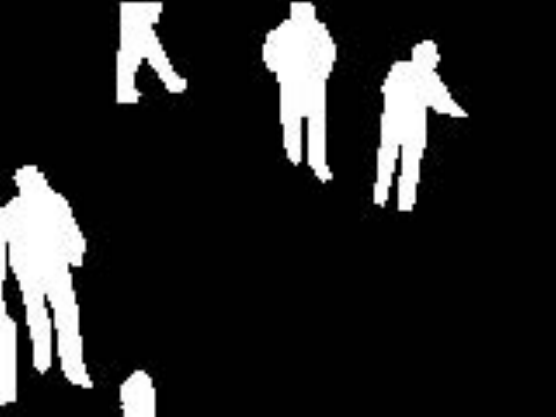}}
\subfigure{\includegraphics[width = 0.18\textwidth]{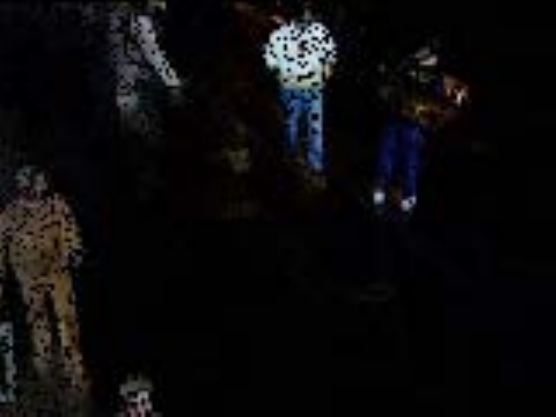}}  \subfigure{\includegraphics[width = 0.18\textwidth]{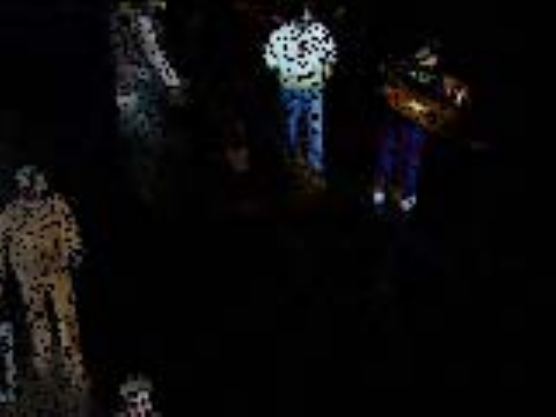}}
\subfigure{\includegraphics[width = 0.18\textwidth]{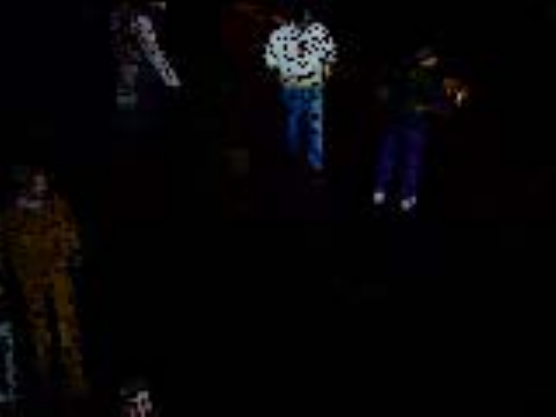}}
\subfigure{\includegraphics[width = 0.18\textwidth]{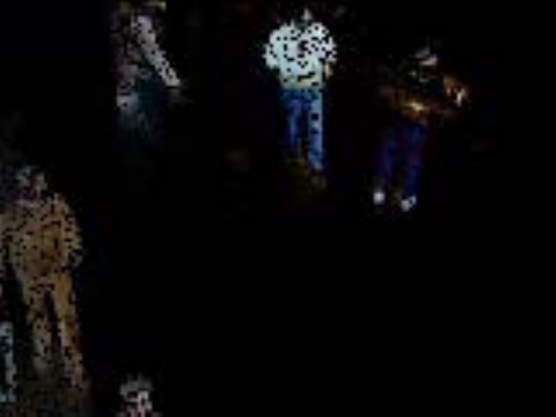}}
\caption{Background modeling results on two frames from Bootstrap video sequence when $\rho = 0.8$. From left column to right column, there are respectively the input frames and corresponding ground-truth of the foreground objects, and the low-rank and sparse components obtained by RPCA, RMC, TRPCA and RTC.}
\label{fig7}
\end{figure*}

In this experiment, we investigate the capability of our algorithm to remove the foreground objects and reconstruct the background with a proportion $\rho$ of pixels available, and compare it with RPCA, RMC and TRPCA. Given a video consisting of $n$ color frames of size $h \times w$, we stack every frame in each color channel as a column vector of size $q \times 1$ where $q = h \times w$ and then collect all column vectors into a matrix of size $3q \times n$ for RPCA and RMC, and into a tensor of size $q \times n \times 3$ for TRPCA and RTC respectively. The parameter $\lambda$ is set to be $\lambda = 1/\sqrt{3q}$ for RPCA and RMC, and $\lambda = 1/\sqrt{q}$ for TRPCA.

We randomly extract 200 frames from each of three popular color videos, Bootstrap, Hall and ShoppingMall\footnote{\url{http://perception.i2r.a-star.edu.sg/bk_model/bk_index.html}}. Then, the input data is generated by masking 20\% of the randomly selected pixels for each frame. As illustrated in Figure~\ref{fig7},  all the methods can separate the background and foreground effectively. We can see that the separation results obtained by our method are slightly better than other approaches visually. In particular, our method extracts the foreground objects with fewer ghosting effects.

\begin{figure}
\centering
\subfigure{\includegraphics[width = 0.18\textwidth]{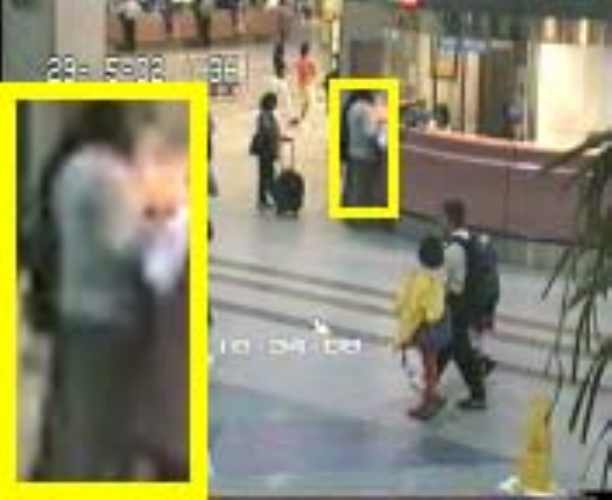}}  \subfigure{\includegraphics[width = 0.18\textwidth]{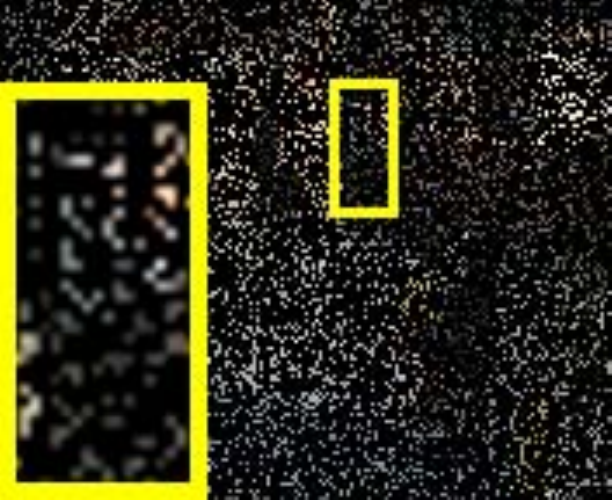}}  \subfigure{\includegraphics[width = 0.18\textwidth]{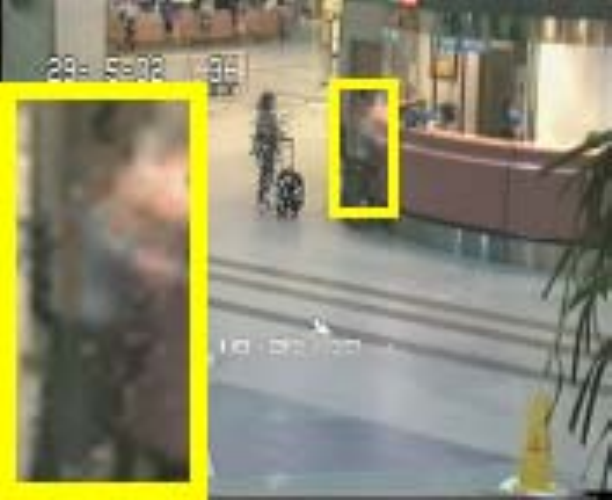}}
\subfigure{\includegraphics[width = 0.18\textwidth]{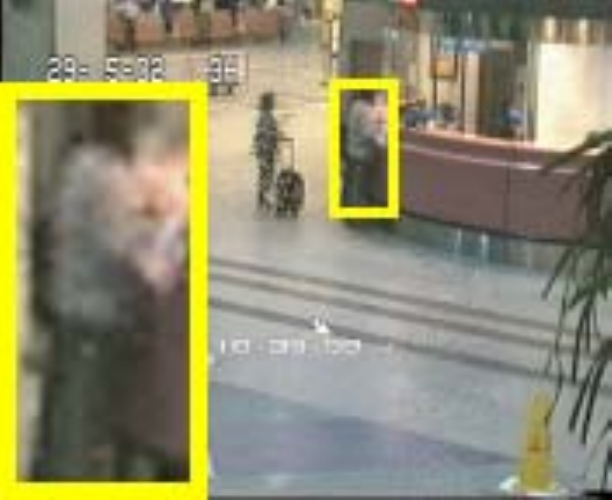}}
\\[-0.3em]
\subfigure{\includegraphics[width = 0.18\textwidth]{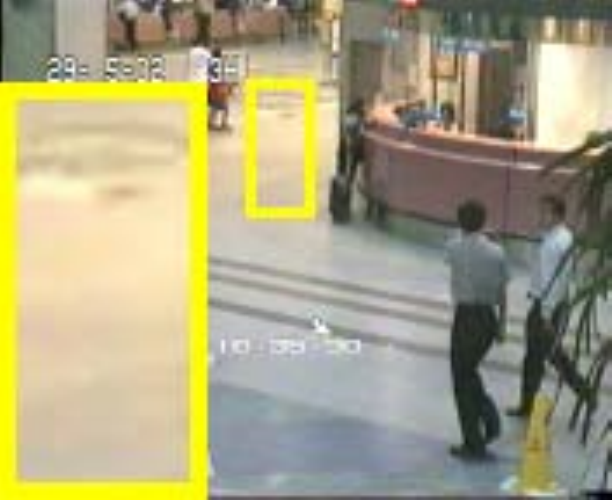}}  \subfigure{\includegraphics[width = 0.18\textwidth]{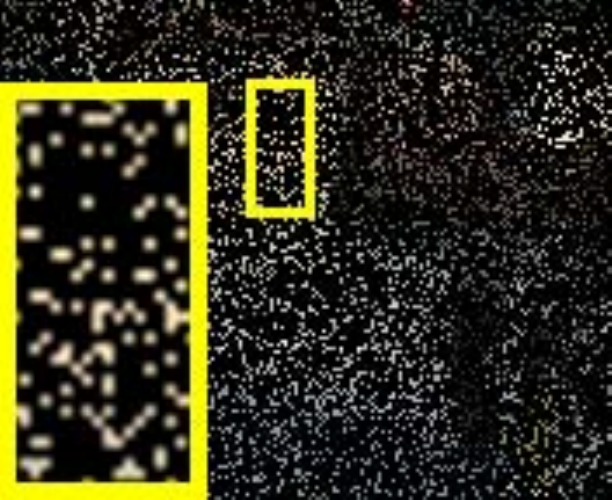}}  \subfigure{\includegraphics[width = 0.18\textwidth]{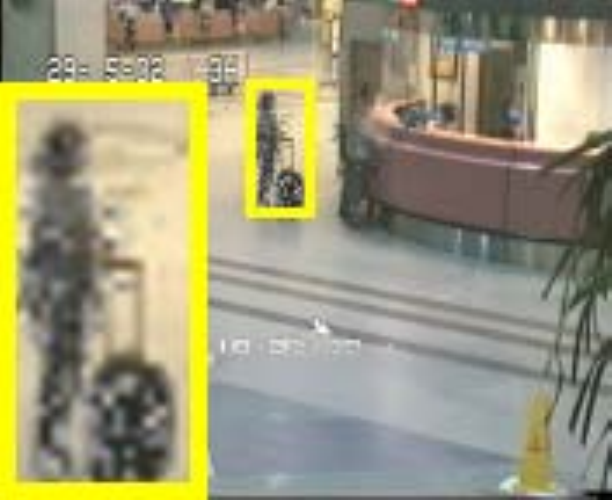}}
\subfigure{\includegraphics[width = 0.18\textwidth]{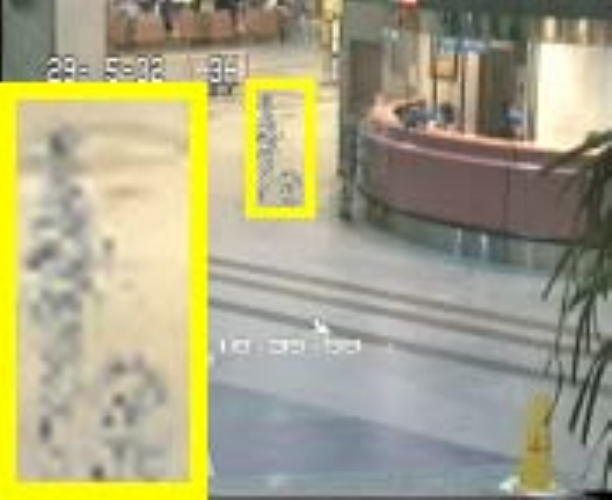}}
\\[-0.1em]
\subfigure{\includegraphics[width = 0.18\textwidth]{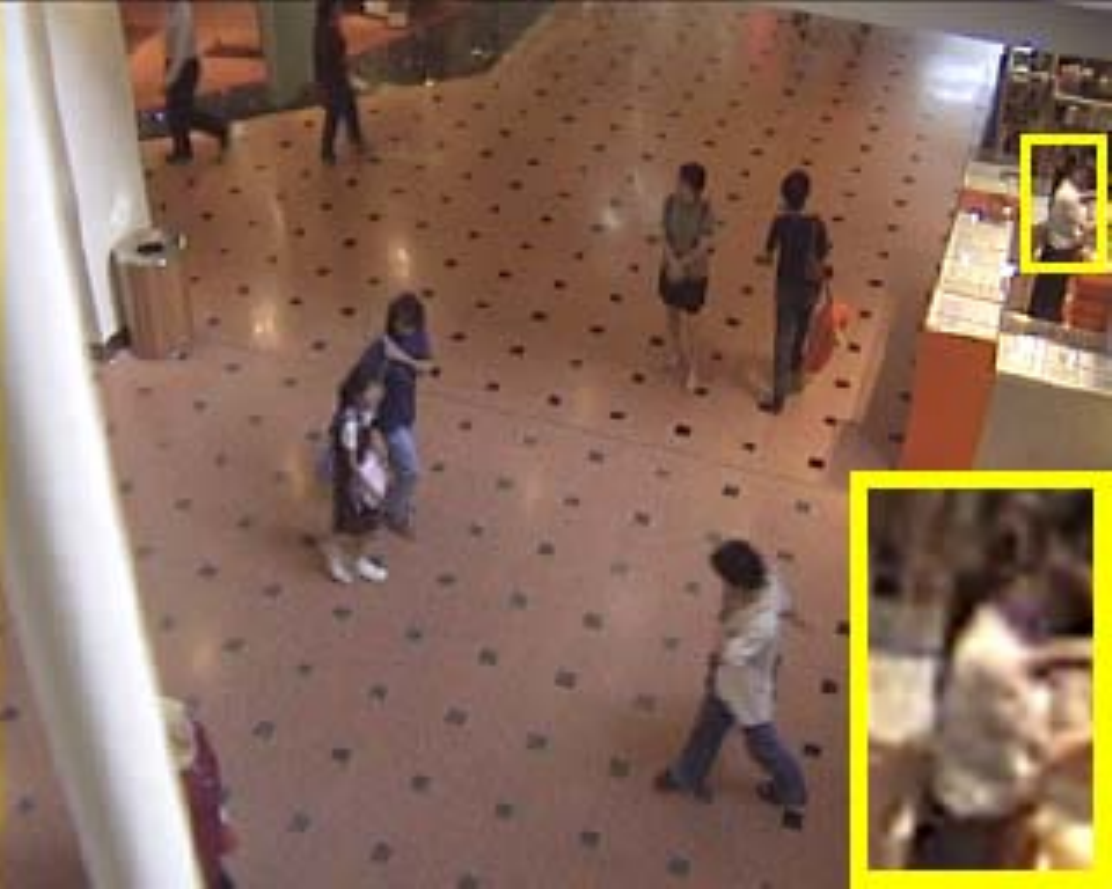}}  \subfigure{\includegraphics[width = 0.18\textwidth]{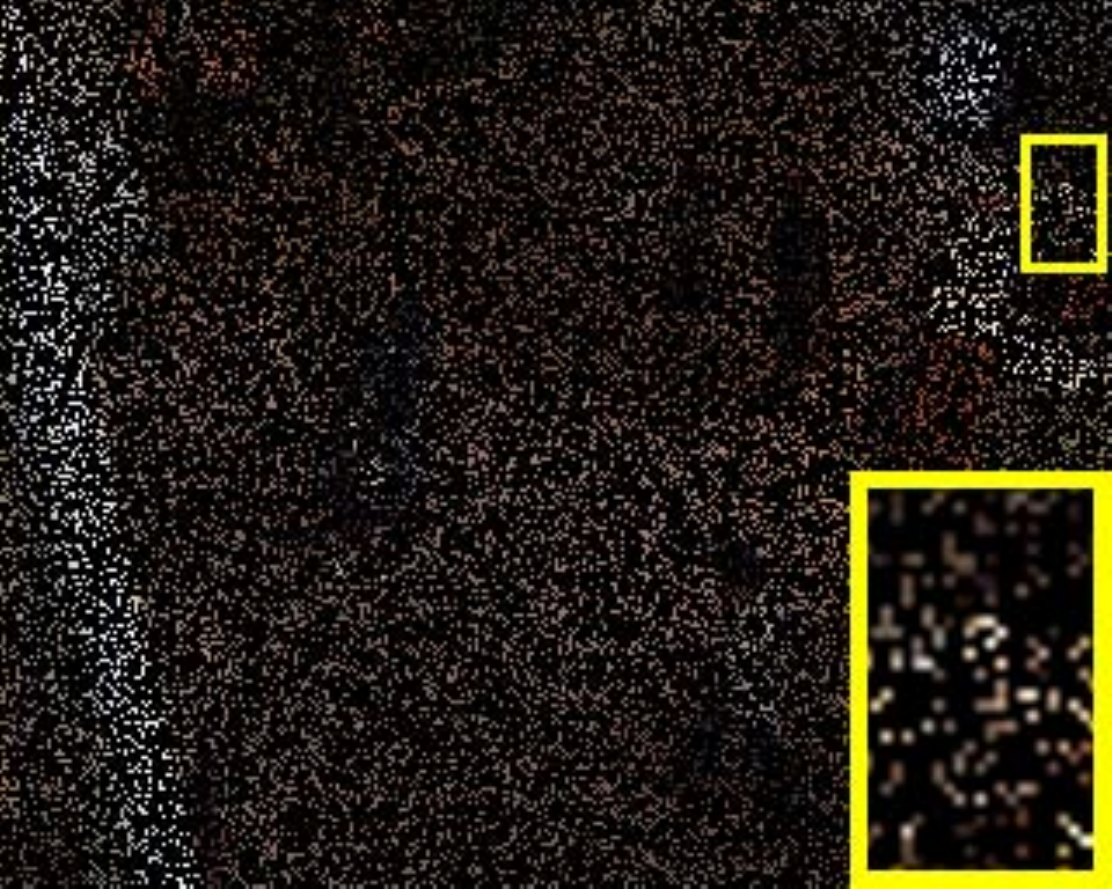}}  \subfigure{\includegraphics[width = 0.18\textwidth]{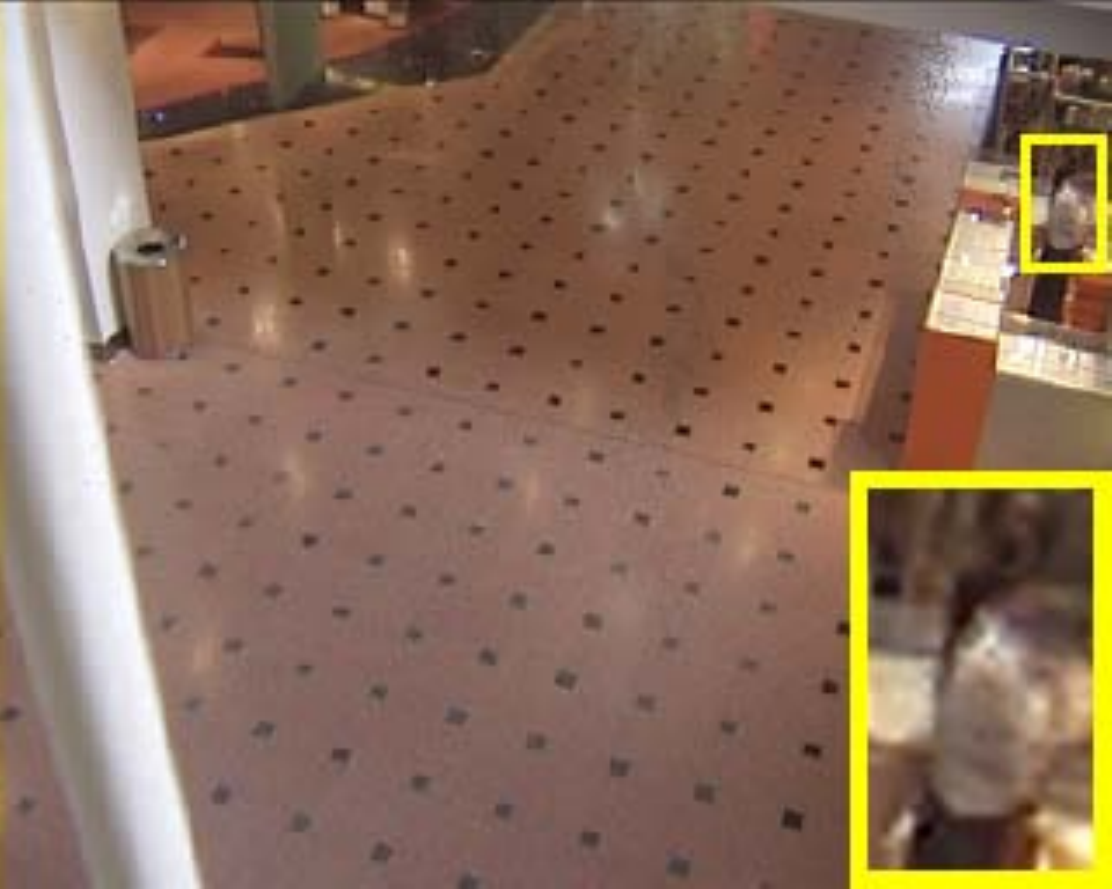}}
\subfigure{\includegraphics[width = 0.18\textwidth]{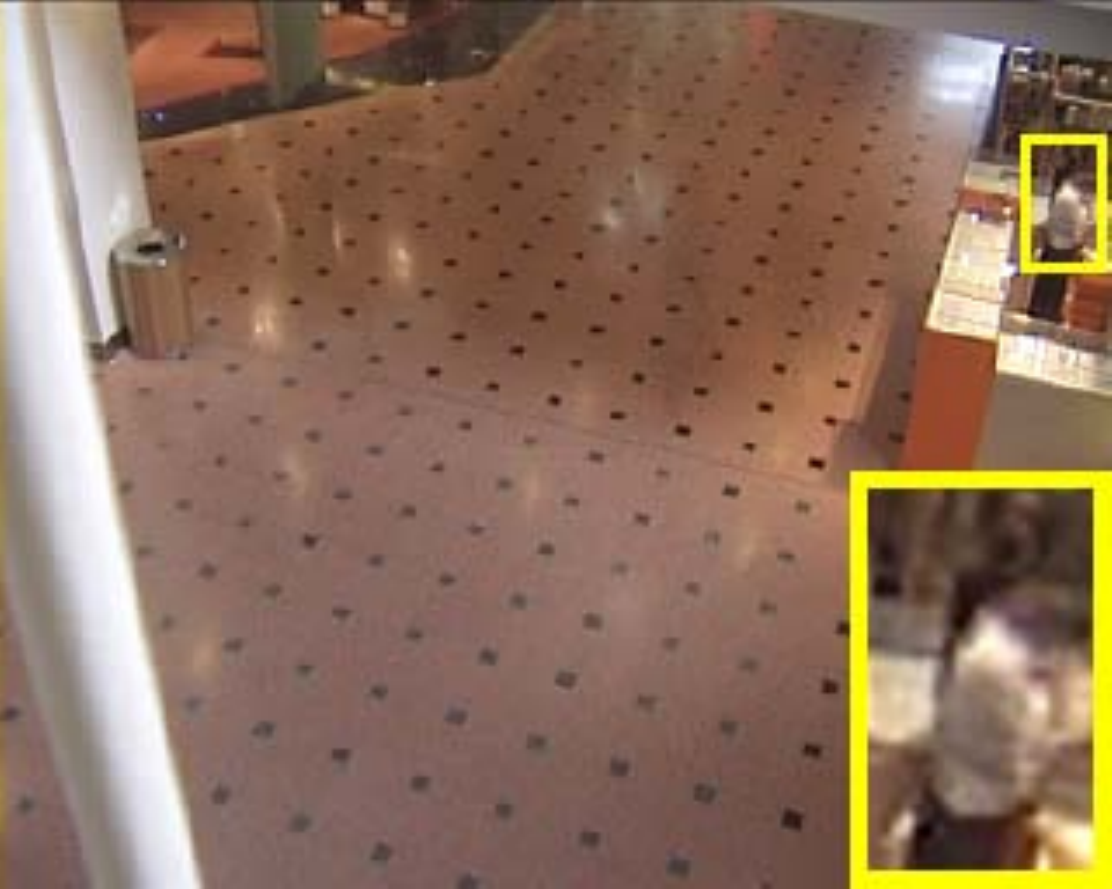}}
\\[-0.3em]
\subfigure{\includegraphics[width = 0.18\textwidth]{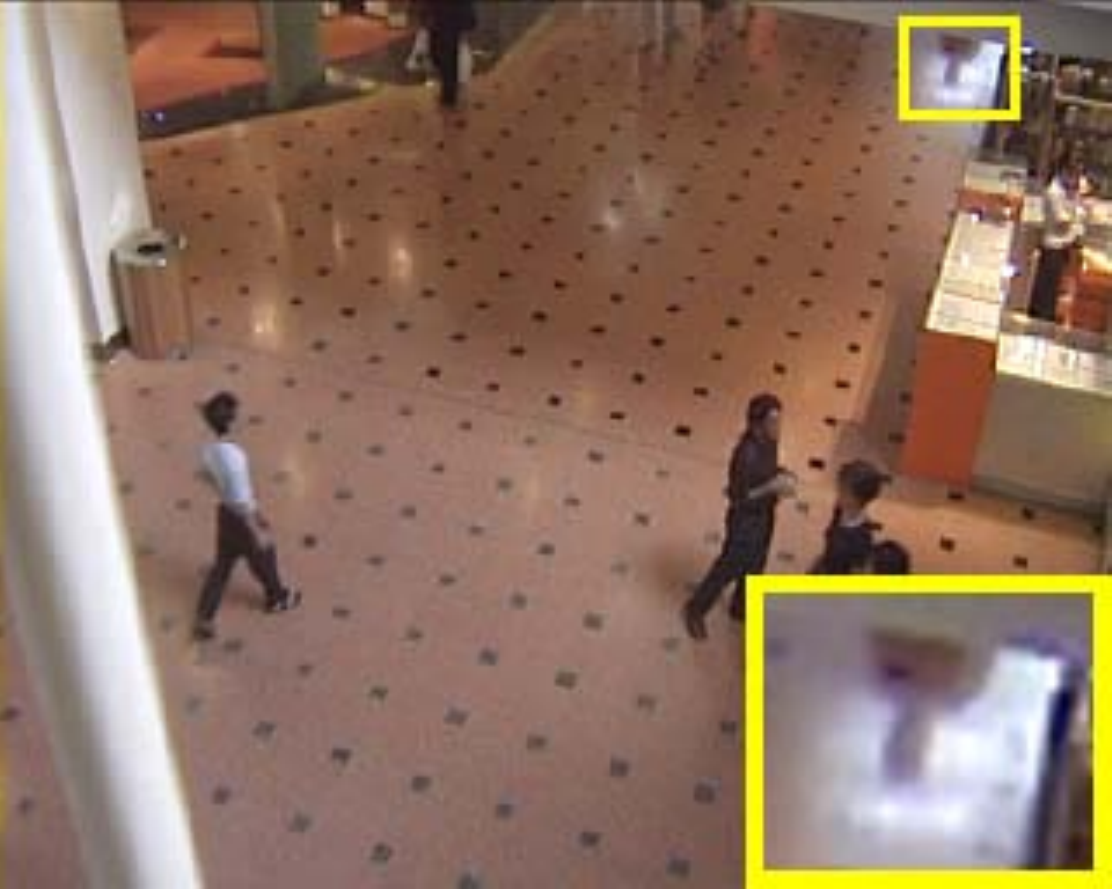}}  \subfigure{\includegraphics[width = 0.18\textwidth]{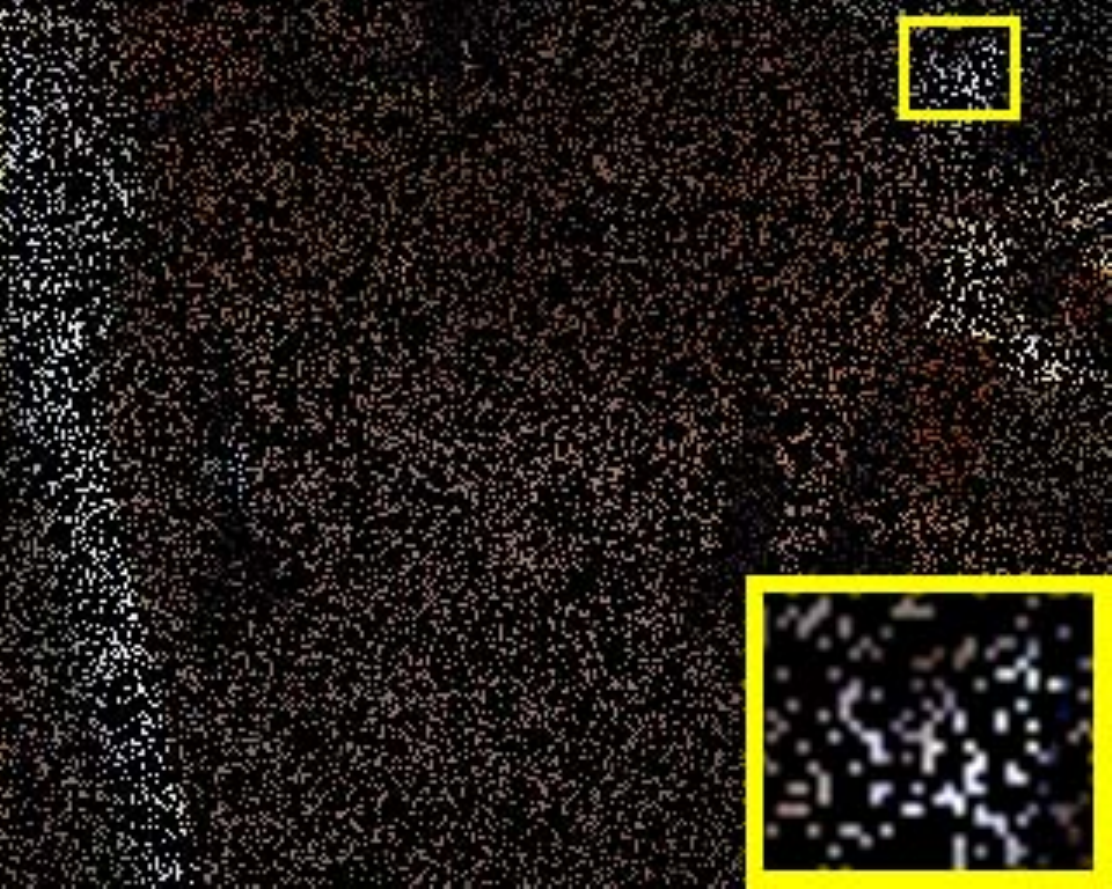}}  \subfigure{\includegraphics[width = 0.18\textwidth]{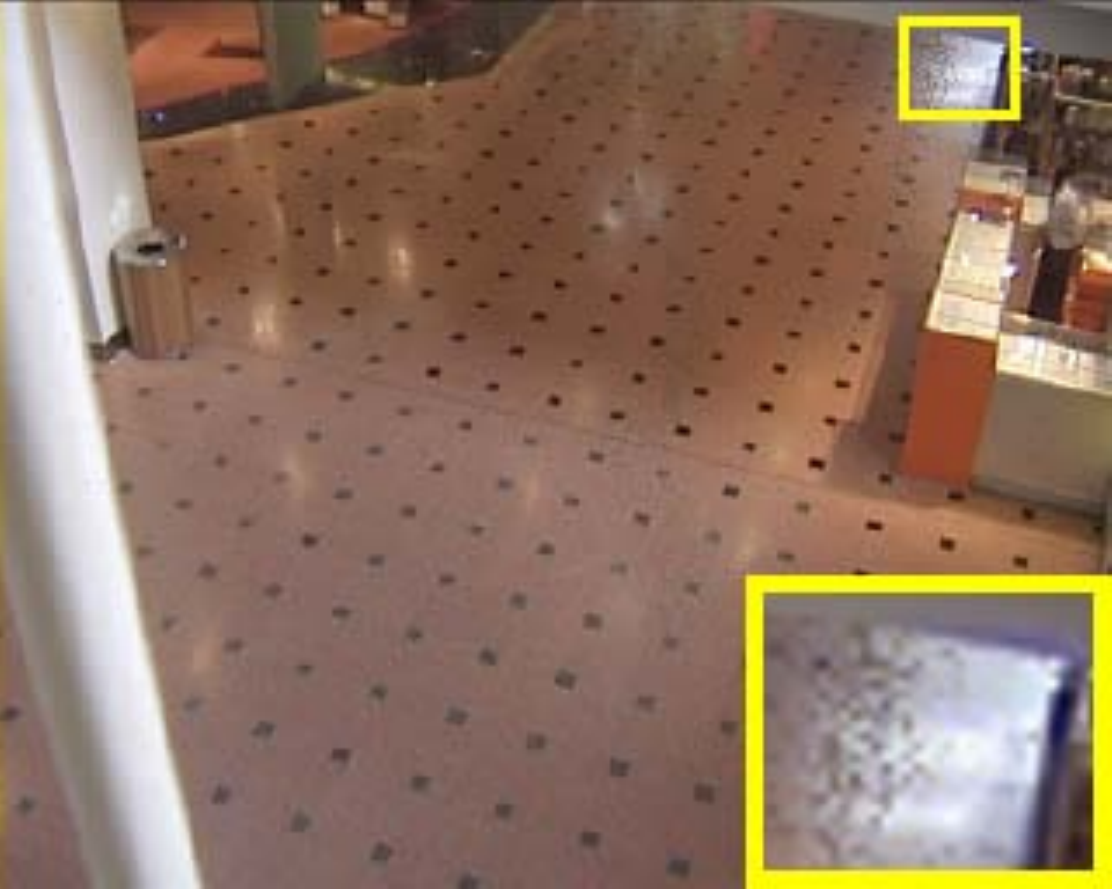}}
\subfigure{\includegraphics[width = 0.18\textwidth]{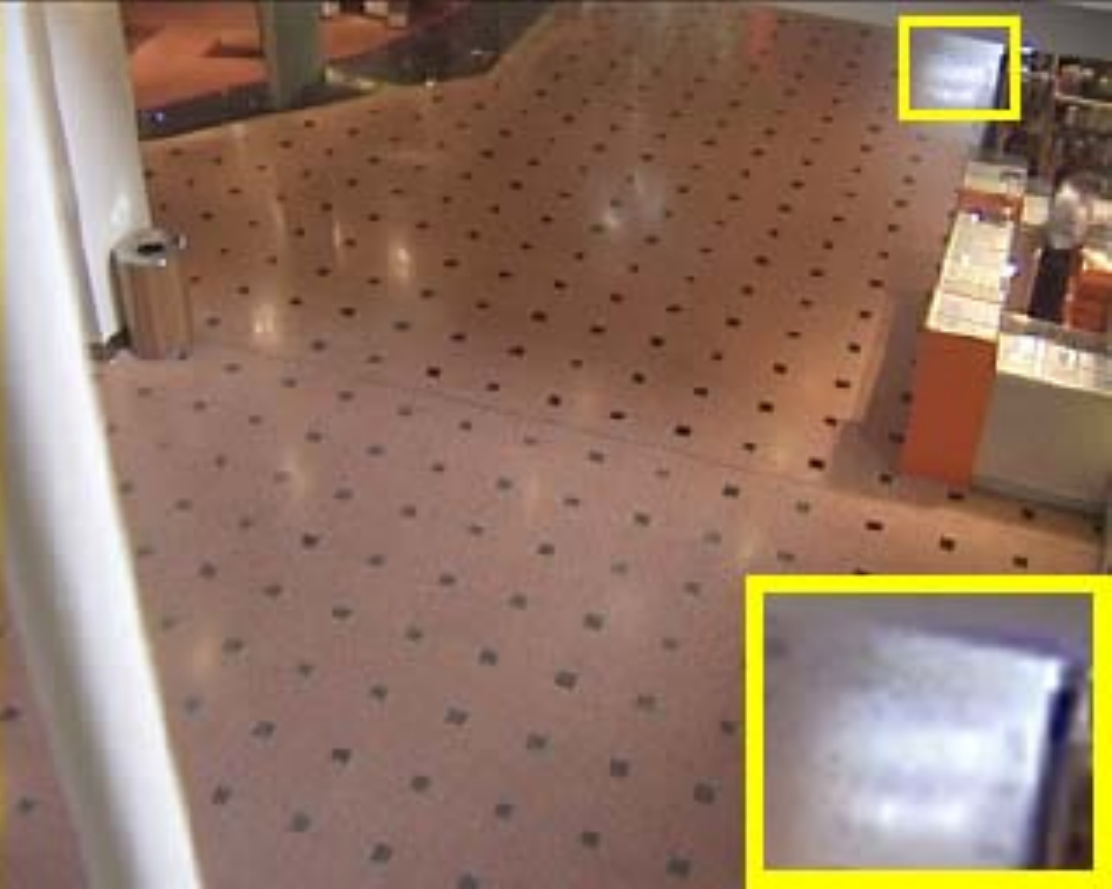}}
\caption{Background modeling results on two frames from Hall and ShoppingMall video sequences respectively when $\rho = 0.2$. From left column to right column, there are respectively the original frames, the input frames, and the backgrounds reconstructed by RMC and RTC.}
\label{fig8}
\end{figure}

To make a further comparison of our algorithm and other methods, we perform additional experiments on the three video sequences by randomly dropping 80\% pixels for each frame. As there are an extremely small proportion of pixels are available, it makes no sense to detect the foreground objects. Instead, we are interested in the background reconstruction in this situation. Both RPCA and TRPCA suffer a failure, while RMC and RTC are capable of recovering the background pretty well. From Figure~\ref{fig8}, it can be seen that our algorithm outperforms RMC once again. Please see the second example in Figure~\ref{fig8}. It is in fact rather difficult to separate the woman as foreground object, because she stands for a while, moves away, and then return very soon. RMC can not separate this person at all. Despite some artifacts, our method provides a much clearer background.

\subsection{Traffic Volume Estimation}
\label{sec7:sub4}

In intelligence transportation systems, traffic flow data, such as traffic volumes, occupancy rates and flow speeds, are usually contaminated by missing values and outliers due to the hardware or software malfunctions. In this experiment, we apply our method to estimation of traffic flow volume from incomplete and noisy measurements.

The data used here are collected by a detector (No.314521) located on SR160-N, Sacramento County, California, from March 1 to May 30, 2011 and can be downloaded from the  Caltrans Performance Measurement System (PeMS)\footnote{\url{http://pems.dot.ca.gov/}}. Since the data are recorded every 5 minutes,  it can be mapped to a third-order tensor $\ten{X}$ of size 7(day) $\times$ 288(time) $\times$ 8(week), which have a low-rank structure because of the periodicity~\cite{Tan2013a,Acar2010}. Therefore, traffic volume estimation can be modeled as a (low-rank) tensor completion problem.

\begin{figure}[!t]
\centering
\includegraphics[width=0.7\textwidth]{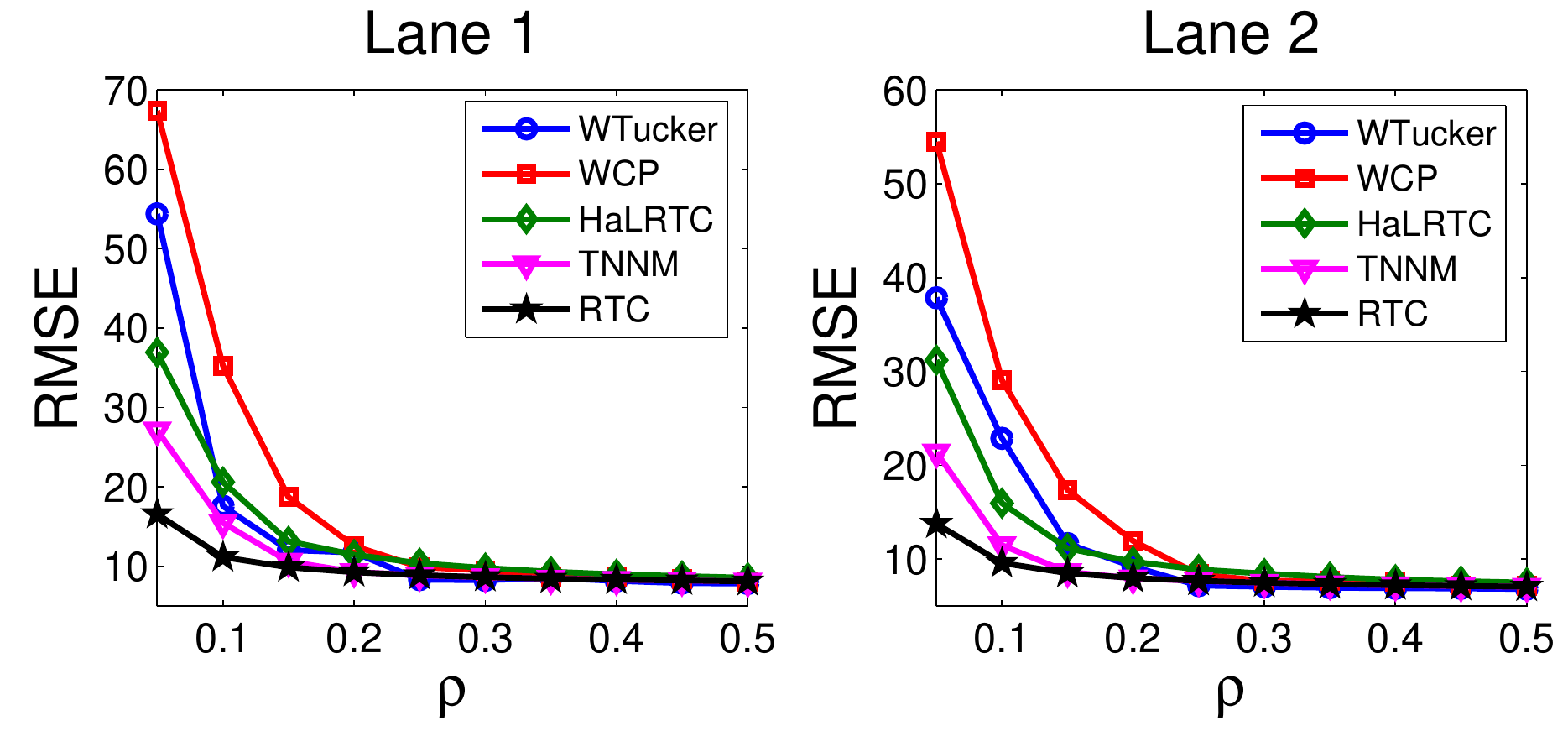}
\caption{Comparison of the estimation results obtained by all the methods in terms of RMSE on traffic data with different sampling rates.}
\label{fig9}
\end{figure}

We randomly sample $\rho$ percentage of the traffic data, and compare the recovery behavior of our approach with several typical tensor completion methods, WTucker~\cite{Filipovic2013}\footnote{\url{http://www.lair.irb.hr/ikopriva/marko-filipovi.html}}, WCP~\cite{Acar2010}\footnote{\url{http://www.sandia.gov/~tgkolda/TensorToolbox/index-2.6.html}}, HaLRTC~\cite{Liu2013} and TNNM~\cite{Zhang2014}. WTucker and WCP are the weighted variants of the classic Tucker and CP decomposition respectively, which have been originally proposed for tensor decomposition with missing values.
Following~\cite{Tan2013a}, we fix the $n$-rank $r_1 = r_2 = r_3 = 2$ for WTucker and set the CP-rank $r = 2$ for WCP. We empirically choose larger $\lambda = 20/\sqrt{n_{(1)}n_3}$\footnote{We find that the performance of our method is not good enough when $\lambda$ is set to the default value $\lambda = 1/\sqrt{n_{(1)}n_3}$ and the empirical setting $\lambda = 20/\sqrt{n_{(1)}n_3}$ allows our method to achieve appealing results.} for our method, since the outliers are extremely sparse in the traffic data. Figure~\ref{fig9} gives the results obtained by all the methods in terms of the Root Mean Square Error (RMSE) defined by
\begin{equation}
\textup{RMSE} = \sqrt{\frac{1}{|\Omega^\perp|} \sum_{(i,j,k) \in \Theta}\big(\ten{L}_{ijk} - \ten{X}_{ijk}\big)^2}, \nonumber
\end{equation}
where $|\Omega^\perp|$ is the total number of entries in the unobserved set $\Omega^\perp$. As expected, our method always outshines other approaches, especially in the case of low sampling rate, say $\rho \leq 0.2$.

\section{Conclusions}
\label{sec8}

In this work, we conduct a rigourous study for the RTC problem which aims to learn a low-tubal-rank tensor from partial observations that are arbitrarily corrupted. Our study rests heavily on recently proposed t-SVD and associated algebraic framework, in which we can define the tubal rank and tubal nuclear norm for tensors. We propose a new group of tensor incoherence conditions which are natural and elegant extensions of the corresponding matrix incoherence conditions respectively and are much weaker than those given by~\cite{Lu2016}. Under these conditions, we show that one can exactly recover a third-order tensor having low tubal-rank with high probability and establish a theoretical bound for exact recovery when using a convex optimization algorithm. Numerical experiments verify our theoretical analysis and the real-world applications demonstrate the superiority of our method over other existing approaches.

Our results confirm again that the t-SVD algebraic framework, which we outline in Section~\ref{sec3}, is more appropriate to capture the low-rank structure in multidimensional data. This suggests that it is very interesting to apply our model and algorithm to other possible applications, such as face recognition, web data mining and bioinformatics. Considering that real data routinely lie in thousands or even billions of dimensions, the computational cost of our method may become expensive. We are require to develop fast algorithms for low-tubal-rank tensor recovery and will explore this important direction in our future work.

\bibliographystyle{plain}
\bibliography{2017_RTC}

\appendix
\section{Proof of Lemma~\ref{lem1}}
\label{app1}

To prove the Lemma~\ref{lem1}-\ref{lem3}, we first introduce the following theorem
\begin{lemma}[Noncommutative Bernstein Inequality\cite{Recht2011}]
Let $\mtx{X}_1, \mtx{X}_2, \dots,\mtx{X}_n$ be independent zero-mean random matrices of dimension $d_1 \times d_2$, and $\rho_k^2 = \max\{\|\E[\mtx{X}_k \mtx{X}_k^T]\|, \|\E[\mtx{X}_k^T \mtx{X}_k]\|\}$. Suppose $\|\mtx{X}_k\| \leq M$ almost surely for all $k$. Then for any $\tau > 0$,
\begin{equation}
\PP\bigg[\bigg\|\sum_{k=1}^n \mtx{X}_k\bigg\| > \tau \bigg] \leq (d_1 + d_2) \exp\bigg(\frac{-\tau^2/2}{\sum_{k=1}^n \rho_k^2 + M\tau/3}\bigg). \label{eq49}
\end{equation}
\label{lem5}
\end{lemma}
\textbf{Proof of Lemma~\ref{lem1}:}
First, we note that
\begin{equation}
\E(\rho^{-1}\PT\PO\PT) = \rho^{-1}\PT\E(\PO)\PT = \PT, \nonumber
\end{equation}
which implies
\begin{equation}
\E(\rho^{-1}\PT\PO\PT - \PT) = 0,\,\,\,\textup{and}\,\,\, \E(\rho^{-1}\overline{\PT\PO\PT} - \overline{\PT}) = 0.\nonumber
\end{equation}
Our goal is to prove the operator $\rho^{-1}\overline{\PT\PO\PT}$ is not far away from its expected value $\overline{\PT}$ in the spectral norm using the Noncommutative Bernstein Inequality.

Give any tensor $\ten{Z} \in \R^{n_1 \times n_2 \times n_3}$,  we can decompose $\PT(\ten{Z})$ as
\begin{equation}
\PT(\ten{Z}) =  \sum_{i,j,k} \<\PT(\ten{Z}), \eijk\> \eijk =  \sum_{i,j,k} \<\ten{Z}, \PT(\eijk)\> \eijk, \nonumber
\end{equation}
which gives
\begin{equation}
\rho^{-1}\PT\PO\PT(\ten{Z})  = \sum_{i,j,k} \rho^{-1}\delta_{ijk}\<\ten{Z}, \PT(\eijk)\>\PT(\eijk), \nonumber
\end{equation}
and implies
\begin{equation}
\rho^{-1}\overline{\PT\PO\PT(\ten{Z})} = \sum_{i,j,k} \rho^{-1}\delta_{ijk}\<\ten{Z}, \PT(\eijk)\>\overline{\PT(\eijk)}. \nonumber
\end{equation}
Define the operator $\mathcal{T}_{ijk}$ which maps $\ten{Z}$ to $\rho^{-1}\delta_{ijk}\<\ten{Z}, \PT(\eijk)\>\PT(\eijk)$. Note that this operator is rank one and has operator norm $\|\mathcal{T}_{ijk}\|_{\textup{op}} = \|\overline{\mathcal{T}_{ijk}}\|_{\textup{op}} = \frac{1}{\rho}\|\PT(\eijk)\|^2_F$, and $\|\PT\|_{\textup{op}} = \|\overline{\PT}\|_{\textup{op}} \leq 1$. Furthermore, we have $\PT = \sum_{ijk} \mathcal{T}_{ijk}$ and $\E(\mathcal{T}_{ijk}) = \frac{1}{n_1n_2n_3}\PT$. Hence, we have
\begin{equation}
\Big\|\overline{\mathcal{T}_{ijk}} - \frac{1}{n_1n_2n_3}\overline{\PT}\Big\|_{\textup{op}} \leq \max\Big\{\frac{1}{\rho}\|\PT(\eijk)\|^2_F, \frac{1}{n_1n_2n_3}\Big\} \leq  \frac{2\mu r}{n_{(2)}\rho}, \nonumber
\end{equation}
where the first inequality uses the fact that if $\mtx{A}$ and $\mtx{B}$ are positive semidefinite matrices, then $\|\mtx{A} - \mtx{B}\| \leq \max\{\|\mtx{A}\|, \|\mtx{B}\|\}$.

On the other hand, we know
\begin{align}
\Big\|\E\Big[\Big(\overline{\mathcal{T}_{ijk}} - \frac{1}{n_1 n_2 n_3}\overline{\PT}\Big)^2\Big]\Big\| & \leq \Big\|\E\Big[\frac{1}{\rho}\|\PT(\eijk)\|^2_F\overline{\mathcal{T}_{ijk}}\Big] - \frac{2}{n_1 n_2 n_3}\overline{\PT}\E(\overline{\mathcal{T}_{ijk}}) + \frac{1}{n^2_1n^2_2n^2_3}\overline{\PT}]\Big\| \nonumber\\
& =  \Big\|\frac{1}{\rho}\|\PT(\eijk)\|^2_F\frac{1}{n_1n_2n_3}\overline{\PT} - \frac{1}{n^2_1 n^2_2 n^2_3}\overline{\PT}\Big\| \nonumber\\
& <  \Big(\frac{1}{\rho} \frac{2\mu r}{n_{(2)}} \frac{1}{n_1 n_2 n_3}\Big)\|\overline{\PT}\| \nonumber\\
& \leq  \frac{2\mu r}{n_{(1)}n^2_{(2)}n_3\rho}. \nonumber
\end{align}
Letting $\tau = \sqrt{\frac{C_0 \mu r \log(n_{(1)}n_3)}{n_{(2)}\rho}} \leq \epsilon$ and using Lemma~\ref{lem5}, we have
\begin{align}
\PP[\|\rho^{-1}\PT\PO\PT - \PT\|_{op} > \tau]
& =  \PP[\|\rho^{-1}\overline{\PT\PO\PT} - \overline{\PT}\|_{op} > \tau] \nonumber\\
& =  \PP\Big[\Big\|\sum_{i,j,k} \big(\overline{\mathcal{T}_{ijk}} - \frac{1}{n_1n_2n_3}\overline{\PT}\big)\Big\|_{op}> \tau \Big] \nonumber\\
& \leq  2n_{(1)}n_3\exp\Bigg(\frac{-\frac{C_0\mu r\log(n_{(1)}n_3)}{n_{(2)}\rho}}{\frac{2\mu r}{n_{(2)}\rho} + \frac{2\mu r}{3n_{(2)}\rho}}\Bigg) \nonumber\\
& =  2(n_{(1)} n_3)^{1-\frac{3}{16}C_0}, \nonumber
\end{align}
which means
\begin{align}
\PP[\|\rho^{-1}\PT\PO\PT - \PT\|_{op} \leq \epsilon]
& \geq  \PP[\|\rho^{-1}\PT\PO\PT - \PT\|_{op} \leq \tau] \nonumber\\
& \geq  1 - 2(n_{(1)}n_3)^{1-\frac{3}{16}C_0}.\nonumber
\end{align}

\section{Proof of Lemma~\ref{lem2}}
\label{app2}

\begin{proof}
Observe that
\begin{equation}
\rho^{-1}\PT\PO\PT(\ten{Z}) = \sum_{i,j,k} \rho^{-1} \delta_{ijk} \ten{Z}_{ijk} \PT(\eijk). \nonumber
\end{equation}
So, the $(a, b, c)$th entry of $\rho^{-1}\PT\PO\PT(\ten{Z}) - \PT(\ten{Z})$ is given by
\begin{equation}
\<\rho^{-1}\PT\PO\PT(\ten{Z}) - \PT(\ten{Z}), \eabc\> = \sum_{i,j,k} \Big(\frac{\delta_{ijk}}{\rho} - 1\Big) \ten{Z}_{ijk} \<\PT(\eijk), \eabc\>. \nonumber
\end{equation}
We define $\ten{H}_{ijk}:=\Big(\frac{\delta_{ijk}}{\rho} - 1\Big) \ten{Z}_{ijk} \<\PT(\eijk), \eabc\>$ and have that
\begin{equation}
|\ten{H}_{ijk}| \leq \frac{1}{\rho}\|\ten{Z}\|_{\infty}\|\PT(\eijk)\|_F\|\PT(\eabc)\|_F \leq \frac{2\mu r}{n_{(2)}\rho}\|\ten{Z}\|_{\infty}. \nonumber
\end{equation}
We also have
\begin{align}
\Big|\E\Big[\sum_{i,j,k} \ten{H}^2_{ijk}\Big]\Big| & \leq \frac{1-\rho}{\rho}\|\ten{Z}\|^2_{\infty} \sum_{i,j,k}|\<\PT(\eijk),\eabc\>|^2 \nonumber\\
& =  \frac{1-\rho}{\rho}\|\ten{Z}\|^2_{\infty} \|\PT(\eabc)\|^2_F \nonumber\\
& \leq  \frac{2\mu r}{n_{(2)}\rho} \|\ten{Z}\|^2_{\infty}. \nonumber
\end{align}
Letting $\tau = \sqrt{\frac{C_0 \mu r \log(n_{(1)}n_3)}{n_{(2)}\rho}}  \leq \epsilon$ and using Lemma~\ref{lem5}, we obtain
\begin{align}
&\PP[(\rho^{-1}\PT\PO\PT(\ten{Z}) - \PT(\ten{Z}))_{abc} > \tau \|\ten{Z}\|_{\infty}] \nonumber\\
& \leq  2\exp\Bigg(\frac{-\frac{C_0 \mu r \log(n_{(1)}n_3)}{n_{(2)}\rho}\|\ten{Z}\|^2_{\infty}}{\frac{2\mu r}{n_{(2)}\rho}\|\ten{Z}\|^2_{\infty} + \frac{2\mu r}{3n_{(2)}\rho}\|\ten{Z}\|^2_{\infty}}\Bigg) \nonumber\\
& \leq  2(n_{(1)}n_3)^{-\frac{3C_0}{16}}. \nonumber
\end{align}
Then using the union bound on every $(a, b, c)$th entry, we have $\|(\rho^{-1}\PT\PO\PT - \PT)\ten{Z}\|_{\infty} \leq \epsilon \|\ten{Z}\|_{\infty}$ holds with probability at least $1-2n_{(1)}^{-(\frac{3C_0}{16}-2)}n_3^{-(\frac{3C_0}{16}-1)}$.
\end{proof}

\section{Proof of Lemma~\ref{lem3}}
\label{app3}

\begin{proof}
Observe that
\begin{equation}
\rho^{-1}\PT\PO\PT(\ten{Z}) - \ten{Z} = \sum_{i,j,k} \Big(\frac{1}{\rho}\delta_{ijk} - 1\Big)\ten{Z}_{ijk}\eijk. \nonumber
\end{equation}
We define $\ten{C}_{ijk}:=\Big(\frac{1}{\rho}\delta_{ijk} - 1\Big)\ten{Z}_{ijk}\eijk$ and have
\begin{equation}
\overline{\ten{C}_{ijk}} = \Big(\frac{1}{\rho}\delta_{ijk} - 1\Big)\ten{Z}_{ijk}\overline{\eijk}. \nonumber
\end{equation}
Note that $\E(\overline{\ten{C}_{ijk}}) = 0$ and $\|\overline{\ten{C}_{ijk}}\| \leq \frac{1}{\rho}\|\ten{Z}\|_{\infty}$. Moreover,
\begin{align}
\Big\|\E\Big[\sum_{i,j,k}\overline{\ten{C}_{ijk}}^{H}\overline{\ten{C}_{ijk}}\Big]\Big\|
& =  \Big\|\E\Big[\sum_{i,j,k}\ten{C}_{ijk}^{H} \ten{C}_{ijk}\Big]\Big\| \nonumber\\
& =  \Big\|\sum_{i,j,k}\ten{Z}^2_{ijk} \tc{e}_j \ast \tc{e}_j^{H} \E\Big (\frac{1}{\rho}\delta_{ijk} - 1\Big)^2\Big\| \nonumber\\
& =  \Big\|\frac{1-\rho}{\rho}\sum_{ijk} \ten{Z}^2_{ijk} \tc{e}_j \ast \tc{e}_j^{H}\Big\|. \nonumber
\end{align}
Since $\tc{e}_j \ast \tc{e}_j^H$ returns a zero tensor except for $(j, j, 1)$th entry equaling 1, we have
\begin{equation}
\Big\|\E\Big[\sum_{i,j,k}\overline{\ten{C}_{ijk}}^H\overline{\ten{C}_{ijk}}\Big]\Big\| = \frac{1-\rho}{\rho}\max_{j}\Big|\sum_{i,k}\ten{Z}^2_{ijk}\Big| \leq \frac{1}{\rho}n_{(1)}n_3\|\ten{Z}\|^2_{\infty}, \nonumber
\end{equation}
and
$\Big\|\E\Big[\sum_{i,j,k}\overline{\ten{C}_{ijk}}\,\overline{\ten{C}_{ijk}}^{H}\Big]\Big\|$ is bounded similarly. Then considering that
\begin{equation}
C'_0\sqrt{\frac{n_{(1)}n_3\log(n_{(1)}n_3)}{\rho}}\|\ten{Z}\|_{\infty} \|\overline{\ten{C}_{ijk}}\| \leq \frac{C'_0}{\sqrt{C_0}\rho}\sqrt{n_{(1)} n_{(2)}n^2_3} \|\ten{Z}\|^2_{\infty}  \leq \frac{C'_0}{\sqrt{C_0}} n_{(1)} n_3 \frac{1}{\rho}\|\ten{Z}\|^2_{\infty}, \nonumber
\end{equation}
and using Lemma~\ref{lem5}, we have
\begin{align}
& \PP\Big[\|\rho^{-1}\PO(\ten{Z}) - \ten{Z}\|> C'_0\sqrt{\frac{n_{(1)}n_3\log(n_{(1)}n_3)}{\rho}}\|\ten{Z}\|_{\infty}\Big] \nonumber\\
& = \PP\Big[\|\rho^{-1}\overline{\PO(\ten{Z})} - \overline{\ten{Z}}\| > C'_0\sqrt{\frac{n_{(1)}n_3\log(n_{(1)}n_3)}{\rho}}\|\ten{Z}\|_{\infty}\Big] \nonumber\\
& = \PP\Big[\Big\|\sum_{ijk} \overline{\ten{C}_{ijk}}\Big\|> C'_0\sqrt{\frac{n_{(1)}n_3\log(n_{(1)}n_3)}{\rho}}\|\ten{Z}\|_{\infty}\Big] \nonumber\\
& \leq  2n_{(1)}n_3 \exp\Bigg(\frac{\frac{-(C'_0)^2n_{(1)}n_3\log(n_{(1)}n_3)}{2\rho}\|\ten{Z}\|^2_{\infty}}{\frac{n_{(1)}n_3\|\ten{Z}\|^2_{\infty}}{\rho}}\Bigg) \nonumber\\
& = 2(n_{(1)}n_3)^{1-\frac{(C'_0)^2}{2}}, \nonumber
\end{align}
which implies that $\|(\OpId - \rho^{-1}\PO)\ten{Z}\| \leq C'_0\sqrt{\frac{n_{(1)}n_3\log(n_{(1)}n_3)}{\rho}}\|\ten{Z}\|_{\infty}$ holds with probability at least $1 - 2(n_{(1)}n_3)^{1-\frac{(C'_0)^2}{2}}$.
\end{proof}

\end{document}